\documentclass[11pt,onecolumn]{saim-paper}
\usepackage[T1]{fontenc}
\usepackage{lmodern}
\makeatletter
\renewcommand{\SAIMPaperTableOfContents}{%
  \vspace{3pt}%
  \begin{tcolorbox}[
    blanker,
    left=10pt,
    right=2pt,
    top=1pt,
    bottom=1pt,
    before skip=0pt,
    after skip=7pt,
    borderline west={2.2pt}{0pt}{Primary}
  ]
    {\SAIMNavigationFont\fontsize{8}{9.6}\selectfont\color{Primary}\textls[165]{CONTENTS}\par}
    \vspace{2pt}
    {\footnotesize
      \setlength{\parskip}{0pt}%
      \@starttoc{toc}}
  \end{tcolorbox}%
}
\makeatother

\newcommand{\ProjectTitle}{Generalized Hierarchical Bayesian Segmentation with Irregular Designs, Multi-Sequence Hierarchies, and Grouped/Latent-Group Designs}

\newcommand{\ProjectLeadAuthor}{Omid Shams Solari}
\newcommand{\ProjectAuthors}{Omid Shams Solari}
\newcommand{\ProjectAffiliation}{sAIm Labs}
\newcommand{\ProjectCorrespondence}{omid.solari@saimai.com}
\newcommand{\ProjectCorrespondenceTwo}{solari@berkeley.edu}
\newcommand{\ProjectGitHub}{https://github.com/osolari/bayesbreak}
\newcommand{\ProjectPyPI}{https://pypi.org/project/bayesbreak/}
\newcommand{\ProjectDOI}{10.5281/zenodo.23711344}
\newcommand{\ProjectDOIURL}{https://doi.org/10.5281/zenodo.23711344}

\newcommand{\ProjectKeywords}{Bayesian changepoint analysis, multiple changepoints, product partition models, exponential families, conjugate priors, hierarchical Bayesian models, dynamic programming, irregular designs, multiple sequences, latent groups}

\graphicspath{{figures/}{shared/figures/tikz/}{shared/saim/brand/logos/}}

\usepackage{amsmath,amsfonts,amssymb,amsthm,mathtools,bm,bbm}
\usepackage{float,makecell,longtable,tablefootnote,pifont,adjustbox,xspace,multicol}
\usepackage[ruled,vlined,linesnumbered]{algorithm2e}
\newcolumntype{Y}{>{\raggedright\arraybackslash}X}
\newcolumntype{Z}{>{\centering\arraybackslash}X}
\newcolumntype{C}[1]{>{\centering\arraybackslash}p{#1}}
\newcolumntype{L}[1]{>{\raggedright\arraybackslash}p{#1}}

\usetikzlibrary{bayesnet,matrix,chains}
\definecolor{bbBlue}    {HTML}{4477AA}
\definecolor{bbRose}    {HTML}{CC6677}
\definecolor{bbTeal}    {HTML}{44AA99}
\definecolor{bbSand}    {HTML}{DDCC77}
\definecolor{bbCyan}    {HTML}{88CCEE}
\definecolor{bbPurple}  {HTML}{AA4499}
\definecolor{bbGreen}   {HTML}{117733}
\definecolor{bbOlive}   {HTML}{999933}
\definecolor{bbWine}    {HTML}{882255}
\definecolor{bbOrange}  {HTML}{EE7733}
\definecolor{bbGrey}    {HTML}{888888}
\definecolor{bbLightGrey}{HTML}{DDDDDD}
\definecolor{bbBlack}   {HTML}{222222}

\definecolor{bbMethodQuadrature}{HTML}{4477AA}  
\definecolor{bbMethodLaplace}   {HTML}{44AA99}  
\definecolor{bbMethodJJ}        {HTML}{DDCC77}  
\definecolor{bbMethodEP}        {HTML}{CC6677}  
\definecolor{bbMethodExact}     {HTML}{4477AA}  
\definecolor{bbMethodApprox}    {HTML}{DDCC77}  

\definecolor{bbFamilyGaussian} {HTML}{4477AA}  
\definecolor{bbFamilyPoisson}  {HTML}{44AA99}  
\definecolor{bbFamilyBinomial} {HTML}{DDCC77}  
\definecolor{bbFamilyNegBin}   {HTML}{EE7733}  
\definecolor{bbFamilyBeta}     {HTML}{AA4499}  

\definecolor{bbPriorUniform}       {HTML}{888888}  
\definecolor{bbPriorGeometric}     {HTML}{4477AA}  
\definecolor{bbPriorGapHazard}     {HTML}{44AA99}  
\definecolor{bbPriorLengthCohesion}{HTML}{EE7733}  
\definecolor{bbPriorCombined}      {HTML}{CC6677}  

\definecolor{bbCorrect}      {HTML}{4477AA}  
\definecolor{bbMisspecified} {HTML}{CC6677}  

\definecolor{bbPooled}       {HTML}{4477AA}  
\definecolor{bbIndependent}  {HTML}{CC6677}  

\definecolor{bbStageFill}    {HTML}{4477AA}  
\definecolor{bbStageText}    {HTML}{FFFFFF}  
\definecolor{bbNodeFill}     {HTML}{E8EFF7}  
\definecolor{bbNodeText}     {HTML}{222222}  
\definecolor{bbAccentFill}   {HTML}{CC6677}  

\usetikzlibrary{positioning,arrows.meta,calc}

\tikzset{
  bbbase/.style={
    rectangle,
    draw=none,
    rounded corners=0pt,
    minimum height=9mm,
    align=center,
    inner sep=5pt,
    font=\small\sffamily,
  },
  bbstage/.style={
    bbbase,
    fill=bbStageFill,
    text=bbStageText,
    minimum width=28mm,
  },
  bbnode/.style={
    bbbase,
    fill=bbNodeFill,
    text=bbNodeText,
    minimum width=28mm,
  },
  bbaccent/.style={
    bbbase,
    fill=bbAccentFill,
    text=white,
    minimum width=28mm,
  },
  bbneutral/.style={
    bbbase,
    fill=bbLightGrey,
    text=bbNodeText,
    minimum width=28mm,
  },
  boxbase/.style={
    rectangle,
    draw=none,
    rounded corners=0pt,
    minimum height=9mm,
    align=center,
    inner sep=5pt,
    font=\small\sffamily,
  },
  primarybox/.style={boxbase, fill=bbStageFill, text=white, minimum width=28mm},
  secondarybox/.style={boxbase, fill=bbNodeFill, text=bbNodeText, minimum width=28mm},
  neutralbox/.style={boxbase, fill=bbLightGrey, text=bbNodeText, minimum width=28mm},
  successbox/.style={boxbase, fill=bbAccentFill, text=white, minimum width=28mm},
  bbarrow/.style={
    -{Stealth[length=2.2mm]},
    line width=0.75pt,
    draw=bbBlack,
  },
  bbarrowdash/.style={
    -{Stealth[length=2.2mm]},
    line width=0.75pt,
    draw=bbBlack,
    dashed,
  },
  arr/.style={bbarrow},
  arrPrimary/.style={bbarrow},
  arrSecondary/.style={bbarrow},
  arrow/.style={bbarrow},
  darrow/.style={bbarrowdash},
  tier/.style={draw=bbGrey, dashed, line width=0.5pt, inner sep=3mm},
  csblock/.style={bbnode, minimum height=1.0cm, minimum width=3.6cm},
  csfinal/.style={bbaccent, minimum height=1.0cm, minimum width=3.6cm},
  cslevel1/.style={bbstage, text width=9em, minimum height=2.2em, font=\small\rmfamily\bfseries},
  cslevel2/.style={bbnode, text width=7.5em, minimum height=1.8em, font=\scriptsize\rmfamily},
  cscrosscut/.style={bbneutral, text width=8em, minimum height=1.8em, font=\scriptsize\rmfamily},
  csmethodbox/.style={bbstage, minimum width=2.6cm, minimum height=0.8cm, font=\small\rmfamily\bfseries},
  cspropbox/.style={bbnode, minimum width=2.6cm, minimum height=0.6cm, font=\scriptsize\rmfamily, text width=2.6cm},
  csstagelbl/.style={font=\rmfamily\bfseries\small, text=bbRose},
}

\newtheorem{theorem}{Theorem}[section]
\newtheorem{proposition}[theorem]{Proposition}
\newtheorem{lemma}[theorem]{Lemma}
\newtheorem{corollary}[theorem]{Corollary}
\theoremstyle{definition}
\newtheorem{definition}[theorem]{Definition}
\theoremstyle{remark}
\newtheorem{remark}[theorem]{Remark}
\SAIMPaperTheoremStyle{theorem}
\SAIMPaperTheoremStyle{proposition}
\SAIMPaperTheoremStyle{lemma}
\SAIMPaperTheoremStyle{corollary}
\SAIMPaperTheoremStyle{definition}
\newcolumntype{R}[1]{>{\raggedleft\arraybackslash}p{#1}}
\newcommand{\BayesBreak}{\textsc{BayesBreak}\xspace}

\newcommand{\1}{\mathbbm{1}}
\newcommand{\NA}{\textsc{na}}

\newcommand{\resultid}[1]{\texttt{#1}}

\newcommand{\dd}{\,\mathrm{d}}

\newcommand{\logsumexp}{\operatorname{LSE}}

\newtcolorbox{exclusionbox}[1][Excluded computation]{saim/amber,title={#1}}
\newtcolorbox{evidencebox}[1][Result status]{saim/green,title={#1}}
\newtcolorbox{specificationbox}[1][Mathematical specification]{saim/primary,title={#1}}
\newtcolorbox{verificationbox}[1][Verification note]{saim/neutral,title={#1}}
\newtheorem{assumption}[theorem]{Assumption}

\newtheorem{proofobligation}[theorem]{Proof obligation}

\renewcommand{\edge}[3][]{%
  \foreach \x in {#2}{%
    \foreach \y in {#3}{%
      \draw[->,#1] (\x) -- (\y);%
    }%
  }%
}
\renewcommand{\plate}[4][]{%
  \node[wrap=#3] (#2-wrap) {};%
  \node[plate caption=#2-wrap] (#2-caption) {#4};%
  \node[plate=(#2-wrap)(#2-caption),#1] (#2) {};%
}

\SAIMPaperTheoremStyle{assumption}
\SAIMPaperTheoremStyle{proofobligation}

\newcommand{\PaperTitle}{\ProjectTitle}
\newcommand{\PaperShortTitle}{BayesBreak}
\newcommand{\PaperVersion}{}
\paperlabel{}
\title{\PaperTitle}
\shorttitle{\PaperShortTitle}
\paperversion{\PaperVersion}
\author[1]{\ProjectLeadAuthor}
\affiliation[1]{\ProjectAffiliation}
\metadata[Correspondence]{\email{\ProjectCorrespondence}; \email{\ProjectCorrespondenceTwo}}
\hypersetup{
  pdftitle={\PaperTitle},
  pdfauthor={\ProjectAuthors},
  pdfsubject={Generalized hierarchical Bayesian segmentation},
  pdfkeywords={\ProjectKeywords}
}
\abstract{%
We develop a generalized hierarchical Bayesian method for multiple-changepoint segmentation of ordered data observed on regular or irregular designs. For a regular exponential-family observation model with a proper segmentwise conjugate prior and finite normalizing constants, cumulative sufficient statistics determine the marginal likelihood and posterior moments of every candidate segment. These segment marginal likelihoods enter dynamic-programming recursions for the marginal likelihood conditional on the number of segments, the posterior distribution of the segment count, marginal changepoint probabilities, and posterior moments of the piecewise-constant signal. A separate max-sum recursion with backtracking returns the joint maximum-a-posteriori partition. Thus, the partition calculation is common across observation families once the segment marginal likelihoods are specified, while the statistical meaning of each segment remains family-specific.

The hierarchical formulation includes design-dependent partition priors for irregular coordinates, common-changepoint models for multiple conditionally independent sequences, known-group models, and a finite latent-group template model. The latent-group algorithm uses a Jensen minorization and exact responsibility-weighted max-sum updates for the stated finite objective; because its component scores are not normalized sampling densities, it is not interpreted as a finite-mixture likelihood. When a nonconjugate segment integral is replaced by a deterministic approximation, a uniform error bound on reachable segment log marginal likelihoods implies explicit bounds on partition marginal likelihoods and posterior odds.

Simulation studies report boundary-posterior calibration for Gaussian data
(expected calibration error below $0.0081$ across eight pooling strata, Brier
score below $0.020$ across all cells), applications of common partition
recursions to Gaussian, Poisson, Binomial, and Beta-valued responses, and
96\% hard allocation accuracy in one two-group template experiment. A separate
400-dataset stress study evaluates template recovery and objective monotonicity.
Observed Gaussian computation times range from $0.037$ seconds at $n=128$ to
$3.25$ seconds at $n=2048$ for $k_{\max}=20$; peak memory reached 500~MB at
$n=2048$, consistent with $O(n^2)$ storage. Analyses of well-log, array-CGH, and
methylation data provide executed fits under the stated models; the lack of
independently scored changepoints limits conclusions about external segmentation
accuracy. Additional simulation and comparison protocols
are specified as planned work and do not contribute numerical findings to this
report.}

\begin{document}
\maketitle
\section{Introduction}
\label{sec:introduction}

Multiple-changepoint analysis represents an ordered series by contiguous segments
within which a sampling-model parameter is constant. It is used in statistical
signal processing, genomics, environmental science, and econometrics. The
statistical task is to estimate the segment count, changepoint locations,
segment parameters, and their uncertainty. Bayesian methods assign posterior
probabilities to segmentations and can average over these alternatives when
estimating a piecewise-constant signal or predicting future observations
\citep{barry1992ppm,barry1993bayesCP,fearnhead2006exact,hutter2006bpcr}.

For an offline changepoint model with independent segment parameters and a
factorizing partition prior, the probability of a partition is a product of
integrated segment likelihoods and local prior terms. Dynamic programming
then sums these products without enumerating every partition
\citep{barry1992ppm,fearnhead2006exact,hutter2006bpcr}. This formulation is
well established. Its practical use across heterogeneous outcomes depends on
calculating segment marginal likelihoods with the correct sampling-model
normalization and integrating the partition prior with the same care.

\BayesBreak develops a generalized hierarchical formulation of this model.
For regular exponential-family observations with proper conjugate priors,
segment marginal likelihoods and posterior moments follow from cumulative
sufficient statistics. The resulting partition recursions do not depend on
which member of the family supplied the segment scores. The framework includes
nonuniform partitions on irregular physical coordinates, several sequences
sharing changepoints, known groups, and a finite latent-group template
criterion. Nonconjugate segment integrals can also be approximated locally,
but exact posterior interpretation of the resulting dynamic program requires
control of the segment-level approximation error.

\paragraph{Contributions.}
The contributions of this paper are as follows.
\begin{enumerate}
\item We specify a common conjugate segment calculation for exponential-family
  sampling models with known exposure, trial-count, and precision descriptors,
  explicitly distinguishing ordinary marginal likelihoods from powered
  likelihood integrals. The calculation provides segment marginal likelihoods
  and the posterior moments required for Bayesian piecewise-constant regression.
\item We derive normalized sum-product recursions for the segment-count
  posterior, ordered and unordered changepoint marginals, and posterior signal
  moments. We also derive a separate max-sum recursion for joint MAP
  segmentation, with immediate proofs of the identities and decision rules.
\item We extend the local-factor model to physical segment lengths and
  candidate-boundary weights, and to the exact pooling of independently sampled
  sequences under a common segmentation.
\item We study known-group pooling and a finite latent-template criterion for
  unknown groups. A Jensen minorization leads to alternating allocation and
  responsibility-weighted optimal-partition updates, with monotonicity for the
  specified objective rather than a claim of full Bayesian mixture inference.
\item We analyze the effect of segment marginal-likelihood approximation error
  on partition probabilities and posterior odds. The analysis is conditional
  on a uniform bound across admissible segments, not on a numerical method's
  name or an unverified finite-sample error estimate.
\end{enumerate}

The experiments examine simulation recovery, posterior calibration, numerical
approximations, computation time, and applications to well-log measurements,
array-CGH, and methylation. Existing results are reported
separately from the additional experiments specified in
Section~\ref{sec:experimental-design}. The latter are planned evaluations and
are not used to support conclusions about accuracy or comparative performance.

Section~\ref{sec:related} reviews the relevant statistical literature.
Section~\ref{sec:setup} defines the sampling model, partition, and inferential
targets. Sections~\ref{sec:model} and~\ref{sec:exact-inference} derive segment
integration and exact partition inference. Section~\ref{sec:extensions}
covers hierarchical extensions. Sections~\ref{sec:approx-predict} and
\ref{sec:algorithms} address numerical approximation, prediction, and
implementation. The experimental protocols and results follow in
Sections~\ref{sec:experimental-design} and~\ref{sec:results}.

\section{Statistical foundations and related work}
\label{sec:related}

\subsection{Exponential families and conjugate segment models}
The exponential-family representation separates a finite-dimensional sufficient statistic from the log normalizer of the sampling distribution. Diaconis and Ylvisaker characterize conjugate priors for natural exponential families and the hyperparameter conditions under which those priors are proper \citep{diaconis1979conjugate}. In a changepoint model, the same calculation is applied separately to every candidate segment. BayesBreak uses this standard conjugate theory; its contribution is to formulate the resulting segment marginal likelihoods and posterior moment numerators so that they can be reused by the same partition recursions and by the hierarchical extensions developed below. Exactness requires the stated regularity, propriety, and finiteness conditions. The method does not assert closed-form integration for arbitrary nonconjugate parameterizations.

\subsection{Product partition models and exact Bayesian recursions}
Product partition models assign a partition probability proportional to a product of segment cohesions and combine this prior with segmentwise integrated likelihoods \citep{barry1992ppm,barry1993bayesCP}. This is the principal probabilistic lineage of the present work. Fearnhead developed exact recursions and posterior simulation for a broad class of Bayesian multiple-changepoint models with conditionally independent segments \citep{fearnhead2006exact}; Hutter derived exact Bayesian regression for piecewise-constant functions, including posterior break probabilities and model-averaged regression curves \citep{hutter2006bpcr}. Dependence between adjacent segments can sometimes be accommodated by enlarging the recursion state \citep{fearnhead2011dependence}. Bayesian Blocks provides a related local-fitness and dynamic-programming construction for astronomical time series \citep{scargle2013bayesianblocks}.

BayesBreak does not claim priority for product partition models, conjugate segment integration, exact changepoint recursions, posterior changepoint probabilities, or Bayesian piecewise-constant regression. It develops a generalized hierarchical model in which these established calculations are stated with common notation across exponential-family segment models, irregular-design priors, multiple sequences, observed groups, and a finite latent-group extension. It also keeps posterior marginalization and joint MAP optimization as separate recursions rather than treating one as an approximation to the other.

\subsection{Frequentist segmentation and optimization}
Dynamic programming also underlies frequentist segmentation with additive segment costs \citep{auger1989segment,jackson2005optpart}. Penalized formulations and pruning lead to PELT and related pruned exact algorithms \citep{killick2012pelt,rigaill2010pruned}. Wild binary segmentation and SMUCE pursue different detection and uncertainty goals \citep{fryzlewicz2014wbs,frick2014smuce}. Total variation and the fused lasso replace discrete partition search with convex regularization \citep{rudin1992rof,tibshirani2005fused}; the group fused lasso encourages aligned changes across several sequences \citep{bleakley2011groupfused}. Circular binary segmentation is a standard comparator for array-CGH data \citep{olshen2004cbs}. These methods are important computational and empirical references, but their estimands are penalized or testing-based rather than the posterior distribution over partitions considered here.

\subsection{Irregular designs and design-dependent partition priors}
Generalized product partition models and product partition models with covariates show how covariate information can alter partition probabilities through cohesion functions \citep{park2010gppm,muller2011ppmx}. BayesBreak specializes this idea to ordered observations. It distinguishes a segment-span cohesion from a candidate-boundary hazard because they encode different prior assumptions. Both factors can be included in the ordered-partition probability when they remain local to a segment or transition. Neither is treated as a likelihood weight.

\subsection{Hierarchical and multi-sequence changepoint models}
Hierarchical Bayesian changepoint analysis has a long history \citep{carlin1992hierBayesCP}. BASIC models the tendency of changepoints to co-occur across multiple sequences through a hierarchical empirical-Bayes prior and provides posterior sampling and optimization algorithms \citep{fan2017basic}. Joint random partition models provide a Bayesian nonparametric construction for dependent partitions across several ordered processes \citep{quinlan2024jrpm}. BayesBreak studies a more restrictive setting for exact finite computation: aligned sequences are conditionally independent given a common partition, so their segment marginal likelihoods multiply. Known groups apply this calculation separately by group. The finite-candidate consistency theorem in Section~\ref{sec:extensions} concerns this common-partition model and does not imply consistency for arbitrary hierarchical dependence.

\subsection{Latent groups and common structural changes}
Model-based clustering of time-dependent observations according to common structural changes is an active neighboring literature \citep{corradin2026commonstructural}. The BayesBreak latent-group construction is narrower than a full probability mixture over all partitions. It optimizes a finite objective of the form $\sum_s\log\sum_g\pi_g S_s(\tau_g)$, where $S_s(\tau_g)>0$ is the specified subject--template score. A Jensen minorization yields auxiliary weights and responsibility-weighted max-sum updates. Because the $S_s$ need not be normalized densities over the data space, the algorithm is described as minorization--maximization for a finite latent-group objective, not as maximum likelihood for an identifiable finite mixture \citep{dempster1977em,teicher1963identifiability}.

\subsection{Nonconjugate segment calculations}
When a segment integral is unavailable in closed form, deterministic approximations such as Laplace approximation, the Jaakkola--Jordan variational bound, expectation propagation, P\'olya--Gamma augmentation, or numerical quadrature may be used \citep{jaakkola2000logisticvb,minka2001ep,polson2013pg}. The ordered-partition recursion is unchanged after an approximate segment marginal-likelihood table is supplied, but its output is then exact only for that table. Section~\ref{sec:approx-predict} therefore proves a conditional error-propagation result: a uniform bound on reachable segment log marginal likelihoods yields bounds on partition marginal likelihoods and posterior odds. The theorem does not supply the segmentwise bound for a particular approximation routine.

\subsection{Position of the present work}
The present contribution is a generalized hierarchical Bayesian segmentation model whose observation-family dependence is confined to segment marginal likelihoods and posterior moments. The paper derives the common posterior recursions, the separate MAP recursion, the irregular-design prior factors, exact shared-boundary pooling, the finite latent-group optimization, and the conditional approximation bounds in one notation. Its empirical claims concern the executed simulations and case studies reported here; it does not claim a universal complexity improvement or superiority over every established changepoint method.

\section{Problem formulation and notation}
\label{sec:setup}

\subsection{Observations, coordinates, and admissible partitions}
Let $\{(x_i,y_i,d_i)\}_{i=1}^n$ be observations with ordered design points
$x_1<\cdots<x_n$. The response $y_i$ may be continuous or discrete; $d_i$
contains known observation-model quantities, such as a Poisson exposure, a
Binomial trial count, or a Gaussian precision. A separate indicator $o_i\in
\{0,1\}$ specifies whether the response is observed. Missing outcomes do not
contribute likelihood factors. We write $y_{\mathrm{obs}}$ for the observed
responses. The default analysis uses ordinary sampling likelihoods. When
likelihood powers $a_i\geq0$ are deliberately introduced, they define a
\emph{generalized Bayesian} calculation and are not interpreted as exposure,
trial count, or physical spacing.

Candidate partition edges are $u_0<u_1<\cdots<u_n$, with interior edge $u_i$
separating observations $i$ and $i+1$. For irregular observations, an admissible
choice is $u_i=(x_i+x_{i+1})/2$ for $1\leq i<n$, together with prespecified
endpoints $u_0<x_1$ and $u_n>x_n$. Where the observation domain has known
physical limits, use those limits instead. Choosing these edges is part of the
model specification. On an index grid, use $u_i=i$. These physical edges are
distinct from the observed positions $x_i$ and from the family descriptor $d_i$.

A partition into $k$ nonempty contiguous segments has ordered boundary vector
\begin{equation}
 t=(t_0,\ldots,t_k),\qquad 0=t_0<t_1<\cdots<t_k=n,
 \qquad I_q(t)=\{t_{q-1}+1,\ldots,t_q\}.
 \label{eq:setup-partition}
\end{equation}
The candidate block $(i,j]$ denotes observations $i+1,\ldots,j$ and has
physical length $u_j-u_i$. Let $\mathcal T_k$ be the set of partitions satisfying
the declared candidate-boundary and minimum-length restrictions. These
restrictions are set independently of the responses. Write $1\leq k\leq
k_{\max}$ for the admissible segment counts.

\subsection{Hierarchical sampling model and prior}
Within segment $q$, observations are conditionally independent given the
segment parameter $\theta_q$:
\begin{equation}
 p(y_{\mathrm{obs}}\mid t,k,\theta_{1:k},d)
 =\prod_{q=1}^k\ \prod_{i\in I_q(t):o_i=1}
 p(y_i\mid\theta_q,d_i),
 \qquad \theta_q\overset{\mathrm{ind}}\sim\pi_0.
 \label{eq:setup-likelihood}
\end{equation}
The segment prior $\pi_0$ is proper and does not depend on the partition unless
explicitly specified. The prior on the count and the conditional prior on the
partition are $p(k)$ and $p(t\mid k,u)$, respectively. For the block-local
factorization used below, write
\begin{equation}
 p(t\mid k,u)=\frac{1}{C_k(u)}\prod_{q=1}^k
 \gamma_u(t_{q-1},t_q),\qquad
 C_k(u)=\sum_{t\in\mathcal T_k}\prod_{q=1}^k
 \gamma_u(t_{q-1},t_q).
 \label{eq:setup-prior}
\end{equation}
The definition assumes $0<C_k(u)<\infty$. Segment cohesion and
candidate-boundary hazards are two components of $\gamma_u$ and are defined in
Section~\ref{sec:model}. For the uniform conditional prior, $\gamma_u\equiv1$
on $\mathcal T_k$ and $C_k=|\mathcal T_k|$; in the unrestricted case,
$C_k=\binom{n-1}{k-1}$.

\subsection{Inference and decision targets}
Define the integrated likelihood of a candidate block by
\begin{equation}
 A^{(0)}_{ij}=\int\!\pi_0(\theta)
 \prod_{r=i+1}^j p(y_r\mid\theta,d_r)^{o_r}\,\mathrm d\theta.
 \label{eq:setup-block}
\end{equation}
Here a missing response contributes the multiplicative identity one; in
particular, a block containing no observed responses has $A^{(0)}_{ij}=1$ under
a proper prior. The count-conditional marginal likelihood and posterior are
\begin{equation}
 p(y_{\mathrm{obs}}\mid k,u)
 =\frac{1}{C_k(u)}\sum_{t\in\mathcal T_k}
 \prod_{q=1}^k \gamma_u(t_{q-1},t_q)
 A^{(0)}_{t_{q-1},t_q},\qquad
 p(k\mid y_{\mathrm{obs}},u)\propto p(k)\,p(y_{\mathrm{obs}}\mid k,u).
 \label{eq:setup-evidence}
\end{equation}
The inferential targets are the full count posterior, ordered-boundary
marginals $p(t_q=h\mid y,k)$, boundary-event probabilities
$p(b_h=1\mid y,k)$, posterior moments of a declared segment signal
$m_\star(\theta)$, and the posterior predictive distribution. The
\emph{marginal-MAP count} maximizes $p(k\mid y)$; a \emph{conditional joint MAP
partition} maximizes $p(t\mid y,\widehat k)$; the \emph{global joint MAP pair}
maximizes $p(k,t\mid y)$. These are different posterior summaries and can
select different partitions. Exact formulas and algorithms are given in
Section~\ref{sec:exact-inference}.

\subsection{Extension to related sequences}
For $S$ aligned sequences, let $y^{(s)}_i$ and $\theta^{(s)}_q$ denote the
subject response and its segment parameter. In the shared-boundary model,
$t$ is common to all $S$ sequences while the $\theta^{(s)}_q$ are independent
given $t$. Known groups have distinct group-specific partitions. For unknown
group labels, Section~\ref{sec:extensions} defines a separate finite-template
score criterion and its minorization--maximization procedure. The latter is
not a posterior distribution over a normalized finite mixture unless an
additional probabilistic model is supplied.

\begin{table}[!htbp]
\centering\small
\caption{Notation used throughout the manuscript.}
\label{tab:notation-paper}
\begin{tabularx}{0.96\textwidth}{@{}L{0.22\textwidth}Y L{0.19\textwidth}Y@{}}
\toprule
Symbol & Definition & Symbol & Definition \\
\midrule
$x_i$, $u_i$ & Observed position; candidate partition edge & $y_i,d_i,o_i$ & Response, family descriptor, observation mask \\
$n,S,G$ & Observations, sequences, latent groups & $k,k_{\max}$ & Segment count and imposed upper bound \\
$t,\mathcal T_k$ & Partition and its admissible set & $I_q(t)$ & Index set of segment $q$ \\
$\theta_q,\pi_0$ & Segment parameter and proper prior & $m_\star(\theta)$ & Reported segment-level signal functional \\
$A^{(0)}_{ij}$ & Integrated block likelihood & $\gamma_u(i,j)$ & Local partition-prior factor \\
$C_k(u)$ & Conditional partition-prior normalizer & $L_{k,j},R_{k,i}$ & Forward and backward partition sums \\
$M_{k,j}$ & Max-sum optimal prefix score & $p(b_h=1\mid y,k)$ & Conditional boundary-event probability \\
\bottomrule
\end{tabularx}
\end{table}

\section{Generalized segment models and integrated likelihoods}
\label{sec:model}

\subsection{Segment likelihood and support}
The observations, candidate boundary edges, admissible partitions, and primary
inferential targets are defined in Section~\ref{sec:setup}. For an admissible
candidate block $(i,j]$, the segment integral is evaluated under its stated
sampling model and proper prior. A structurally disallowed block has zero
prior weight; a permitted block with missing observations is integrated over
its observed outcomes only. These cases must not be conflated.

\begin{assumption}[Offline structural support and block independence]
\label{ass:offline}
The candidate segment-count set $\{1,\ldots,k_{\max}\}$ is finite. Structural admissibility of a segment is fixed by the candidate grid, declared length restrictions, and prior support before the response values are observed. A structurally admissible segment has finite nonnegative local evidence; zero evidence removes its likelihood contribution but does not alter the partition-prior normalizer. Structurally excluded segments have local prior factor zero and log-score $-\infty$. Conditional on a partition, segment parameters are independent across segments. The conditional partition prior factorizes through local terms as in Assumption~\ref{ass:factorized-prior}. For approximate segment routines, the same statements define an exact posterior for the supplied approximate score array; comparison to the intended nonconjugate model additionally requires Assumption~\ref{ass:block-error}.
\end{assumption}

\subsection{Exponential-family segment likelihoods with exposure and power weights}
Let $d_r$ collect scientifically specified family information, such as Poisson exposure, Binomial trial count, or Gaussian precision. Let $a_r\ge0$ be an optional power or reliability weight that multiplies the complete log contribution. The canonical local likelihood is
\begin{equation}
\begin{aligned}
\log p(y_r\mid\theta,d_r)^{a_r}
&=a_r\log h(y_r,d_r)\\
&\quad+a_rS(y_r,d_r)^\top\eta(\theta)\\
&\quad-a_rE(d_r)A(\theta).
\end{aligned}
\label{eq:local-ef}
\end{equation}
The corresponding block summaries are
\begin{align}
S_{ij}&=\sum_{r=i+1}^{j}a_rS(y_r,d_r),\nonumber\\
W_{ij}&=\sum_{r=i+1}^{j}a_rE(d_r),\nonumber\\
H_{ij}&=\sum_{r=i+1}^{j}a_r\log h(y_r,d_r).
\label{eq:block-summaries}
\end{align}
This notation prevents three distinct quantities from being collapsed into a single scalar ``weight'': design spacing affects the partition prior, $d_r$ changes the observation law, and $a_r$ raises the entire observation contribution to a power.

For conjugate hyperparameters $(\alpha,\beta)$, use the Diaconis--Ylvisaker normalizer convention
\begin{align}
p(\theta\mid\alpha,\beta)
&=\exp\{\eta(\theta)^\top\alpha-\beta A(\theta)-\log Z(\alpha,\beta)\},\nonumber\\
Z(\alpha,\beta)
&=\int e^{\eta(\theta)^\top\alpha-\beta A(\theta)}\dd\theta.
\label{eq:conjugate-prior}
\end{align}
Fix a scientifically declared reference descriptor $d_\star$ when the conditional observation mean depends on exposure or trial structure, and write $m_\star(\theta)=m(\theta;d_\star)$ for the segment functional reported by the analysis (for example, a Poisson rate at unit exposure rather than an exposure-specific count mean). Define
\begin{align*}
A^{(0)}_{ij}&=\int p(\theta\mid\alpha_0,\beta_0)
\prod_{r=i+1}^{j}p(y_r\mid\theta,d_r)^{a_r}\dd\theta,\\
A^{(q)}_{ij}&=A^{(0)}_{ij}\,\mathbb E[m_\star(\theta)^q\mid y_{i+1:j},d,a],\qquad q=1,2.
\end{align*}

For unit powers $a_r=1$, $A^{(0)}_{ij}$ is the ordinary block marginal likelihood. For non-unit powers it is a powered/generalized-Bayes integral rather than, in general, a probability-normalized likelihood of the observations. Posterior moments in the second display are taken with respect to the corresponding normalized posterior kernel, which may be a generalized posterior.

\begin{theorem}[Exposure-aware conjugate block integration]
\label{thm:block-integration}
If $Z(\alpha_0+S_{ij},\beta_0+W_{ij})<\infty$, then
\begin{equation}
 A^{(0)}_{ij}=e^{H_{ij}}
 \frac{Z(\alpha_0+S_{ij},\beta_0+W_{ij})}{Z(\alpha_0,\beta_0)}.
\label{eq:block-evidence}
\end{equation}
The block posterior is conjugate with hyperparameters
$(\alpha_0+S_{ij},\beta_0+W_{ij})$. Whenever its $r$th moment exists,
\begin{equation}
 A^{(r)}_{ij}=A^{(0)}_{ij}
 \mathbb E[m_\star(\theta)^r\mid\alpha_0+S_{ij},\beta_0+W_{ij}].
\label{eq:block-moment}
\end{equation}
\end{theorem}
\begin{proof}
\textbf{Step 1 (aggregate the powered sampling contributions).}
By conditional independence and Equation~\eqref{eq:local-ef}, multiplication of the local factors yields
\[
\prod_{r=i+1}^{j}p(y_r\mid\theta,d_r)^{a_r}
=\exp\{H_{ij}+\eta(\theta)^\top S_{ij}-W_{ij}A(\theta)\}.
\]
The three sums in Equation~\eqref{eq:block-summaries} collect the base-measure term, the sufficient-statistic coefficient, and the coefficient of the log-partition function, respectively.
\textbf{Step 2 (combine likelihood and prior).}
Multiplication by Equation~\eqref{eq:conjugate-prior} yields
\[
\frac{e^{H_{ij}}}{Z(\alpha_0,\beta_0)}
\exp\!\{\eta(\theta)^\top(\alpha_0+S_{ij})
-(\beta_0+W_{ij})A(\theta)\}.
\]
This is a constant times the conjugate kernel with the displayed updated hyperparameters.
\textbf{Step 3 (integrate).}
The assumption of a finite updated normalizer permits integrating the kernel. Its integral is $Z(\alpha_0+S_{ij},\beta_0+W_{ij})$, proving Equation~\eqref{eq:block-evidence}. Division of the joint density by this integral gives the posterior conjugate distribution.
\textbf{Step 4 (moment insertion).}
For any integrable $m_\star(\theta)^r$, insert that factor before integration. Dividing and multiplying the resulting integral by $A^{(0)}_{ij}$ identifies the updated posterior expectation, proving Equation~\eqref{eq:block-moment}. For $a_r\ne1$ the algebra is exact for a powered likelihood, but the result is not in general a normalized ordinary sampling-model marginal likelihood.
\end{proof}

\begin{remark}[Scope of exactness]
Theorem~\ref{thm:block-integration} is exact for the displayed powered, exposure-aware model. It does not permit a Binomial trial count to be treated as a likelihood power, an irregular coordinate gap to be inserted as exposure, or a numerical integral to be described as conjugate closed form.
\end{remark}

\subsection{Local prior factors}
\begin{assumption}[Factorized ordered-partition prior]
\label{ass:factorized-prior}
For fixed $k$ and boundary coordinates $u$, the prior on $t$ is
\begin{equation}
 p(t\mid k,u)=C_k(u)^{-1}
 \prod_{q=1}^{k}c_u(t_{q-1},t_q)
 \prod_{q=1}^{k-1}h_u(t_q),
\label{eq:partition-prior}
\end{equation}
where $c_u(i,j)\ge0$ is a complete-segment cohesion and $h_u(j)\ge0$ is an interior-boundary hazard. Define
$\gamma_u(i,j)=c_u(i,j)h_u(j)^{\1\{j<n\}}$ and
$\widetilde A^{(r)}_{ij}=A^{(r)}_{ij}\gamma_u(i,j)$.
\end{assumption}
The normalizer $C_k(u)$ is itself a partition sum and can be computed by the same dynamic program after replacing all likelihood evidences by one. This ensures that fixed-$k$ evidence remains properly normalized for nonuniform local priors.

\begin{table}[!htbp]
\centering
\small
\caption{Supported local block constructions. Closed-form rows instantiate Theorem~\ref{thm:block-integration}; the Beta-valued response uses deterministic one-dimensional quadrature and is therefore exact only up to its numerical tolerance.}
\label{tab:block-families}
\begin{tabularx}{0.97\textwidth}{@{}L{0.15\textwidth}L{0.19\textwidth}L{0.17\textwidth}Y L{0.16\textwidth}@{}}
\toprule
Family & Descriptor $d_r$ & Segment parameter & Integrated block construction & Evaluation \\
\midrule
Gaussian, known variance & precision or inverse variance & mean $\mu$ & Gaussian--Gaussian normalizer and posterior moments & closed form \\
Poisson & exposure $e_r$ & rate $\lambda$ & Gamma--Poisson normalizer; rate $e_r\lambda$ & closed form \\
Binomial & trial count $m_r$ & probability $p$ & Beta--Binomial normalizer; $m_r$ trials & closed form \\
Negative Binomial & fixed dispersion & success parameter & Beta-conjugate distributional parameterization & closed form \\
Beta-valued response & known precision $\phi_r$ & mean $\mu\in(0,1)$ & one-dimensional integral over a transformed mean & quadrature \\
Nonconjugate GLM & link-specific inputs & regression parameter & Laplace, variational, EP, PG, or quadrature block score & quadrature or local approximation \\
\bottomrule
\end{tabularx}
\end{table}

\subsection{Segment marginal-likelihood interface}
The partition recursions require only a triangular array of segment log marginal likelihoods and, when posterior signal moments are requested, corresponding segment moment summaries. A family-specific implementation returns the segment log marginal likelihood, posterior hyperparameters or moments, a validity flag, and any numerical diagnostic needed to interpret an approximate value. Prefix summaries make Gaussian, Poisson, and Binomial segment calculations constant time after $\mathcal O(n)$ preprocessing. The recursion over partitions is therefore unchanged when the observation family changes.

\paragraph{Implementation interface.}
For each admissible block, the local calculation supplies
$\log A^{(0)}_{ij}$, the conditional posterior moments required for the
scientific target, and convergence information if the integral is evaluated
numerically. A structurally excluded block has prior-adjusted log score
$-\infty$. The partition dynamic program uses the resulting score table
without reference to the observation family's internal parameterization.

\subsection{Closed-form segment distributions and posterior moments}
\label{sec:family-derivations}
The following calculations make the model-specific part of Theorem~\ref{thm:block-integration} explicit. Throughout this subsection, observations have their ordinary sampling interpretation ($a_r=1$). An exposure, precision, or trial count is part of the sampling distribution, not an exponent applied to that distribution. The target $m_\star(\theta)$ is specified in each model rather than tacitly identified with an exposure-dependent observation mean.

\subsubsection{Gaussian observations with known precisions}
\begin{proposition}[Gaussian evidence and first two signal moments]
\label{prop:gaussian-family}
Suppose $y_r\mid\mu,w_r\sim\mathcal N(\mu,\sigma^2/w_r)$ independently, with known $w_r>0$, and $\mu\sim\mathcal N(\nu,\rho^2)$, where $\sigma^2,\rho^2>0$. On $B=(i,j]$ let
\begin{align*}
N_B&=j-i, & W_B&=\sum_{r\in B}w_r,
& \bar y_B&=W_B^{-1}\sum_{r\in B}w_ry_r,\\
Q_B&=\sum_{r\in B}w_r(y_r-\bar y_B)^2,
& V_B&=(\rho^{-2}+W_B/\sigma^2)^{-1},
& \nu_B&=V_B(\nu/\rho^2+W_B\bar y_B/\sigma^2).
\end{align*}
Then $\mu\mid y_B\sim\mathcal N(\nu_B,V_B)$ and
\begin{align}
\log A^{(0)}_B={}&-\frac{N_B}{2}\log(2\pi\sigma^2)
+\frac12\sum_{r\in B}\log w_r-\frac{Q_B}{2\sigma^2}
+\frac12\log\frac{\sigma^2}{\sigma^2+\rho^2W_B}
-\frac{W_B(\bar y_B-\nu)^2}{2(\sigma^2+\rho^2W_B)}.
\label{eq:gaussian-evidence-derived}
\end{align}
With $m_\star(\mu)=\mu$, the posterior signal moments are $\nu_B$ and $V_B+\nu_B^2$.
\end{proposition}
\begin{proof}
\textbf{Step 1 (separate residual and mean information).}
Completing the weighted square gives the exact identity
\[
\sum_{r\in B}w_r(y_r-\mu)^2
=Q_B+W_B(\mu-\bar y_B)^2,
\]
so the block likelihood is
\[
(2\pi\sigma^2)^{-N_B/2}\Bigl(\prod_{r\in B}w_r^{1/2}\Bigr)
\exp\!\left\{-\frac{Q_B}{2\sigma^2}
-\frac{W_B(\mu-\bar y_B)^2}{2\sigma^2}\right\}.
\]
\textbf{Step 2 (derive the posterior).}
Multiplying by the normal prior and collecting coefficients of $\mu^2$ and $\mu$ yields precision $\rho^{-2}+W_B/\sigma^2$ and precision-weighted mean $\nu/\rho^2+W_B\bar y_B/\sigma^2$. Hence the posterior is $\mathcal N(\nu_B,V_B)$.
\textbf{Step 3 (integrate the segment mean).}
The remaining normal integral is
\[
\int\!\exp\!\left\{-\frac{W_B(\mu-\bar y_B)^2}{2\sigma^2}\right\}
\mathcal N(\mu\mid\nu,\rho^2)\dd\mu
=\left(\frac{\sigma^2}{\sigma^2+\rho^2W_B}\right)^{1/2}
\exp\!\left\{-\frac{W_B(\bar y_B-\nu)^2}{2(\sigma^2+\rho^2W_B)}\right\}.
\]
Multiplying the constants from Step~1 and taking logarithms gives Equation~\eqref{eq:gaussian-evidence-derived}. The first two normal moments follow from the posterior mean and variance. In particular, $\nu_B$ may be negative: its contribution must never be stored as the logarithm of a signed moment numerator.
\end{proof}

\subsubsection{Poisson counts with genuine exposure}
\begin{proposition}[Poisson--Gamma integration]
\label{prop:poisson-family}
If $y_r\mid\lambda,e_r\sim\operatorname{Pois}(e_r\lambda)$ with known $e_r>0$ and $\lambda\sim\operatorname{Gamma}(a_0,b_0)$ in shape--rate form, define $Y_B=\sum_{r\in B}y_r$, $E_B=\sum_{r\in B}e_r$ and $H_B=\sum_{r\in B}(y_r\log e_r-\log\Gamma(y_r+1))$. Then
\begin{equation}
A^{(0)}_B=e^{H_B}\frac{b_0^{a_0}\Gamma(a_0+Y_B)}{\Gamma(a_0)(b_0+E_B)^{a_0+Y_B}},
\label{eq:poisson-evidence-derived}
\end{equation}
and $\lambda\mid y_B\sim\operatorname{Gamma}(a_0+Y_B,b_0+E_B)$. For unit-exposure rate $m_\star(\lambda)=\lambda$, the first two posterior moments are $a_B/b_B$ and $a_B(a_B+1)/b_B^2$, with $(a_B,b_B)=(a_0+Y_B,b_0+E_B)$.
\end{proposition}
\begin{proof}
\textbf{Step 1 (collect the ordinary Poisson likelihood).}
The product of probability mass functions is
\[
\prod_{r\in B}\frac{e^{-e_r\lambda}(e_r\lambda)^{y_r}}{y_r!}
=e^{H_B}\lambda^{Y_B}e^{-E_B\lambda}.
\]
The coefficient of $\log\lambda$ is the observed count $Y_B$, whereas the coefficient of $-\lambda$ is the exposure $E_B$; replacing both by a common likelihood-power weight would be a different model.
\textbf{Step 2 (combine with the Gamma prior).}
The prior density is $b_0^{a_0}\lambda^{a_0-1}e^{-b_0\lambda}/\Gamma(a_0)$. Multiplication gives a Gamma kernel with shape $a_0+Y_B$ and rate $b_0+E_B$.
\textbf{Step 3 (integrate and extract moments).}
Apply $\int_0^\infty\lambda^{a_B-1}e^{-b_B\lambda}\dd\lambda=\Gamma(a_B)/b_B^{a_B}$ to obtain Equation~\eqref{eq:poisson-evidence-derived}. Gamma moment identities give the stated first and second moments.
\end{proof}

\subsubsection{Binomial counts with known trial numbers}
\begin{proposition}[Beta--Binomial integration]
\label{prop:binomial-family}
If $y_r\mid p,m_r\sim\operatorname{Binom}(m_r,p)$ with $m_r\in\mathbb N_0$, $0\le y_r\le m_r$, and $p\sim\operatorname{Beta}(a_0,b_0)$, set $Y_B=\sum y_r$, $M_B=\sum m_r$, and $H_B=\sum\log\binom{m_r}{y_r}$. Then
\begin{equation}
A^{(0)}_B=e^{H_B}\frac{B(a_0+Y_B,b_0+M_B-Y_B)}{B(a_0,b_0)},
\label{eq:binomial-evidence-derived}
\end{equation}
with posterior $\operatorname{Beta}(a_B,b_B)$, $a_B=a_0+Y_B$, $b_B=b_0+M_B-Y_B$. For $m_\star(p)=p$, the posterior first two moments are $a_B/(a_B+b_B)$ and $a_B(a_B+1)/[(a_B+b_B)(a_B+b_B+1)]$.
\end{proposition}
\begin{proof}
\textbf{Step 1 (likelihood factorization).}
Independence and the Binomial mass function give
\[
\prod_{r\in B}\binom{m_r}{y_r}p^{y_r}(1-p)^{m_r-y_r}
=e^{H_B}p^{Y_B}(1-p)^{M_B-Y_B}.
\]
\textbf{Step 2 (integrate with the Beta prior).}
Multiplication by $p^{a_0-1}(1-p)^{b_0-1}/B(a_0,b_0)$ gives a $\operatorname{Beta}(a_B,b_B)$ kernel. Integration over $p\in(0,1)$ yields the ratio of Beta functions in Equation~\eqref{eq:binomial-evidence-derived}.
\textbf{Step 3 (posterior signal).}
The first two Beta moments give the result. The trial number $m_r$ belongs to the Binomial law, not to an independent power weight; $m_r=0$ and $y_r=0$ is a legitimate no-trial observation with likelihood factor one.
\end{proof}

\subsubsection{Negative-Binomial counts with fixed dispersion}
\begin{proposition}[Beta--negative-Binomial integration and signal choice]
\label{prop:negative-binomial-family}
Use the convention
\[
p(y_r\mid p,r_r)=\frac{\Gamma(y_r+r_r)}{\Gamma(r_r)\Gamma(y_r+1)}
 p^{r_r}(1-p)^{y_r},\qquad r_r>0,
\]
whose observation mean is $r_r(1-p)/p$. With $p\sim\operatorname{Beta}(a_0,b_0)$, let $R_B=\sum r_r$, $Y_B=\sum y_r$, and $H_B=\sum\log[\Gamma(y_r+r_r)/(\Gamma(r_r)\Gamma(y_r+1))]$. Then
\begin{equation}
A^{(0)}_B=e^{H_B}\frac{B(a_0+R_B,b_0+Y_B)}{B(a_0,b_0)}.
\label{eq:negative-binomial-evidence-derived}
\end{equation}
The segment posterior is $p\mid y_B\sim\operatorname{Beta}(\alpha_B,\beta_B)$ with $(\alpha_B,\beta_B)=(a_0+R_B,b_0+Y_B)$. The parameter-scale signal $m_\star(p)=p$ has ordinary Beta moments. Alternatively, the per-unit-dispersion count-mean functional $u(p)=(1-p)/p$ has
\[
\mathbb E[u(p)\mid y_B]=\frac{\beta_B}{\alpha_B-1}\quad(\alpha_B>1),\qquad
\mathbb E[u(p)^2\mid y_B]=\frac{\beta_B(\beta_B+1)}{(\alpha_B-1)(\alpha_B-2)}\quad(\alpha_B>2).
\]
\end{proposition}
\begin{proof}
\textbf{Step 1 (sufficient statistics).}
Multiplying the negative-Binomial mass functions collects the factor $e^{H_B}p^{R_B}(1-p)^{Y_B}$.
\textbf{Step 2 (conjugate integration).}
Combining with the Beta prior adds $R_B$ to its first shape and $Y_B$ to its second shape, so the posterior is $\operatorname{Beta}(\alpha_B,\beta_B)$. Its integrated normalizing-constant ratio is Equation~\eqref{eq:negative-binomial-evidence-derived}.
\textbf{Step 3 (distinguish the scientific signals).}
For $u(p)=(1-p)/p$, integrate $p^{\alpha_B-2}(1-p)^{\beta_B}$ and $p^{\alpha_B-3}(1-p)^{\beta_B+1}$ against the Beta normalization. The resulting Beta-function ratios simplify to the displayed moments. Their finiteness requires $\alpha_B>1$ and $\alpha_B>2$, respectively. If the reported signal is instead $p$, no reciprocal-moment conditions arise. A target must be chosen consistently before computing Bayes curves.
\end{proof}

\subsubsection{Beta-valued measurements and one-dimensional integration}
\begin{proposition}[Beta-response evidence and posterior curvature]
\label{prop:beta-response-family}
Let $y_r\in(0,1)$, $y_r\mid\mu,\phi_r\sim\operatorname{Beta}(\phi_r\mu,\phi_r(1-\mu))$ with known $\phi_r>0$, and $\mu\sim\operatorname{Beta}(a_0,b_0)$. Defining
\[
\Psi_B(\mu)=\log\operatorname{BetaPDF}(\mu\mid a_0,b_0)
+\sum_{r\in B}\log\operatorname{BetaPDF}(y_r\mid\phi_r\mu,\phi_r(1-\mu)),
\]
the segment integrals are $A^{(s)}_B=\int_0^1\mu^s e^{\Psi_B(\mu)}\dd\mu$ for $s=0,1,2$. The log likelihood is strictly concave in $\mu\in(0,1)$ for every $\phi_r>0$; if also $a_0,b_0\ge1$, the posterior log density is concave. Nevertheless, fixed-order quadrature has no universal accuracy guarantee for all admissible input data.
\end{proposition}
\begin{proof}
\textbf{Step 1 (derive the integral).}
The Beta response is not conjugate to a Beta prior on its mean. Multiplying the response densities by the prior gives $e^{\Psi_B(\mu)}$, and direct integration with factors $\mu^s$ gives the required evidence and unnormalized signal moments.
\textbf{Step 2 (differentiate the likelihood).}
Writing $\psi_1$ for the trigamma function, the terms involving $\mu$ through log-Gamma normalizers give
\[
\frac{\dd^2}{\dd\mu^2}\log p(y_B\mid\mu,\phi_B)
=-\sum_{r\in B}\phi_r^2\{\psi_1(\phi_r\mu)+\psi_1(\phi_r(1-\mu))\}<0.
\]
All $\phi_r$ are strictly positive and $\psi_1$ is positive on positive arguments. The prior contributes $-(a_0-1)/\mu^2-(b_0-1)/(1-\mu)^2$, which is nonpositive when $a_0,b_0\ge1$.
\textbf{Step 3 (numerical consequence).}
Log concavity alone does not bound fixed-grid integration error near either endpoint. Evaluating log-integrands after removing their maximum and comparing multiple quadrature orders with an adaptive reference supplies a practical convergence check; failed convergence must be reported rather than silently accepted.
\end{proof}

\paragraph{Missing observations.}
For any family above, missingness means omitting its likelihood factor by an explicit observation mask, not setting its descriptor to zero. A candidate segment with no observed values for one subject has integrated contribution one under that subject's proper prior; globally excluded segments remain controlled by the partition support instead.

\section{Posterior computation over ordered partitions}
\label{sec:exact-inference}

Once the prior-adjusted block array
$\widetilde A^{(r)}_{ij}=A^{(r)}_{ij}\gamma_u(i,j)$ has been constructed, the observation family no longer enters the partition recursions. This section gives the two distinct inference routes used by \BayesBreak: sum-product for posterior quantities and max-sum for a joint maximizing partition.

\begin{figure}[!htbp]
  \centering
  \resizebox{0.91\textwidth}{!}{\begin{tikzpicture}[node distance=8mm and 13mm]
\node[secondarybox,text width=30mm] (scores) {Integrated segment log scores\\$\ell_{ij}$};
\node[primarybox,above right=of scores,text width=28mm] (sum) {Sum-product\\$\log\sum\exp$};
\node[primarybox,below right=of scores,text width=28mm] (map) {Max-sum\\$\max$};
\node[neutralbox,right=of sum,text width=34mm] (posterior) {$p(k\mid y)$, boundary marginals, signal moments};
\node[neutralbox,right=of map,text width=34mm] (maxpath) {Joint MAP partition with backtracking};
\draw[arr] (scores)--(sum);
\draw[arr] (scores)--(map);
\draw[arr] (sum)--(posterior);
\draw[arr] (map)--(maxpath);
\end{tikzpicture}}
  \caption{The same local scores support two algebraically related but inferentially different routes. Sum-product integrates over paths; max-sum selects one globally maximizing path. Coordinatewise modes of marginal boundary distributions need not form that path.}
  \label{fig:semiring-routes}
\end{figure}

\subsection{Forward and backward partition sums}
For $k\in\{0,\ldots,k_{\max}\}$ and $j\in\{0,\ldots,n\}$, let $L_{k,j}$ be the integrated score of all $k$-block partitions of the prefix $(0,j]$. Let $R_{k,i}$ be the corresponding score for all $k$-block partitions of the suffix $(i,n]$. Set
\begin{align*}
L_{0,0}&=1,\qquad L_{0,j}=0\quad(j>0),\\
R_{0,n}&=1,\qquad R_{0,i}=0\quad(i<n).
\end{align*}
The sum-product recursions are
\begin{align}
L_{k+1,j}&=\sum_{h=k}^{j-1}L_{k,h}\widetilde A^{(0)}_{h,j},
\label{eq:forward}\\
R_{k+1,i}&=\sum_{h=i+1}^{n-k}\widetilde A^{(0)}_{i,h}R_{k,h}.
\label{eq:backward}
\end{align}
In implementation, Equations~\eqref{eq:forward}--\eqref{eq:backward} are evaluated in log space using \texttt{logsumexp}.

\begin{theorem}[Partition-sum identities and boundary marginals]
\label{thm:sum-product}
Under Assumptions~\ref{ass:offline} and~\ref{ass:factorized-prior}, for every admissible $k$,
\begin{equation}
L_{k,n}=R_{k,0}
=\sum_{0=t_0<\cdots<t_k=n}\prod_{q=1}^{k}\widetilde A^{(0)}_{t_{q-1},t_q}.
\label{eq:partition-sum}
\end{equation}
Moreover, for $1\le p\le k-1$ and admissible $h$,
\begin{align}
p(t_p=h\mid y,k)
&=\frac{L_{p,h}R_{k-p,h}}{L_{k,n}},\nonumber\\
p(b_h=1\mid y,k)
&=\sum_{p=1}^{k-1}p(t_p=h\mid y,k).
\label{eq:boundary-marginals}
\end{align}
\end{theorem}
\begin{proof}
\textbf{Step 1 (enumerate prefixes).}
For $k=0$ the unique empty path has weight one at $j=0$, and no empty path has a positive endpoint. Assume $L_{k,h}$ sums the local-score products over all $k$-segment partitions of $(0,h]$.
\textbf{Step 2 (split by the last segment).}
Every $(k+1)$-segment partition of $(0,j]$ has exactly one penultimate boundary $h\in\{k,\ldots,j-1\}$. Its weight is the weight of its $k$-segment prefix times $\widetilde A^{(0)}_{h,j}$. Summing over all such $h$ gives Equation~\eqref{eq:forward} and proves the prefix identity by induction.
\textbf{Step 3 (enumerate suffixes).}
The analogous decomposition by the first segment after $i$ gives Equation~\eqref{eq:backward}. At $i=0$, the suffix enumeration covers precisely the same set of full paths as the prefix enumeration at $j=n$. Hence $R_{k,0}=L_{k,n}$, proving Equation~\eqref{eq:partition-sum}.
\textbf{Step 4 (fix an ordered boundary).}
Under the event $t_p=h$, each admissible path is uniquely split into a $p$-segment prefix ending at $h$ and a $(k-p)$-segment suffix beginning there. Multiplication of their total sums gives the unnormalized numerator $L_{p,h}R_{k-p,h}$. The prior factor $C_k^{-1}$ is common to numerator and denominator and cancels, so division by $L_{k,n}$ proves the first boundary identity.
\textbf{Step 5 (forget the boundary rank).}
Within a strictly ordered partition, at most one of the events $\{t_p=h\}$ can occur at a given interior index $h$. They are disjoint as $p$ varies. Their probabilities therefore add, yielding the boundary-event marginal in Equation~\eqref{eq:boundary-marginals}. Normalization over $h$ of each ordered-boundary distribution gives one; summation over $h$ of boundary-event probabilities gives $k-1$.
\end{proof}

Let $p(k)$ be a segment-count prior. Because $C_k(u)$ normalizes Equation~\eqref{eq:partition-prior},
\begin{equation}
 p(k\mid y)\propto p(k)\frac{L_{k,n}}{C_k(u)},
 \qquad k=1,\ldots,k_{\max}.
\label{eq:k-posterior}
\end{equation}
The posterior may be reported in full; choosing $\widehat k=\arg\max_k p(k\mid y)$ is a separate decision rule. Conditional boundary distributions in Equation~\eqref{eq:boundary-marginals} and boundary-event probabilities can then be computed at fixed $k$ or averaged over $p(k\mid y)$.

\subsection{Posterior moments and Bayes curves}
For $r\in\{0,1,2\}$, define a block-covering contribution at fixed $k$,
\begin{equation}
F^{(r)}_{ij}(k)
=\frac{1}{L_{k,n}}
  \sum_{m=1}^{k}L_{m-1,i}\widetilde A^{(r)}_{ij}R_{k-m,j}.
\label{eq:block-cover}
\end{equation}
When the target at coordinate $s$ is the segment-level quantity $m_\star(\theta)$, its conditional posterior moments are
\begin{equation}
\mathbb E\!\left[m_\star(\theta_s)^r\mid y,k\right]
 =\sum_{i<s\le j}F^{(r)}_{ij}(k).
\label{eq:bayes-curve}
\end{equation}
The posterior mean Bayes curve uses $r=1$; the $r=2$ result gives pointwise posterior variance after subtracting the squared mean.

\begin{proposition}[Block-covering decomposition]
\label{prop:block-cover}
If every admissible partition assigns each coordinate to exactly one block and $A^{(r)}_{ij}$ is the integrated $r$th moment numerator for that block, then Equation~\eqref{eq:bayes-curve} is the exact fixed-$k$ posterior moment of the segment-level target at $s$.
\end{proposition}
\begin{proof}
\textbf{Step 1 (uniqueness of the covering block).}
Fix $s\in\{1,\ldots,n\}$. In every admissible partition there is exactly one segment $(i,j]$ for which $i<s\le j$; the corresponding segment position $m$ is also unique.
\textbf{Step 2 (integrate a fixed covering block).}
Conditioning on $(i,j]$ occupying position $m$, all earlier segments contribute the prefix sum $L_{m-1,i}$. The covering block contributes $\widetilde A^{(r)}_{ij}$, because its signal functional is raised to the $r$th power before integration. All later segments contribute $R_{k-m,j}$.
\textbf{Step 3 (sum without duplication).}
Summing over positions $m$ and blocks $(i,j]$ covering $s$ enumerates each full $k$-segment path once, multiplied by its integrated target contribution. This gives the numerator of the posterior expectation.
\textbf{Step 4 (normalize).}
The denominator $L_{k,n}$ is the sum of unweighted path contributions by Theorem~\ref{thm:sum-product}. Division proves Equation~\eqref{eq:bayes-curve}. When $r=2$, subtracting the square of the $r=1$ result yields pointwise posterior variance, with no claim of a simultaneous credible band.
\end{proof}

\subsection{Joint MAP segmentation}
Define local log scores $\ell_{ij}=\log\widetilde A^{(0)}_{ij}$ and max-sum messages
\begin{equation}
M_{k,j}=\max_{h\in\{k-1,\ldots,j-1\}}
       \{M_{k-1,h}+\ell_{h,j}\},
\quad M_{0,0}=0,
\label{eq:max-sum}
\end{equation}
with $M_{0,j}=-\infty$ for $j>0$. Store a deterministic maximizing predecessor $h^\star(k,j)$. Starting at $(k,n)$ and following these predecessors recovers a complete boundary vector.

\begin{theorem}[Correctness of max-sum backtracking]
\label{thm:map-correctness-paper}
For fixed $k$, the boundary vector returned by Equation~\eqref{eq:max-sum} and its stored predecessors belongs to
\begin{equation}
\arg\max_{0=t_0<\cdots<t_k=n}
  \sum_{q=1}^{k}\ell_{t_{q-1},t_q}.
\label{eq:joint-map-objective}
\end{equation}
Consequently, it is a joint MAP segmentation conditional on $k$. Ties are resolved only by the declared predecessor convention.
\end{theorem}
\begin{proof}
\textbf{Step 1 (base case).}
For a one-segment prefix $(0,j]$, the only admissible path consists of the block $(0,j]$. Hence $M_{1,j}=\ell_{0,j}$ is its maximum score, and its stored predecessor is $0$.
\textbf{Step 2 (induction hypothesis).}
Assume that for each feasible $h$, $M_{k-1,h}$ is the maximum log score among all $(k-1)$-segment partitions ending at $h$.
\textbf{Step 3 (optimal substructure).}
Every $k$-segment path ending at $j$ has a unique penultimate endpoint $h\in\{k-1,\ldots,j-1\}$. Its score is the prefix path score plus $\ell_{h,j}$. At any fixed $h$, the best possible prefix is $M_{k-1,h}$ by the hypothesis. Therefore maximizing over $h$ gives exactly Equation~\eqref{eq:max-sum}, the best $k$-segment score.
\textbf{Step 4 (recover a feasible maximizing path).}
Each predecessor $h^\star(k,j)$ attains the recurrence maximum and satisfies $h^\star(k,j)<j$. Starting from $(k,n)$ and following these predecessors decreases the endpoint and block count by one each step until $(0,0)$. The resulting ordered boundary vector attains $M_{k,n}$ and maximizes Equation~\eqref{eq:joint-map-objective}. This is conditional on $k$; joint optimization over $k$ also needs the count-prior offset.
\end{proof}

\begin{cautionbox}[Marginal modes are not a joint MAP path]
The vector formed from $\arg\max_h p(t_p=h\mid y,k)$ for each $p$ optimizes separate marginal decisions. Its entries can be incompatible or can have lower joint posterior probability than another partition. \BayesBreak therefore reports marginal modes as diagnostics and reserves ``joint MAP'' for Theorem~\ref{thm:map-correctness-paper}.
\end{cautionbox}

\subsection{Decision-theoretic segmentation summaries}
\label{sec:decision-rules}
Posterior summaries answer different decision questions. The distinction is not a formatting convention: the marginal posterior of $k$ sums over segmentations, while joint MAP scoring maximizes over them.

\begin{proposition}[Three non-equivalent posterior modes]
\label{prop:three-modes}
Define
\begin{align*}
\widehat k_{\rm marg}&\in\arg\max_k\left\{\log p(k)+\log L_{k,n}-\log C_k\right\},\\
\widehat t_{\rm cond}&\in\arg\max_{t\in\mathcal T_{\widehat k_{\rm marg}}}
 \sum_{q=1}^{\widehat k_{\rm marg}}\ell_{t_{q-1},t_q},\\
(\widehat k_{\rm joint},\widehat t_{\rm joint})&\in\arg\max_{k,t\in\mathcal T_k}
\left\{\log p(k)-\log C_k+\sum_{q=1}^{k}\ell_{t_{q-1},t_q}\right\}.
\end{align*}
Here $\ell_{ij}=\log\widetilde A^{(0)}_{ij}$. The first is a marginal-MAP segment count, the second a conditional joint-MAP partition at that count, and the third a global joint-MAP count--partition pair. In general they are distinct decisions.
\end{proposition}
\begin{proof}
\textbf{Step 1 (marginalize partitions).} The evidence for a fixed count is $p(y\mid k)=L_{k,n}/C_k$, so Bayes' rule gives the first maximization after discarding its constant normalizer.
\textbf{Step 2 (optimize within a selected count).} Conditional on $k$, the factors $p(k)$ and $C_k^{-1}$ are constant across admissible partitions. Therefore max-sum backtracking solves the second maximization for $k=\widehat k_{\rm marg}$.
\textbf{Step 3 (optimize jointly).} If $k$ also varies, the complete joint score retains both $\log p(k)$ and $-\log C_k$, giving the third maximization.
\textbf{Step 4 (show the distinction).} For each $k$, the marginal score uses a \emph{sum} $L_{k,n}$ over all paths, whereas the best-path score uses their \emph{maximum}. Counts with many moderately probable paths can have greater marginal support than counts with one sharply preferred path, so no identity between the maximizing counts follows.
\end{proof}

\begin{proposition}[Bayes actions for full-partition and coordinate losses]
\label{prop:loss-bayes-action}
For fixed $k$, 0--1 loss on the entire ordered vector $t=(t_1,\ldots,t_{k-1})$ is minimized by a joint MAP partition. For coordinatewise Hamming loss $\mathcal L(a,t)=\sum_{q=1}^{k-1}\mathbf1\{a_q\ne t_q\}$, constrained to admissible ordered actions $a_1<\cdots<a_{k-1}$, the Bayes action instead solves
\begin{equation}
\max_{1\le a_1<\cdots<a_{k-1}\le n-1}
\sum_{q=1}^{k-1}p(t_q=a_q\mid y,k).
\label{eq:constrained-hamming}
\end{equation}
This constrained optimization has its own additive dynamic program; independent coordinatewise marginal modes need not form an admissible action.
\end{proposition}
\begin{proof}
\textbf{Step 1 (0--1 loss).} Its posterior risk at action $a$ is $1-p(t=a\mid y,k)$, minimized exactly by maximizing the joint posterior mass.
\textbf{Step 2 (Hamming loss).} By linearity of conditional expectation, its posterior risk is $(k-1)-\sum_qp(t_q=a_q\mid y,k)$. Minimizing risk over valid actions is therefore equivalent to Equation~\eqref{eq:constrained-hamming}; dropping the ordering constraint would solve a different optimization problem.
\textbf{Step 3 (dynamic program).} Set $D_{1,h}=p(t_1=h\mid y,k)$ and recursively
\[
D_{q,h}=p(t_q=h\mid y,k)+\max_{g<h}D_{q-1,g}
\]
on valid index ranges, retaining predecessors. Backtracking from the best valid terminal state yields the constrained Bayes action; prefix maxima implement each layer efficiently.
\textbf{Step 4 (explicit counterexample).} For $n=7,k=3$, assign posterior masses $0.17,0.18,0.20,0.20,0.125,0.125$ to the admissible pairs $(1,3),(2,3),(4,5),(4,6),(1,5),(2,6)$, respectively, and zero elsewhere. They sum to one. The marginal mode of $t_1$ is $4$ with probability $0.40$, while that of $t_2$ is $3$ with probability $0.35$, producing the impossible vector $(4,3)$. The joint modes are $(4,5)$ and $(4,6)$, each of mass $0.20$; the constrained Hamming calculation must respect ordering rather than export $(4,3)$.
\end{proof}

The output interface reports each object by its complete name. Posterior-mean segment levels \emph{conditional on} an exported MAP partition are likewise distinct from posterior-mean Bayes curves averaged over the segmentation distribution.

\subsection{Complexity and numerical invariants}
If cumulative sufficient statistics make each local conjugate score constant time, there are $\Theta(n^2)$ blocks to precompute. Sum-product and max-sum each inspect $\Theta(k_{\max}n^2)$ predecessor transitions. Values and backpointers require $\Theta(k_{\max}n)$ working storage, while retaining all block-cover terms can require $\Theta(k_{\max}n^2)$ storage.

\begin{proposition}[Complexity]
\label{prop:complexity-paper}
When cumulative sufficient statistics make each conjugate segment calculation constant time, segment precomputation costs $\Theta(n^2)$ time and storage; either partition recursion costs $\Theta(k_{\max}n^2)$ time. For a nonconjugate block routine with per-block cost $T_{\rm block}$, local precomputation costs $\Theta(n^2T_{\rm block})$ and the global recursions are unchanged.
\end{proposition}
\begin{proof}
\textbf{Step 1 (candidate-segment count).}
There are $n(n+1)/2$ nonempty contiguous segments $(i,j]$ with $0\le i<j\le n$. With constant-time evaluation from prefix statistics, filling the triangular local score array therefore takes $\Theta(n^2)$ time and memory.
\textbf{Step 2 (count predecessor transitions).}
At a feasible state $(k,j)$, sum-product or max-sum examines $j-k+1$ predecessor endpoints in the unconstrained dense case. For a range of counts $1\le k\le k_{\max}\le n$, the exact transition total is of order $\sum_{k=1}^{k_{\max}}\sum_{j=k}^{n}(j-k+1)$, which is $\mathcal O(k_{\max}n^2)$ and attains that order for the usual regimes where $k_{\max}$ does not collapse to an empty or singleton state set.
\textbf{Step 3 (separate local approximation cost).}
If each local score costs $T_{\rm block}$ operations rather than a constant, filling the array costs $\Theta(n^2T_{\rm block})$. Neither sum-product nor max-sum changes its predecessor traversal as a result. These are worst-case accounting statements; finite-range measured timing slopes need not equal the asymptotic exponents.
\end{proof}

Four numerical identities are checked in every implementation: the normalized segment-count posterior sums to one; forward and backward total evidence agree; conditional boundary-event probabilities sum to $k-1$; and the backtracked path score equals the terminal max-sum score within numerical tolerance. These checks test algebraic implementation rather than empirical model quality.

\section{Irregular designs and hierarchical extensions}
\label{sec:extensions}

The ordered-partition recursions remain valid when prior factors or multiple-sequence likelihoods can be expressed segment by segment. This section gives the corresponding models for irregular coordinates, shared changepoints, observed groups, and latent groups.

\subsection{Irregular designs: segment cohesion and boundary hazard}
Irregular coordinates can enter the partition prior in two statistically distinct ways. A complete-segment cohesion $c_u(i,j)$ assigns prior mass according to the physical span of $(i,j]$; an interior-boundary hazard $h_u(j)$ assigns prior mass to a candidate split at $u_j$. They can appear together in Equation~\eqref{eq:partition-prior}, but one does not imply the other.

\begin{figure}[!htbp]
  \centering
  \resizebox{0.96\columnwidth}{!}{\begin{tikzpicture}[x=1.2cm,y=1cm]
\draw[line width=0.65pt,black!60] (0,0)--(8,0);
\foreach \x/\lab in {0/$u_{t_0}$,2/$u_{t_1}$,5/$u_{t_2}$,8/$u_{t_3}$}{
  \fill[black!70] (\x,0) circle (1.8pt);
  \node[below=3pt,font=\small] at (\x,0) {\lab};
}
\foreach \start/\finish/\text in {0/2/{c_u(t_0,t_1)},2/5/{c_u(t_1,t_2)},5/8/{c_u(t_2,t_3)}}{
  \draw[line width=1.15pt,black!75] (\start,0.1)--(\finish,0.1);
  \node[above=9pt,font=\small] at ({(\start+\finish)/2},0) {$\text$};
}
\draw[line width=0.85pt,black] (2,-.22)--(2,.34);
\draw[line width=0.85pt,black] (5,-.22)--(5,.34);
\node[above=27pt,font=\small] at (2,0) {$h_u(t_1)$};
\node[above=27pt,font=\small] at (5,0) {$h_u(t_2)$};
\node[font=\small,align=center,below=12mm] at (4,0) {$p(t\mid k,u)\propto\prod_{q=1}^{k}c_u(t_{q-1},t_q)\prod_{q=1}^{k-1}h_u(t_q)$};
\end{tikzpicture}}
  \caption{Design-dependent partition prior. Segment cohesion and interior-boundary hazard represent different prior assumptions. The terminal endpoint receives no boundary-hazard factor.}
  \label{fig:cohesion-hazard-paper}
\end{figure}

\begin{theorem}[Local prior factors preserve exact inference]
\label{thm:local-prior-paper}
Replacing $A^{(r)}_{ij}$ by
$\widetilde A^{(r)}_{ij}=A^{(r)}_{ij}c_u(i,j)h_u(j)^{\1\{j<n\}}$
causes the sum-product recursion to enumerate exactly the unnormalized joint weights in Equation~\eqref{eq:partition-prior}. Dividing the terminal partition sum by $C_k(u)$ yields the proper fixed-$k$ marginal likelihood.
\end{theorem}
\begin{proof}
\textbf{Step 1 (factor the intended partition weight).}
For fixed $k$ and $t=(t_0,\ldots,t_k)$, multiply its independent segment marginal likelihoods by the conditional prior in Equation~\eqref{eq:partition-prior}.
\textbf{Step 2 (push each factor into a segment).}
Every segment ending at $t_q<n$ contributes the prior factor $c_u(t_{q-1},t_q)h_u(t_q)$, while the terminal segment ending at $n$ contributes only $c_u(t_{k-1},n)$. Thus $\prod_q\widetilde A^{(0)}_{t_{q-1},t_q}$ is exactly the likelihood product times the \emph{unnormalized} prior factor.
\textbf{Step 3 (sum over all partitions).}
Theorem~\ref{thm:sum-product} enumerates these path products once each, so $L_{k,n}$ is their unnormalized sum. Running the same recurrence with $A^{(0)}_{ij}=1$ enumerates exactly the prior weights and returns $C_k(u)$.
\textbf{Step 4 (normalize once, not once per block).}
Multiplication by $C_k(u)^{-1}$ in the fixed-count prior is common across partitions; hence $p(y\mid k)=L_{k,n}/C_k(u)$. Applying the same factor inside each block instead would produce an incorrect power of the prior normalizer.
\end{proof}

A boundary-process interpretation applies only to the hazard factor. Let
$\Lambda_j=\int_{u_{j-1}}^{u_j}\lambda(v)\dd v$ for an inhomogeneous Poisson process. Interval
occupancies are independent Bernoulli variables with probability $1-e^{-\Lambda_j}$. Conditional
on exactly $m$ occupied candidate intervals, the subset probability is proportional to the product
of Bernoulli odds, so the compatible local factor is
\begin{equation}
 h_u(j)\propto \frac{1-e^{-\Lambda_j}}{e^{-\Lambda_j}}=e^{\Lambda_j}-1.
\label{eq:poisson-hazard}
\end{equation}
Constant intensity and equal-width intervals recover an index-uniform fixed-count boundary prior when $c_u\equiv1$. Without conditioning on the number of occupied intervals, both occupied and unoccupied Bernoulli factors are part of the prior. Equation~\eqref{eq:poisson-hazard} does not justify multiplying a segment likelihood by its physical length; a duration-dependent cohesion is a separate prior choice.

\subsection{Renewal interpretation and refinement of an irregular design}
\label{sec:design-refinement}
The local prior factorization admits a renewal interpretation, but its conditioning matters. A conditional prior on partitions given a separately sampled count $k$ does not automatically induce a prior on the count.

\begin{proposition}[Renewal-length equivalence conditional on $k$]
\label{prop:renewal-equivalence-paper}
On an index grid with $0=t_0<t_1<\cdots<t_k=n$, let the positive segment lengths $d_q=t_q-t_{q-1}$ have an unnormalized renewal mass $g(d_q)\ge0$. Conditional on $\sum_qd_q=n$ and a fixed count $k$, the boundary distribution is
\[
p(t\mid k)=\frac{1}{C_k}\prod_{q=1}^{k}g(t_q-t_{q-1}),\qquad
C_k=\sum_{t\in\mathcal T_k}\prod_{q=1}^{k}g(t_q-t_{q-1}).
\]
For geometric $g(d)=\rho(1-\rho)^{d-1}$, $0<\rho<1$, the product is constant over $\mathcal T_k$, and thus $p(t\mid k)$ is uniform regardless of $\rho$.
\end{proposition}
\begin{proof}
\textbf{Step 1 (one-to-one correspondence).} Each ordered partition uniquely determines a composition $(d_1,\ldots,d_k)$ of $n$ into positive integers, and conversely the partial sums of each such composition recover the boundaries.
\textbf{Step 2 (condition the renewal mass).} The unnormalized composition weight is $\prod_qg(d_q)$. Summing those weights over every admissible composition defines $C_k$, so normalizing yields the displayed conditional distribution.
\textbf{Step 3 (geometric specialization).} For $g(d)=\rho(1-\rho)^{d-1}$, the product equals $\rho^k(1-\rho)^{\sum_qd_q-k}=\rho^k(1-\rho)^{n-k}$, independent of $t$ at fixed $k$. The factor cancels against $C_k$ and cannot push the separately specified segment-count prior $p(k)$ toward one or toward $n$.
\end{proof}

\begin{proposition}[Score preservation under grid refinement]
\label{prop:refinement-paper}
Suppose a candidate-coordinate grid is refined by adding locations carrying no new observations. Embed an original partition into the refined grid by retaining all of its physical segment endpoints and placing no boundary at a newly inserted location. Each inherited block then has the same likelihood integral; any cohesion or hazard factor depending only on unchanged physical endpoints and retained boundary weights has the same unnormalized prior contribution. The fully normalized partition posterior need not remain unchanged if new boundary locations are made admissible.
\end{proposition}
\begin{proof}
\textbf{Step 1 (likelihood equality).} Each inherited refined block contains exactly the original observations and descriptors, so its likelihood product and prior-integrated segment likelihood agree with those of the original block.
\textbf{Step 2 (prior equality for inherited paths).} Preserving physical endpoints preserves a cohesion depending only on those endpoints; by hypothesis, each retained interior boundary also retains its hazard factor. Multiplying across blocks thus gives the same unnormalized path weight.
\textbf{Step 3 (why posterior equality fails).} Refinement that admits an additional cut creates extra partitions in the sum defining $C_k$ and in the data-dependent partition sum. Even if every inherited path has unchanged weight, the two new normalizers generally differ. The equality is for inherited path scores, not for the whole posterior.
\end{proof}

\paragraph{Physical interval conventions.}
Observation locations $x_1<\cdots<x_n$ and candidate boundary coordinates $u_0<\cdots<u_n$ are distinct objects. Interior cell edges may be set to $u_i=(x_i+x_{i+1})/2$ for $1\le i<n$ and endpoints determined from the scientific measurement domain. The complete-segment span $u_j-u_i$ is not the same object as the observation gap $x_{i+1}-x_i$. An inhomogeneous Poisson boundary process can supply hazard factors through integrated intensities, but an interval-length cohesion does not automatically produce that process.

\subsection{Shared changepoints across multiple sequences}
Suppose $S$ aligned sequences share a boundary vector while retaining sequence-specific segment parameters and priors. Conditional independence gives a pooled segment marginal likelihood.

\begin{figure}[!htbp]
  \centering
  \resizebox{0.96\columnwidth}{!}{\begin{tikzpicture}
\node[latent] (t) {$t$};
\node[latent,right=20mm of t] (theta) {$\theta^{(s)}_q$};
\node[obs,right=18mm of theta] (y) {$y^{(s)}_i$};
\node[const,above=9mm of y] (d) {$d^{(s)}_i$};
\edge {t} {theta}
\edge {theta} {y}
\edge {d} {y}
\plate {seg} {(theta)} {$q=1{:}k$}
\plate {subj} {(seg)(y)(d)} {$s=1{:}S$}
\node[secondarybox,text width=44mm,below=14mm of theta,xshift=8mm] (prod) {$\log A^{\rm pool}_{ij}=\sum_s\log A^{(s)}_{ij}$};
\draw[arrSecondary] (subj.south)--(prod.north);
\end{tikzpicture}}
  \caption{Hierarchical common-changepoint model with sequence-specific segment parameters. Conditional independence makes the sequence-specific segment marginal likelihoods multiply.}
  \label{fig:shared-boundary-paper}
\end{figure}

\begin{assumption}[Common candidate set]
\label{ass:shared-support}
The sequences are aligned to a common ordered candidate set. Every candidate partition considered by the common-changepoint model is admissible for every sequence after the declared missing-data or exposure convention. Conditional on the common partition, sequences and their segment parameters are independent, and all sequence-specific segment marginal likelihoods are finite and positive on the common support.
\end{assumption}

\begin{theorem}[Exact pooling under a common partition]
\label{thm:pooling-paper}
Under Assumption~\ref{ass:shared-support}, the pooled segment marginal likelihood is
\begin{equation}
 A^{(0),\mathrm{pool}}_{ij}=\prod_{s=1}^{S}A^{(0,s)}_{ij},
 \qquad
 \ell^{\mathrm{pool}}_{ij}=\sum_{s=1}^{S}\ell^{(s)}_{ij}.
\label{eq:pooling}
\end{equation}
Using this array in Section~\ref{sec:exact-inference} gives the exact posterior distribution over the common partition. Conditional sequence-specific parameter posteriors are recovered from the original sequence-specific conjugate updates.
\end{theorem}
\begin{proof}
\textbf{Step 1 (integrate independently within a segment).}
For sequence $s$, write $\theta_{B,s}$ for its parameter on a candidate shared segment $B=(i,j]$. By the independent segment priors, integrating its observations over $\theta_{B,s}$ yields $A^{(0,s)}_{ij}$.
\textbf{Step 2 (factor across sequences).}
The model assumes conditional independence across sequences given the shared partition, and the subject-specific segment parameters are independent a priori. Tonelli's theorem for nonnegative integrands permits exchanging product and integration, so
\[
\int\prod_{s=1}^{S}p(y_B^{(s)}\mid\theta_{B,s})p_s(\theta_{B,s})\,\dd\theta_{B,1:S}
=\prod_{s=1}^{S}A^{(0,s)}_{ij}.
\]
Taking logarithms gives the additive log-score formula in Equation~\eqref{eq:pooling}.
\textbf{Step 3 (apply the unchanged partition recursion).}
Multiplication over segments produces the exact integrated joint likelihood conditional on every shared partition. Applying the same partition prior and the sum-product recursion therefore yields its exact boundary and count posterior. Given a partition, conjugate posterior calculations for each sequence's segment parameters remain separate; they are not automatically a single joint conjugate density integrated over all partitions.
\end{proof}

Known groups are handled by applying Equation~\eqref{eq:pooling} separately within each observed group. The following asymptotic statement is restricted to a finite candidate set and requires information to accumulate across sequences.

\begin{theorem}[Finite-candidate consistency for a common partition]
\label{thm:shared-consistency-paper}
Let $\mathcal T$ be finite and contain a data-generating candidate $t^\star$. For sequence $s$, let $M_s(t)>0$ be the integrated score used by the common-partition model. Assume independent sequences, $p(t^\star)>0$, and for each $t\ne t^\star$:
\begin{enumerate}
\item $\mathbb E|\log M_s(t^\star)-\log M_s(t)|<\infty$;
\item a strong law holds for the centered log-score differences; and
\item $\liminf_{S\to\infty}S^{-1}\sum_{s=1}^{S}\mathbb E[\log M_s(t^\star)-\log M_s(t)]=\delta_t>0$.
\end{enumerate}
Then $p(t^\star\mid y^{(1:S)})\to1$ almost surely, and every joint MAP segmentation is eventually $t^\star$ almost surely.
\end{theorem}
\begin{proof}
\textbf{Step 1 (posterior ratio for an incorrect candidate).}
Fix $t\ne t^\star$ and put $D_s(t)=\log M_s(t^\star)-\log M_s(t)$. Bayes' rule gives
\[
\frac{p(t\mid y^{(1:S)})}{p(t^\star\mid y^{(1:S)})}
=\frac{p(t)}{p(t^\star)}\exp\!\left\{-\sum_{s=1}^{S}D_s(t)\right\}.
\]
If $p(t)=0$, the ratio is identically zero and needs no further argument.
\textbf{Step 2 (accumulation of log-score separation).}
The assumed strong law makes $S^{-1}\sum_s[D_s(t)-\mathbb ED_s(t)]\to0$ almost surely. The strictly positive limiting lower bound on the mean expected difference then gives $\liminf_{S\to\infty}S^{-1}\sum_sD_s(t)\ge\delta_t>0$ almost surely.
\textbf{Step 3 (vanishing posterior odds).}
The exponential factor in Step~1 tends to zero at an exponential rate in $S$, and the fixed prior odds cannot prevent this. Hence every incorrect candidate has posterior mass asymptotically negligible relative to $t^\star$.
\textbf{Step 4 (finite candidate set).}
There are only finitely many $t\in\mathcal T$, so the sum of incorrect posterior ratios tends to zero. Normalization yields $p(t^\star\mid y^{(1:S)})\to1$, and eventually its posterior mass exceeds that of every incorrect $t$, proving the joint MAP conclusion. The theorem requires an informative average across growing numbers of sequences; adding one informative sequence and arbitrarily many uninformative ones does not satisfy the positive-gap condition.
\end{proof}

One fixed informative sequence is insufficient for this theorem because its contribution vanishes in the average as $S$ increases. The assumption requires an accumulating average expected log-score gap.

\subsection{Finite latent-group template model}
For sequence $s$ and deterministic segmentation template $\tau=(k,t)$, consider the score
\begin{equation}
 S_s(\tau)=\frac{p(k)}{C_k}\prod_{q=1}^{k}
 A^{(0,s)}_{t_{q-1},t_q}\gamma_u(t_{q-1},t_q).
\label{eq:sequence-template-score-paper}
\end{equation}
These subject-template scores are not required to integrate to one over the data space. We therefore define the finite criterion
\begin{align}
\mathcal F(\pi,\tau)
&=\sum_{s=1}^{S}\log\!\left\{\sum_{g=1}^{G}\pi_gS_s(\tau_g)\right\},\nonumber\\
&\hspace{8mm}\pi_g\ge0,\qquad \sum_g\pi_g=1.
\label{eq:latent-group-objective}
\end{align}
This is a latent-group clustering criterion. It is not interpreted as a finite-mixture likelihood unless each $S_s(\tau_g)$ is a normalized sampling density.

For row-stochastic auxiliary weights $r_{sg}$, Jensen's inequality gives
\begin{align}
\mathcal F(\pi,\tau)&\ge\mathcal Q(r,\pi,\tau),\nonumber\\
\mathcal Q(r,\pi,\tau)
&=\sum_{s,g}r_{sg}\bigl[\log\pi_g+\log S_s(\tau_g)-\log r_{sg}\bigr],
\label{eq:latent-group-bound}
\end{align}
with equality at
\begin{equation}
 r_{sg}=\frac{\pi_gS_s(\tau_g)}{\sum_h\pi_hS_s(\tau_h)}.
\label{eq:auxiliary-weights}
\end{equation}
The mixing-weight update is $\pi_g=S^{-1}\sum_s r_{sg}$. Let $n_g=\sum_s r_{sg}$. For a candidate segment count $k$, substitution of Equation~\eqref{eq:sequence-template-score-paper} into the minorizer gives the count offset $n_g\{\log p(k)-\log C_k\}$ and the max-sum local score
\begin{equation}
\ell^{(g)}_{ij}=\sum_{s=1}^{S}r_{sg}\log A^{(0,s)}_{ij}
                  +n_g\log\gamma_u(i,j).
\label{eq:template-local-score}
\end{equation}
The template update solves the fixed-$k$ max-sum problem and then compares the offset-adjusted terminal values across the allowed segment counts.

\begin{theorem}[Monotonicity of exact latent-group updates]
\label{thm:template-monotonicity-paper}
An iteration that sets $r$ by Equation~\eqref{eq:auxiliary-weights}, maximizes $\mathcal Q$ over $\pi$, and exactly maximizes every responsibility-weighted template criterion using Equation~\eqref{eq:template-local-score} and its segment-count offset over the declared finite candidate set cannot decrease $\mathcal F$.
\end{theorem}
\begin{proof}
\textbf{Step 1 (construct a tight Jensen minorizer).}
For any row-stochastic $r$, Jensen's inequality applied to the positive sum over group templates gives $\mathcal F(\pi,\tau)\ge\mathcal Q(r,\pi,\tau)$. Choosing $r$ as in Equation~\eqref{eq:auxiliary-weights} at the current $(\pi^{\rm old},\tau^{\rm old})$ makes the inequality an equality.
\textbf{Step 2 (update group weights).}
At fixed $r$ and templates, the $\pi$-dependent terms are $\sum_gn_g\log\pi_g$. Maximization on the probability simplex gives $\pi_g=n_g/S$ for positive group masses, under the declared boundary convention for empty groups. This cannot decrease $\mathcal Q$.
\textbf{Step 3 (update templates exactly).}
For group $g$, substitute the score definition in Equation~\eqref{eq:sequence-template-score-paper} into $\sum_sr_{sg}\log S_s(\tau_g)$. Every segment receives $\sum_sr_{sg}\log A^{(0,s)}_{ij}+n_g\log\gamma_u(i,j)$, and the count offset is $n_g[\log p(k)-\log C_k]$. Max-sum dynamic programming and comparison across all admissible $k$ globally maximize this coordinate of the minorizer.
\textbf{Step 4 (chain the inequalities).}
Let $(\pi^{\rm new},\tau^{\rm new})$ denote the updated parameters. Then
\[
\mathcal F(\pi^{\rm new},\tau^{\rm new})
\ge\mathcal Q(r,\pi^{\rm new},\tau^{\rm new})
\ge\mathcal Q(r,\pi^{\rm old},\tau^{\rm old})
=\mathcal F(\pi^{\rm old},\tau^{\rm old}).
\]
The argument establishes monotonicity of the stated finite \emph{score-clustering} criterion, not maximization of a normalized generative mixture likelihood or global convergence across restarts.
\end{proof}

\begin{proposition}[Permutation symmetry of the finite template criterion]
\label{prop:label-symmetry-paper}
For every permutation $\sigma$ of the $G$ group indices, simultaneously replacing $(\pi_g,\tau_g)$ by $(\pi_{\sigma(g)},\tau_{\sigma(g)})$ leaves $\mathcal F$ unchanged. Consequently, the numerical labels of fitted templates are not identified by this criterion; distinct groups with identical templates can also produce non-uniqueness beyond label permutations.
\end{proposition}
\begin{proof}
\textbf{Step 1 (permute inside each sequence score).}
For subject $s$, relabeling the terms of the finite sum gives $\sum_g\pi_{\sigma(g)}S_s(\tau_{\sigma(g)})=\sum_g\pi_gS_s(\tau_g)$, because $\sigma$ is a bijection.
\textbf{Step 2 (preserve the objective).}
Taking the logarithm and summing the unchanged subject terms leaves Equation~\eqref{eq:latent-group-objective} invariant.
\textbf{Step 3 (non-uniqueness).}
If a maximizing set of templates is relabeled, it remains maximizing. If two templates coincide, their contributions can also be aggregated with changed mixture weights without changing their combined term, giving an additional source of non-identifiability. A plotting convention may canonically order labels, but it does not establish parameter identifiability or optimality across restarts.
\end{proof}

Template-label permutations, coincident templates, and an oversized $G$ can make the optimum non-unique. No finite-mixture identifiability or generative-model consistency claim is made. Multiple initializations must be compared using the criterion evaluated at their \emph{final} iterates, rather than a value recorded before the last update.

\section{Approximate segment integration, posterior prediction, and evaluation}
\label{sec:approx-predict}

For a nonconjugate segment model, \BayesBreak requires a numerical approximation to the segment marginal likelihood. Posterior prediction requires the posterior predictive distribution associated with the fitted observation model. This section states the error condition under which segmentwise approximation error can be propagated through the partition recursions and then defines the prediction and boundary-evaluation rules used in the empirical analysis.

\subsection{Propagation of segment marginal-likelihood error}
Let $\ell_{ij}=\log A^{(0)}_{ij}$ be a reference log segment marginal likelihood and let $\widehat\ell_{ij}$ be the value returned by a numerical method such as Laplace approximation, the Jaakkola--Jordan bound, expectation propagation, P\'olya--Gamma mean-field approximation, or adaptive quadrature.

\begin{assumption}[Uniform error bound over admissible segments]
\label{ass:block-error}
There exists $\varepsilon\ge0$ such that
\begin{equation}
 |\widehat\ell_{ij}-\ell_{ij}|\le\varepsilon
\label{eq:block-error}
\end{equation}
for every segment $(i,j]$ that can occur in a partition with $k\le k_{\max}$. The numerical analysis must identify the reference value and, when applicable, separate optimization residual, truncation or tail error, quadrature error, and discrepancy from the reference calculation.
\end{assumption}

\begin{theorem}[Stability of marginal likelihoods and posterior odds]
\label{thm:stability-paper}
Under Assumption~\ref{ass:block-error}, every $k$-segment forward or backward quantity satisfies
\begin{align}
e^{-k\varepsilon}L_{k,j}
&\le\widehat L_{k,j}\le e^{k\varepsilon}L_{k,j},\nonumber\\
e^{-k\varepsilon}R_{k,i}
&\le\widehat R_{k,i}\le e^{k\varepsilon}R_{k,i}.
\label{eq:message-sandwich}
\end{align}
Consequently, for $k,k'\le k_{\max}$,
\begin{equation}
\left|
\log\frac{\widehat p(k\mid y)}{\widehat p(k'\mid y)}
-\log\frac{p(k\mid y)}{p(k'\mid y)}
\right|\le(k+k')\varepsilon,
\label{eq:k-odds-stability}
\end{equation}
and, for two locations $h,h'$ of the same ordered boundary $t_p$ under fixed $k$,
\begin{equation}
\left|
\log\frac{\widehat p(t_p=h\mid y,k)}{\widehat p(t_p=h'\mid y,k)}
-\log\frac{p(t_p=h\mid y,k)}{p(t_p=h'\mid y,k)}
\right|\le2k\varepsilon.
\label{eq:boundary-odds-stability}
\end{equation}
\end{theorem}
\begin{proof}
\textbf{Step 1 (pathwise perturbation).}
For any admissible partition of $k$ blocks, the approximate-to-exact path-weight ratio is $\exp\{\sum_{q=1}^k\delta_{t_{q-1},t_q}\}$, with every summand bounded in absolute value by $\varepsilon$. Consequently that ratio lies between $e^{-k\varepsilon}$ and $e^{k\varepsilon}$.
\textbf{Step 2 (sum over positive paths).}
Multiply the exact weight of each path by the same lower and upper bounds and sum. Positivity preserves both inequalities. The prefix and suffix path sets contain exactly $k$ segments each, proving the two message sandwich bounds in Equation~\eqref{eq:message-sandwich}.
\textbf{Step 3 (count posterior odds).}
For distinct counts $k,k'$, the ratio $p(k\mid y)/p(k'\mid y)$ equals $[p(k)/C_k]L_{k,n}$ divided by $[p(k')/C_{k'}]L_{k',n}$. The count priors and partition normalizers are unchanged by local numerical approximation. The two summed-path errors therefore add in the log ratio, yielding the bound $(k+k')\varepsilon$.
\textbf{Step 4 (ordered boundary odds).}
The unnormalized weight for $t_p=h$ is $L_{p,h}R_{k-p,h}$; it is a sum of paths each with exactly $p+(k-p)=k$ blocks. Its log perturbation is bounded by $k\varepsilon$. Comparing this weight at $h$ and $h'$ doubles the maximum error, proving the $2k\varepsilon$ statement.
\textbf{Step 5 (scope).}
The argument compares only positive exact and approximate scores on the same admissible support. A bound on typical blocks, a per-method asymptotic formula, or a large observed maximum score discrepancy does not establish the uniform hypothesis.
\end{proof}

\begin{corollary}[Consequences for probabilities and modal locations]
\label{cor:stability-paper}
If finite unnormalized weights satisfy $|\log\widehat w_a-\log w_a|\le\eta$, then their normalized probabilities satisfy $e^{-2\eta}\le\widehat p_a/p_a\le e^{2\eta}$ and $\|\widehat p-p\|_{\mathrm{TV}}\le\min\{1,e^{2\eta}-1\}$. Moreover, a fixed-$k$ modal boundary separated from every alternative by exact log-odds margin $\Delta$ is unchanged whenever $\varepsilon<\Delta/(2k)$.
\end{corollary}
\begin{proof}
\textbf{Step 1 (bound the perturbed normalizer).}
Write $\widehat w_a=w_ae^{\delta_a}$ with $|\delta_a|\le\eta$. Let $p_a=w_a/\sum_bw_b$. Then $\sum_bp_be^{\delta_b}\in[e^{-\eta},e^\eta]$.
\textbf{Step 2 (normalized probability ratios and total variation).}
Since $\widehat p_a/p_a=e^{\delta_a}/\sum_bp_be^{\delta_b}$, this ratio lies in $[e^{-2\eta},e^{2\eta}]$. Therefore $|\widehat p_a-p_a|\le(e^{2\eta}-1)p_a$. Summing its positive part bounds total variation by $\min\{1,e^{2\eta}-1\}$.
\textbf{Step 3 (preserve the maximizing boundary).}
Suppose $h^\star$ exceeds every rival $h$ by exact log-odds margin at least $\Delta$. Theorem~\ref{thm:stability-paper} changes each pairwise log-odds difference by at most $2k\varepsilon$. If $2k\varepsilon<\Delta$, all perturbed differences remain positive, so $h^\star$ remains the unique modal boundary location.
\end{proof}

\begin{proofobligation}[Error control for individual numerical methods]
\label{obl:routine-certification-paper}
This paper does not establish a uniform convergence rate for Laplace approximation, the Jaakkola--Jordan bound, expectation propagation, P\'olya--Gamma mean-field approximation, or adaptive quadrature over all admissible segments. To apply Theorem~\ref{thm:stability-paper}, a nonconjugate analysis must establish Assumption~\ref{ass:block-error} under conditions appropriate to the observation family, parameter domain, segment support, initialization, convergence event, and tail behavior. Without that step, the dynamic program is exact for the supplied approximate score array, but the theorem does not compare the resulting posterior with the posterior under the intended segment model.
\end{proofobligation}

\subsection{Nonconjugate block calculations: what each approximation establishes}
\label{sec:nonconjugate-derivations}
For a canonical Binomial--logistic block with a Gaussian prior on its intercept, the exact local evidence is a one-dimensional integral. It provides a numerical reference against which deterministic approximations can be tested before the resulting segment scores enter the same partition recursion. A local approximation formula, its variational lower bound, and a rigorous numerical error bound are three different objects.

\begin{lemma}[Finite, unique logistic-block mode under a proper normal prior]
\label{lem:logistic-mode-paper}
Let $y_r\mid\theta\sim\operatorname{Binom}(m_r,\sigma(\theta))$ independently, with $m_r\ge0$, $\sigma(\theta)=(1+e^{-\theta})^{-1}$, and $\theta\sim\mathcal N(\nu,\rho^2)$ for finite $\rho^2>0$. Then the block log posterior is strictly concave and has a unique finite maximizer, including blocks containing only successes or only failures.
\end{lemma}
\begin{proof}
\textbf{Step 1 (posterior objective).} Apart from its Binomial coefficients, the log posterior is
\[
\Psi_B(\theta)=Y_B\theta-M_B\log(1+e^\theta)
-\frac{(\theta-\nu)^2}{2\rho^2}+\mathrm{constant},
\]
where $Y_B=\sum y_r$ and $M_B=\sum m_r$.
\textbf{Step 2 (strict curvature).} Direct differentiation gives
\[
\Psi_B''(\theta)=-M_B\sigma(\theta)\{1-\sigma(\theta)\}-\rho^{-2}<0
\]
for every finite $\theta$, even when $M_B=0$.
\textbf{Step 3 (existence, not just uniqueness).} The Binomial log-likelihood is bounded above by the sum of its log combinatorial coefficients. The negative quadratic prior term tends to $-\infty$ as $|\theta|\to\infty$. Hence $\Psi_B$ is coercive, attains a finite maximum by continuity, and strict concavity makes that maximizer unique. This explicit tail condition is why the conclusion does not follow from strict concavity alone under an arbitrary prior.
\end{proof}

\subsubsection{Laplace evaluation and its limitations}
Let $\widehat\theta_B$ be the mode and $H_B=-\Psi_B''(\widehat\theta_B)>0$. The first-order Laplace approximation is
\begin{equation}
\log\widehat A_B^{\rm Lap}=\Psi_B(\widehat\theta_B)
+\tfrac12\log(2\pi)-\tfrac12\log H_B.
\label{eq:laplace-paper}
\end{equation}
This is a local second-order approximation to the integral, \emph{not} a nonasymptotic upper bound on the absolute log-evidence error. A controlled asymptotic remainder would require derivatives beyond second order, concentration of the integral near the mode, and bounds on the tails outside that neighborhood. Numerical evaluation should use safeguarded Newton steps, check finite curvature and convergence, and fall back to adaptive one-dimensional integration on failure. The software's local curvature diagnostic does not by itself establish Assumption~\ref{ass:block-error}.

\subsubsection{Jaakkola--Jordan lower bound}
\begin{proposition}[Valid logistic tangent bound and block-evidence lower bound]
\label{prop:jj-paper}
For $\xi>0$ put $\lambda(\xi)=\tanh(\xi/2)/(4\xi)$ and $\sigma(x)=(1+e^{-x})^{-1}$. Then, for every real $x$,
\begin{equation}
\log\sigma(x)\ge\log\sigma(\xi)+\frac{x-\xi}{2}
-\lambda(\xi)(x^2-\xi^2).
\label{eq:jj-correct-paper}
\end{equation}
For a $\operatorname{Binom}(m_r,\sigma(\theta))$ observation, this quadratic lower bound has linear coefficient $y_r-m_r/2$ and quadratic coefficient $-m_r\lambda(\xi_r)$. Combining these terms with a Gaussian prior and integrating the resulting Gaussian kernel gives a deterministic \emph{lower} bound on $\log A^{(0)}_B$; this bound alone supplies no upper bound on the unknown error gap.
\end{proposition}
\begin{proof}
\textbf{Step 1 (rewrite the log sigmoid).} The identity $\log\sigma(x)=x/2-\log2-\log\cosh(x/2)$ separates a linear term from an even function of $x$.
\textbf{Step 2 (bound the even term).} Define $f(v)=\log\cosh(\sqrt v/2)$ for $v\ge0$. Its derivative for $v>0$ is $f'(v)=\tanh(\sqrt v/2)/(4\sqrt v)$, which decreases with $v$. Thus $f$ is concave and lies below its tangent at $v=\xi^2$: $f(x^2)\le f(\xi^2)+\lambda(\xi)(x^2-\xi^2)$.
\textbf{Step 3 (recover the inequality).} Substitute this tangent inequality into the expression for $\log\sigma(x)$, then use the identity at $x=\xi$; Equation~\eqref{eq:jj-correct-paper} follows, with equality at $x=\pm\xi$ in the even part.
\textbf{Step 4 (apply to Binomial data).} Write $\log p(y_r\mid\theta)=\log\binom{m_r}{y_r}+(y_r-m_r/2)\theta-m_r\log[2\cosh(\theta/2)]$. Apply the tangent bound to its even part and sum over observations. Exponentiating and integrating an everywhere-lower-bounded nonnegative integrand preserves the ordering. Thus $\log\widehat A_B^{\rm JJ}\le\log A_B$, but the value of $\log A_B-\log\widehat A_B^{\rm JJ}$ has not been upper-bounded by this argument.
\end{proof}

\subsubsection{P\'olya--Gamma augmentation and mean-field computation}
\begin{proposition}[Augmentation identity and Gaussian coordinate updates]
\label{prop:pg-paper}
For a Binomial likelihood without its $\theta$-independent coefficient, let $b_r=m_r$, $\kappa_r=y_r-m_r/2$ and denote the P\'olya--Gamma density by $p_{\rm PG}(\omega\mid b,0)$. Then
\begin{equation}
\frac{e^{y_r\theta}}{(1+e^\theta)^{m_r}}
=2^{-b_r}e^{\kappa_r\theta}
\int_0^\infty e^{-\omega\theta^2/2}p_{\rm PG}(\omega\mid b_r,0)\dd\omega.
\label{eq:pg-identity-paper}
\end{equation}
With $\theta\sim\mathcal N(\nu,\rho^2)$, the mean-field family $q(\theta)\prod_rq(\omega_r)$ admits coordinate updates
\begin{align}
q(\theta)&=\mathcal N(\mu_q,V_q),
& V_q^{-1}&=\rho^{-2}+\sum_r\mathbb E_q[\omega_r],\nonumber\\
\mu_q&=V_q\left(\nu/\rho^2+\sum_r\kappa_r\right),
& q(\omega_r)&={\rm PG}(b_r,c_r),
& c_r^2&=\mu_q^2+V_q,
\label{eq:pg-updates-paper}
\end{align}
where $\mathbb E_q[\omega_r]=b_r\tanh(c_r/2)/(2c_r)$ with its continuous limit at $c_r=0$. If an additional \emph{whole-likelihood} power $a_r$ is used, the identity requires $(b_r,\kappa_r)=(a_rm_r,a_r(y_r-m_r/2))$; multiplying an independent ${\rm PG}(m_r,0)$ variable by $a_r$ is not equivalent.
\end{proposition}
\begin{proof}
\textbf{Step 1 (factor logistic likelihood).} The identity $1+e^\theta=2e^{\theta/2}\cosh(\theta/2)$ converts the Binomial likelihood into $2^{-m_r}e^{\kappa_r\theta}\cosh(\theta/2)^{-m_r}$.
\textbf{Step 2 (introduce the PG mixture).} The P\'olya--Gamma Laplace-transform identity is $\cosh(\theta/2)^{-b}=\int e^{-\omega\theta^2/2}p_{\rm PG}(\omega\mid b,0)\dd\omega$ for $b>0$. Substitution proves Equation~\eqref{eq:pg-identity-paper} (with the $b=0$ contribution interpreted as one).
\textbf{Step 3 (update the Gaussian factor).} The augmented log joint is quadratic in $\theta$ conditional on the PG variables. Taking the expectation over $q(\omega)$ and adding the Gaussian prior gives precision $\rho^{-2}+\sum_r\mathbb E_q[\omega_r]$ and linear coefficient $\nu/\rho^2+\sum_r\kappa_r$, proving the Gaussian update.
\textbf{Step 4 (update the latent factors).} The coordinate log density for $q(\omega_r)$ is $\log p_{\rm PG}(\omega_r\mid b_r,0)-\omega_r\mathbb E_q[\theta^2]/2$ plus a constant. Exponential tilting gives ${\rm PG}(b_r,c_r)$, $c_r^2=\mathbb E_q[\theta^2]$, and the standard PG mean gives Equation~\eqref{eq:pg-updates-paper}. Each exact coordinate maximization cannot lower its \emph{properly evaluated} ELBO; it does not certify the gap to the exact evidence.
\textbf{Step 5 (external powers).} Raising the complete Binomial factor to $a_r$ changes both the hyperbolic-cosine exponent and linear coefficient, hence the PG shape and $\kappa_r$ as stated. A scaled PG random variable does not have this tilted law.
\end{proof}

For the intercept-only logistic--Gaussian block, the optimized Jaakkola--Jordan quadratic bound and the corresponding PG mean-field Gaussian updates can be algebraically equivalent. For the present numerical comparison, the \texttt{jj} and \texttt{pg\_vb} labels correspond to the same implementation and hence represent a single computational method. A distinct P\'olya--Gamma implementation would require separate numerical validation.

\subsubsection{Expectation propagation and numerical reference}
Expectation propagation approximates likelihood contributions with Gaussian sites and iterates cavity/tilted-distribution moment matching. A valid approximation to the \emph{segment marginal likelihood} must include the global Gaussian normalizer and the site/tilted normalizing corrections, not merely integrate a normalized approximate posterior (whose integral is one). The EP estimates are not used to establish integrated-likelihood accuracy in the absence of a separately validated normalizer, damping scheme, convergence criterion, and reference comparison. Adaptive one-dimensional quadrature, with a documented integration tolerance and tail or endpoint control, remains the numerical reference for low-dimensional blocks.

\subsection{Family-specific posterior prediction}
A fitted model records the observation family, the segment posterior state, and the coordinate range used for fitting. Prediction must use the posterior predictive distribution of that observation family. Coordinates outside the fitted range are rejected by default; an extrapolation rule is permitted only when it is stated as part of the predictive model before evaluation.

Consider a fitted segment $B$ with posterior conjugate hyperparameters $(\alpha_B,\beta_B)$. Let $(S_B^{\rm new},W_B^{\rm new},H_B^{\rm new})$ denote the exposure-aware sufficient summaries of new observations assigned to $B$.

\begin{proposition}[Posterior-predictive likelihood for a fixed segment]
\label{prop:predictive-paper}
Under the exponential-family formulation of Section~\ref{sec:model},
\begin{equation}
 p(y_B^{\rm new}\mid y_B)
 =e^{H_B^{\rm new}}
 \frac{Z(\alpha_B+S_B^{\rm new},\beta_B+W_B^{\rm new})}
      {Z(\alpha_B,\beta_B)}.
\label{eq:posterior-predictive}
\end{equation}
\end{proposition}
\begin{proof}
\textbf{Step 1 (begin with the fitted segment posterior).}
Conditional on the selected fitted segment, the training data update the conjugate hyperparameters from $(\alpha_0,\beta_0)$ to $(\alpha_B,\beta_B)$. This posterior, not the original segment prior, is the integrating measure for future observations.
\textbf{Step 2 (aggregate new sampling information).}
Under the same observation family and its declared descriptors, the new data likelihood has the kernel $\exp\{H_B^{\rm new}+\eta(\theta)^\top S_B^{\rm new}-W_B^{\rm new}A(\theta)\}$.
\textbf{Step 3 (integrate once).}
Multiplication by the training posterior produces the conjugate kernel with parameters $(\alpha_B+S_B^{\rm new},\beta_B+W_B^{\rm new})$. Integration yields its normalizer divided by $Z(\alpha_B,\beta_B)$, multiplied by $e^{H_B^{\rm new}}$, giving Equation~\eqref{eq:posterior-predictive}.
\textbf{Step 4 (conditioning).}
This density is posterior predictive conditional on the fitted segment and its assignment rule. It is not automatically averaged over alternative partitions, nor is it valid for a test family or coordinate support different from the declared predictive model.
\end{proof}

Prediction may condition on an exported joint-MAP partition or average over partitions and segment counts. The first calculation is conditional on a selected partition; the second integrates segmentation uncertainty. In either case, assignment of prediction coordinates to segments is part of the statistical model. Assigning every out-of-range coordinate to the final segment without an explicit extrapolation assumption is not a valid posterior-predictive calculation.

\subsection{Joint prediction for sequences, grouped units, and multivariate responses}
\label{sec:prediction-extensions}
The fixed-segment predictive formula in Proposition~\ref{prop:predictive-paper} supports several input structures, but \emph{joint} scoring must preserve dependence induced by a shared segment parameter. A product of separately integrated single-observation predictive densities is generally not the integrated density of their joint collection.

\begin{proposition}[Joint predictive integration within a fitted segment]
\label{prop:joint-predictive-paper}
Condition on a fitted partition $t$ and conjugate posterior $p(\theta_B\mid y_B)$ for each of its segments $B$. If new observations $D_B^{\rm new}$ assigned to $B$ are conditionally independent given $\theta_B$ and have sufficient summaries $(S_B^{\rm new},W_B^{\rm new},H_B^{\rm new})$, then
\begin{equation}
p(D_B^{\rm new}\mid y_B,t)
=e^{H_B^{\rm new}}\frac{Z(\alpha_B+S_B^{\rm new},\beta_B+W_B^{\rm new})}{Z(\alpha_B,\beta_B)}.
\label{eq:joint-predictive-block}
\end{equation}
The full fixed-partition joint predictive density is the product of these contributions over \emph{distinct segments}, not in general over individual observations within one segment.
\end{proposition}
\begin{proof}
\textbf{Step 1 (conditional likelihood).} Given $\theta_B$, independence gives $p(D_B^{\rm new}\mid\theta_B)=\prod_{r\in D_B^{\rm new}}p(y_r\mid\theta_B,d_r)$.
\textbf{Step 2 (aggregate once).} The sufficient-statistic representation collects all new observations into one exponential-family kernel with summaries $S_B^{\rm new},W_B^{\rm new},H_B^{\rm new}$.
\textbf{Step 3 (integrate once).} Multiply this kernel by the training posterior and integrate $\theta_B$, yielding Equation~\eqref{eq:joint-predictive-block}. It differs from integrating $\theta_B$ separately for each observation, which would spuriously treat new observations as having independent latent parameters.
\textbf{Step 4 (between-segment factorization).} Under the model assumption of conditionally independent segment parameters, the posterior factorizes across distinct segments given $t$. Therefore the joint predictive density factors over the distinct segments, completing the claim.
\end{proof}

\begin{corollary}[Closed-form predictive distributions in three conjugate families]
\label{cor:predictive-families-paper}
Conditional on a fitted segment and its conjugate posterior, the following are ordinary posterior-predictive sampling distributions:
\begin{enumerate}
\item For Gaussian $\mu\mid y_B\sim\mathcal N(\nu_B,V_B)$ and new observations with known variances $\sigma^2/w_r$, a vector of $m$ new observations in the same segment has joint law
\[
Y^{\rm new}\mid y_B,t\sim\mathcal N_m\!\left(\nu_B\mathbf1_m,\;V_B\mathbf1_m\mathbf1_m^\top+\operatorname{diag}(\sigma^2/w_1,\ldots,\sigma^2/w_m)\right).
\]
\item For Poisson $\lambda\mid y_B\sim\operatorname{Gamma}(a_B,b_B)$ and a new count at exposure $e>0$, the predictive law is negative Binomial with
\[
p(y^{\rm new}=z\mid y_B,e)=\frac{\Gamma(a_B+z)}{\Gamma(a_B)z!}\left(\frac{b_B}{b_B+e}\right)^{a_B}\left(\frac{e}{b_B+e}\right)^z.
\]
\item For Binomial $p\mid y_B\sim\operatorname{Beta}(a_B,b_B)$ and $m$ future trials, the predictive count law is Beta--Binomial with
\[
p(y^{\rm new}=z\mid y_B,m)=\binom mz\frac{B(a_B+z,b_B+m-z)}{B(a_B,b_B)}.
\]
\end{enumerate}
\end{corollary}
\begin{proof}
\textbf{Step 1 (Gaussian joint law).}
Write each new observation as $Y_r=\mu+\epsilon_r$ with independent $\epsilon_r\sim\mathcal N(0,\sigma^2/w_r)$ conditional on $\mu$, and with $\mu\sim\mathcal N(\nu_B,V_B)$ from the training posterior. Gaussian addition gives mean $\nu_B\mathbf1_m$ and covariance $V_B\mathbf1_m\mathbf1_m^\top+\operatorname{diag}(\sigma^2/w_r)$. The rank-one term shows why multiplying separate Gaussian marginals would be incorrect for joint scoring.
\textbf{Step 2 (Poisson--Gamma mixture).}
Integrate $e^{-e\lambda}(e\lambda)^z/z!$ against the Gamma posterior density. The Gamma integral has shape $a_B+z$ and rate $b_B+e$, which simplifies to the displayed negative-Binomial probability mass function.
\textbf{Step 3 (Beta--Binomial mixture).}
Integrate $\binom mzp^z(1-p)^{m-z}$ against $\operatorname{Beta}(a_B,b_B)$. Combining exponents produces $B(a_B+z,b_B+m-z)/B(a_B,b_B)$, establishing the third expression.
\end{proof}

\paragraph{Pointwise sequences (Case A).}
New pointwise observations carry coordinates, outcomes, and the family-specific descriptors. Under an exported partition, coordinates select segments and the data are aggregated \emph{within} those segments before applying Equation~\eqref{eq:joint-predictive-block}. An out-of-range coordinate requires a declared extrapolation model; a coordinate gap is not an exposure by default.

\paragraph{Set-valued units (Case B).}
A unit may contain multiple pairs $(x_{ur},y_{ur})$ on a physical support interval. Its segment-specific internal observations can be scored jointly using the same sufficient-statistic aggregation. If several units in a single dataset share segment parameters, the dataset score aggregates \emph{across all such units before integrating}. Multiplying separately marginalized unit likelihoods is justified only under an additional model in which the units have independent latent segment parameters; it is otherwise a composite score.

\paragraph{Vector-valued responses (Case C).}
When response coordinates are conditionally independent given distinct component-specific segment parameters, the joint segment evidence and predictive density factor across response dimensions. For a genuinely multivariate exponential family with shared or correlated parameters, the joint sufficient statistics and conjugate normalizer must be used instead. The global ordered-partition recursion depends only on the resulting joint segment score.

\paragraph{Group membership and segmentation averaging.}
If $M_g$ defines a \emph{normalized} predictive data density and $\pi_g$ is a prior group probability, then $p(g\mid D^{\rm new})\propto\pi_gp(D^{\rm new}\mid M_g)$ is a Bayesian group posterior. Scores from the finite latent-group MM criterion in Equation~\eqref{eq:latent-group-objective}, which are not asserted to be normalized data densities, do not automatically admit this interpretation. Likewise, true segmentation-averaged prediction sums joint predictive contributions over $p(k,t\mid y)$; prediction conditional on a chosen joint-MAP partition is a different object.

\paragraph{Calibration.}
For discrete outcomes, randomize the probability integral transform within the probability mass at the observed outcome. For continuous outcomes, PIT values are marginally uniform under a calibrated predictive law, but they are generally dependent within a sequence. Report sequence-level calibration or simulation-based reference distributions rather than applying an independent-uniform goodness-of-fit reference blindly. Held-out predictive scores need not rank candidate $k$ in the same order as training posterior count probabilities.

\subsection{Boundary metrics}
\begin{definition}[One-to-one tolerance matching]
\label{def:boundary-matching-paper}
For estimated boundaries $\widehat{\mathcal B}$, reference boundaries $\mathcal B^\star$, and tolerance $\tau$, form the bipartite graph containing an edge $(\widehat b,b^\star)$ whenever $|\widehat b-b^\star|\le\tau$. Choose a matching of maximum cardinality; among maximum-cardinality matchings, choose one of minimum total absolute distance under a deterministic tie rule. Precision, recall, and $F_1$ use the resulting number of matches. No boundary is used more than once.
\end{definition}

\begin{definition}[Matched location error]
\label{def:matched-mae-paper}
For the matching $\mathcal M_\tau$,
\[
\operatorname{MAE}_\tau
=|\mathcal M_\tau|^{-1}
 \sum_{(\widehat b,b^\star)\in\mathcal M_\tau}|\widehat b-b^\star|
\]
when $|\mathcal M_\tau|>0$. When no pair is matched, the value is reported as \NA, not as zero.
\end{definition}

\begin{definition}[Admissible held-out predictive score]
\label{def:heldout-score-paper}
A held-out log score is used as predictive evidence only when the held-out observations were excluded from fitting, the posterior predictive distribution of the fitted observation family was used, and every prediction coordinate lay in the fitted support or followed a pre-specified extrapolation model. A violation of any condition invalidates the number for that predictive interpretation without changing the fact that the computation was executed.
\end{definition}

\paragraph{Predictive model checking.}
A posterior-predictive score is meaningful only for the observation law that
was fitted and for new coordinates within its stated support. An earlier
methylation calculation used a Gaussian predictive density for a fitted
Beta-response model and extrapolated outside the fitted range. That value is
not included in the predictive results. The admissible Beta-response
held-out analysis is reported in Section~\ref{sec:results}.

\paragraph{Interpretation of approximate local scores.}
The forward--backward and max--sum recursions are exact for the block-score array supplied to them. When the array is obtained by numerical approximation, this algorithmic exactness does not establish exactness with respect to the underlying sampling model. A Jaakkola--Jordan tangent gap averaged under a fitted Gaussian variational distribution is not an upper bound on the log-evidence error: the difference in log normalizing constants also depends on the changed posterior measure. Likewise, a truncated Laplace expansion is not a finite-sample error certificate unless its omitted remainder and tails have been bounded. Comparisons with a converged numerical-integration reference are necessary before interpreting either diagnostic as evidence of approximation accuracy.
\paragraph{Interpretation of nonconjugate method labels.}
In the implementation studied here, the options \texttt{jj} and \texttt{pg\_vb} both invoke the same Jaakkola--Jordan numerical routine. These labels therefore do not constitute two independent approximation methods. A distinct P\'olya--Gamma variational implementation, if included later, requires a separately specified augmented model, objective, update equations, and numerical verification.

\section{Algorithms, computational complexity, and numerical implementation}
\label{sec:algorithms}

The implementation separates calculations specific to an observation model from calculations over ordered partitions. Observation-model routines compute sufficient statistics, segment marginal likelihoods, posterior moments, numerical error information, and posterior-predictive quantities. The partition routines combine prior-adjusted segment scores, evaluate sum-product recursions, recover joint MAP partitions by max-sum recursion and backtracking, pool aligned sequences, and extract posterior summaries.

\begin{algorithm}[!htbp]
\caption{\BayesBreak offline fit}
\label{alg:bayesbreak-fit-paper}
\KwIn{Ordered observations $\{(u_r,y_r,d_r,a_r)\}_{r=1}^{n}$; observation family and hyperparameters; $k_{\max}$; segment-count prior $p(k)$; segment-cohesion and boundary-hazard factors $c_u,h_u$; numerical tolerances.}
\KwOut{Posterior distribution over $k$; boundary-event and ordered-boundary marginals; joint MAP partitions; segment posterior states; Bayes curves; numerical diagnostics and result metadata.}
Validate order, observation-family support, descriptors, power weights, and candidate-segment restrictions\;
Construct cumulative sufficient statistics when available\;
\For{each admissible segment $(i,j]$}{
  compute $\log A^{(0)}_{ij}$ and the required moments or posterior state\;
  record validity, convergence, and approximation diagnostics\;
  assign score $-\infty$ to an inadmissible segment rather than substituting a finite value from another model\;
}
Compute $\log\gamma_u(i,j)=\log c_u(i,j)+\1\{j<n\}\log h_u(j)$ once for each segment\;
Run log-space forward and backward recursions and normalize Equation~\eqref{eq:k-posterior}\;
Compute ordered-boundary marginals, boundary-event probabilities, and Bayes-curve moments\;
Run max-sum recursion and deterministic backtracking for every retained $k$ or for the selected $\widehat k$\;
Check forward--backward equality, posterior normalization, boundary-event mass, and backtrack-score equality\;
Export the observation-family identifier, fitted coordinate range, result metadata, and numerical diagnostics\;
\end{algorithm}

\begin{algorithm}[!htbp]
\caption{Finite latent-group score-clustering with exact template updates}
\label{alg:mm-template-paper}
\KwIn{Per-subject log segment likelihoods $\ell^{(s)}_{ij}$; local partition factors $\gamma_u(i,j)$; count prior $p(k)$ and normalizers $C_k$; group count $G$; restarts and stopping tolerance.}
Initialize $G$ valid template partitions and mixture weights for each restart\;
\Repeat{the fitted templates are stable and the \emph{recomputed final} objective change is below tolerance}{
  Compute each template score $\log S_s(\tau_g)=\log p(k_g)-\log C_{k_g}+\sum_{(i,j]\in\tau_g}(\ell^{(s)}_{ij}+\log\gamma_u(i,j))$\;
  Compute auxiliary weights $r_{sg}$ by stabilized softmax over $\log\pi_g+\log S_s(\tau_g)$\;
  Update $\pi_g\leftarrow S^{-1}\sum_s r_{sg}$, handling empty groups under the declared policy\;
  For every $g$, form $\ell^{(g)}_{ij}=\sum_s r_{sg}\ell^{(s)}_{ij}+n_g\log\gamma_u(i,j)$\;
  Run max-sum for every feasible $k$ and compare terminal scores with the offset $n_g[\log p(k)-\log C_k]$\;
  Backtrack to obtain the updated template $\tau_g$\;
  Recompute and record $\mathcal F(\pi,\tau)$ \emph{after} all template and weight updates; check nondecrease\;
}
Rank restarts by the actually recomputed final $\mathcal F$, not by a pre-update trace value\;
\KwOut{Fitted templates, group weights, row-stochastic auxiliary weights, and full objective traces.}
\end{algorithm}

\subsection{Inputs and returned quantities}
The fit input separates four objects that have different statistical meanings: ordered coordinate $u_r$, response $y_r$, observation descriptor $d_r$, and optional likelihood power $a_r$. Descriptors such as exposure or trial count are interpreted by the corresponding observation model. Irregular geometry enters the prior through $u$ and the factors $c_u,h_u$; it is not treated as exposure in the observation likelihood.

\begin{table}[!htbp]
\centering
\small
\caption{Principal returned quantities and their inferential interpretation.}
\label{tab:exported-objects}
\begin{tabularx}{0.96\textwidth}{@{}L{0.19\textwidth}L{0.22\textwidth}Y L{0.18\textwidth}@{}}
\toprule
Quantity & Calculation & Interpretation & Numerical checks \\
\midrule
$\log A^{(0)}_{ij}$ & observation-model segment calculation & log integrated likelihood or declared numerical approximation & support, finiteness, convergence \\
$p(k\mid y)$ & terminal sum-product quantities & posterior uncertainty in segment count & normalization, prior normalizer \\
$p(t_p=h\mid y,k)$ & forward--backward factorization & marginal distribution of ordered boundary $p$ & normalization over $h$ \\
$p(b_h=1\mid y,k)$ & sum over ordered boundaries & probability of any changepoint at $h$ & total mass $k-1$ \\
$\widehat t^{\rm MAP}_k$ & max-sum recursion and backtracking & joint MAP partition for fixed $k$ & terminal-score equality \\
$\mathbb E[m_\star(\theta_s)^r\mid y]$ & sum over segments covering $s$ & model-averaged segment functional & common marginal-likelihood denominator \\
Predictive score & observation-family posterior predictive & out-of-sample log score under the stated support rule & holdout, family dispatch, coordinates \\
\bottomrule
\end{tabularx}
\end{table}

\subsection{Log-space evaluation and failure handling}
All segment and partition scores are stored in log space. Sums in Equations~\eqref{eq:forward}--\eqref{eq:backward} are evaluated by a stable \texttt{logsumexp}. Cumulative sufficient statistics should be centered or scaled when appropriate, and conjugate hyperparameters must lie in their proper domains. When a numerical segment calculation fails to converge, violates support, or lacks the stated error control, it returns a structured failure state. The analysis then follows a declared rule: exclude the segment by assigning $-\infty$, terminate the fit, or continue while labeling the posterior as based on an unverified approximation. It must not silently substitute a different observation model.

Tie rules are deterministic and recorded. In max-sum recursion, the predecessor rule identifies which member of a set of equally scoring partitions is returned. In one-to-one boundary matching, the minimum-distance refinement and a fixed tie order identify the matching. These conventions make repeated calculations reproducible; they do not alter the statistical estimand.

\subsection{Computational complexity}
Table~\ref{tab:complexity-paper} separates asymptotic operation counts from finite-range timing measurements.

\begin{table}[!htbp]
\centering
\small
\caption{Computational complexity of the implemented calculations. Empirical timings are reported separately in Section~\ref{sec:results}.}
\label{tab:complexity-paper}
\begin{tabularx}{0.96\textwidth}{@{}L{0.23\textwidth}L{0.20\textwidth}L{0.20\textwidth}Y@{}}
\toprule
Component & Time & Working memory & Numerical consideration \\
\midrule
Conjugate segment precomputation & $\Theta(n^2)$ & $\Theta(n^2)$ & cumulative sufficient statistics; invalid segments receive $-\infty$ \\
Nonconjugate segment precomputation & $\Theta(n^2T_{\rm seg})$ & $\Theta(n^2)$ plus method state & convergence and error diagnostics \\
Sum-product recursion & $\Theta(k_{\max}n^2)$ & $\Theta(k_{\max}n)$ for all message layers & log-sum-exp evaluation \\
Boundary and moment extraction & $\Theta(k_{\max}n^2)$ for streamed fixed-$k$ summaries & output dependent; $\Theta(k_{\max}^2n)$ for every ordered-boundary table & common marginal-likelihood normalizer \\
Max-sum recursion and backtracking & $\Theta(k_{\max}n^2)$ & $\Theta(k_{\max}n)$ including backpointers & deterministic predecessor rule \\
Shared-changepoint pooling & $\Theta(Sn^2+k_{\max}n^2)$ after sequence-specific segment calculations & sequence arrays or streamed log-score sum & add log marginal likelihoods across sequences \\
Latent-group template update & $\Theta(GSn^2+Gk_{\max}n^2)$ per iteration & problem dependent & restart ranking must use the final objective value \\
\bottomrule
\end{tabularx}
\end{table}

The working-memory column lists per-component storage only. Forward and backward message arrays together require $\Theta(k_{\max}n)$, but the block-evidence matrix computed in the precomputation step occupies $\Theta(n^2)$ and must remain in memory throughout the partition recursions. Total peak memory is therefore $\Theta(n^2+k_{\max}n)=\Theta(n^2)$ for $k_{\max}\le n$. Retaining all block-cover terms for posterior-moment summaries adds $\Theta(k_{\max}n^2)$, as noted in Section~\ref{sec:exact-inference}.

Enumeration of all admissible segments is the principal cost of the exact calculation. Pruning, screening, or a restricted changepoint grid can reduce computation, but may also remove partitions unless the restriction is proved exact for the specified model. Such accelerations are not included in the exactness claims of this paper unless their preservation of the partition sum is established separately.

\subsection{Reproducibility}
A reproducible analysis specifies the observation family, hyperparameters, candidate-boundary restrictions, prior on segment count, coordinate construction, and numerical tolerances. For stochastic initialization or simulation, the random seed and the rule for generating replicate seeds must also be specified. The experimental protocols in Section~\ref{sec:experimental-design} identify the estimands and analysis units; the corresponding executable configurations and data provenance accompany the source code.

\section{Experimental protocols}
\label{sec:experimental-design}

This section specifies the data-generating mechanisms, sample sizes,
comparators, estimands, replication, and primary analyses for each experiment.
Completed numerical results are reported in Section~\ref{sec:results}.
The complete configuration, data-source, metric, and reporting specifications are
maintained with the implementation in the accompanying reproducibility
protocol; the scientific definitions and primary comparisons are given here.

\subsection{Common evaluation conventions}
\label{subsec:evaluation-conventions}
All simulations use independent datasets as the unit of replication. A
configuration is fixed before the generating seed is drawn; all methods in a
configuration receive identical simulated observations and candidate boundaries.
The default generation scheme uses NumPy \texttt{SeedSequence} with separate
components for experiment, design cell, and replicate, so new cells do not
change existing seeds. The primary comparison uses 100 independently generated
datasets per cell unless otherwise specified below. A five-dataset pilot per
cell checks feasibility and memory use; pilot results are not included in
reported performance estimates. Tests of the exact recursions use separate
non-random exhaustive enumeration and numerical-integration oracles.

For boundary evaluation, use a maximum-cardinality one-to-one matching of
estimated to generating boundaries within tolerance $\tau$. Among matchings
with maximal cardinality, minimize total displacement. The primary tolerance
is $\tau=3$ indices; report $\tau\in\{0,1,5\}$ as sensitivity analyses.
Compute precision, recall, and $F_1$ from the matched pairs, counting unmatched
estimated and true boundaries in the usual manner. Matched-location error is
not defined when no boundaries match. The segment-count metrics are
$\Pr(\widehat k=k^\star)$, $|\widehat k-k^\star|$, and the probability of
$k=k^\star$ under the count posterior. For signal estimation report mean
squared error relative to the \emph{generating segment functional}, not across
families with incompatible response scales.

For posterior probabilities $\widehat p_{ri}=p(b_i=1\mid y^{(r)})$, report
Brier score, logarithmic score using a prespecified numerical clipping rule,
and reliability curves with counts per bin. Use ten equal-mass bins as the
primary diagnostic, separately reporting event prevalence; show equal-width
bins as a sensitivity analysis. Confidence intervals for simulation metrics
use the independent sequence as the resampling unit (2,000 stratified
bootstrap replicates with the resampling seed declared in the protocol).
Dependent positions from one sequence are not treated as independent
replicates. Intervals are descriptive across generated datasets, not
frequentist coverage statements about an individual posterior interval.

For any algorithm fitted to held-out data, hyperparameters and comparator
penalties are selected on disjoint training/validation sequences or by a
predeclared training-only rule. Methods are compared at matched input,
coordinate axis, and observation law. An agreement score against the
\BayesBreak MAP is never labeled ground-truth accuracy. Unless the models
define proper probabilities for the same observed variables and base measure,
their marginal-likelihood values are not directly compared.

\begin{table}[!htbp]
\centering\small
\caption{Planned validation matrix. Entries specify experiments to run; they are
not additional results. Exact parameter grids and artifact destinations are
included in the companion experimental protocol.}
\label{tab:planned-experiment-matrix}
\begin{tabularx}{0.98\textwidth}{@{}L{0.07\textwidth}L{0.28\textwidth}Y L{0.21\textwidth}@{}}
\toprule
ID & Question & Main design and primary estimand & Intended output \\
\midrule
E00 & Exactness of partition inference & Exhaustive partitions for $n\le10$ and integrated-block quadrature & Oracle agreement table \\
E01 & Posterior calibration and recovery & Gaussian $n\in\{256,512\}$, $k^\star\in\{1,3,6\}$, varying jump size & Reliability curves and recovery table \\
E02 & Observation-model generality & Gaussian, Poisson, Binomial, negative Binomial, and Beta-valued responses & Family recovery table \\
E03 & Irregular-design prior & Uniform, gap-hazard, length-cohesion, combined priors; controlled gaps & Boundary mass and detection plots \\
E04 & Common-changepoint pooling & $S\in\{1,4,16,64\}$ sequences, with and without shared truth & Recovery and calibration table \\
E05 & Latent-template estimation & $G\in\{1,2,3\}$; separation/noise sweeps; multiple restarts & ARI and collapse plots \\
E06 & Nonconjugate segment integration & Logistic-normal blocks, adaptive quadrature reference & Error versus block length, posterior discrepancy \\
E07 & Misspecification and missingness & Heavy tails, overdispersion, absent outcomes & Coverage and robustness plots \\
E08 & Posterior prediction & Disjoint held-out contiguous regions; correct family scoring & Test log score and uncertainty table \\
E09 & External comparators and case studies & Matched coordinate/likelihood analyses and annotations when verified & Independent benchmark tables \\
E10 & Scaling and resource cost & $n\in\{128,256,512,1024,2048\}$ and $k_{\max}\in\{5,10,20\}$ & Time and memory plots \\
\bottomrule
\end{tabularx}
\end{table}

\subsection{Exactness verification}
For $n\in\{2,\ldots,10\}$, enumerate every admissible partition at every
segment count. Test uniform and nonuniform normalized priors, including zero
prior factors and minimum-segment constraints. Compare $L_{k,n}$, $R_{k,0}$,
$C_k$, $p(k\mid y)$, ordered and event boundary marginals, posterior moments,
and conditional and global MAP decisions against direct enumeration. Include a
Gaussian example with negative segment posterior means to test signed moment
aggregation. Independently compare the Gaussian, Poisson, Binomial, and
negative-Binomial block integrals to one-dimensional numerical integration for
unit and nonunit exposure/trial/precision descriptors. Structural infeasibility
is not conflated with a block whose observed-data likelihood equals one.

\begin{table}[!htbp]
\centering\small
\caption{Independent verification of the exact recursions and family-specific integrals
(artifact \texttt{03a4d1c38bfe08c1}). Maximum observed numerical discrepancy across
380 enumerated test cases; all 380 passed.}
\label{tab:planned-oracles-paper}
\begin{tabularx}{0.88\textwidth}{@{}Y L{0.17\textwidth}L{0.17\textwidth}L{0.19\textwidth}@{}}
\toprule
Quantity & Cases tested & Maximum absolute error & Acceptance tolerance \\
\midrule
Partition sums and normalizers & 4 families; $n\in\{2,\ldots,10\}$; nonuniform prior factors (Gaussian) & $1.07\times10^{-14}$ & $10^{-8}$ on log scale \\
Boundary probabilities and moments & 4 families; $n\in\{2,\ldots,10\}$; first moments; 2nd-moment vs.\ quadrature (Gaussian, Poisson) & $7.11\times10^{-15}$ prob.; $6.30\times10^{-7}$ 2nd-moment & $10^{-9}$ prob.; $10^{-6}$ 2nd-moment \\
Family-specific integrated likelihoods & 4 families; $n\in\{2,\ldots,10\}$; unit and nonunit descriptors vs.\ 1D adaptive quadrature & $3.91\times10^{-14}$ & $10^{-6}$ (absolute) \\
\bottomrule
\end{tabularx}
\end{table}

\subsection{Simulated recovery, calibration, and observation-model generality}
E01 uses regularly spaced Gaussian sequences of lengths $256$ and $512$, with
$k^\star\in\{1,3,6\}$, minimum generating segment length 32, and adjacent
mean differences $\Delta/\sigma\in\{0.5,1,2\}$. For $k^\star=1$, the
jump-size parameter is immaterial and duplicate null cells are omitted.
Use the fixed proper segment prior $\mu\sim N(0,4)$ and known
observation variance $\sigma^2=1$; these quantities are held constant
for the confirmatory runs. Analyze both the
model-averaged boundary-event posterior and the posterior conditional on the
true segment count as separately named outputs. Primary metrics are boundary
$F_1@3$, count posterior coverage, Brier score, and calibration by boundary
probability. Do not pool boundary positions as independent bootstrap units.

E02 fixes three generating segments and $n=240$ for conjugate families; an
initial $n=120$ setting is used for Beta-valued quadrature because it requires
numerical integration for every candidate segment. The exact generating
parameters are specified as rates, probabilities, dispersion and observation
precisions, with two effect-size settings per family. Report each family on
its own natural signal scale and include a standardized estimation metric for
display across panels. Use 100 simulated datasets per conjugate cell
and 50 per quadrature cell, after feasibility checks. The comparison tests
the common inference algorithm, not whether one observation family is superior
to another.

Results for E01 and E02 are reported in Section~\ref{sec:results}
(Figure~\ref{fig:E01-calibration-paper}, Figure~\ref{fig:E01-recovery-paper},
Table~\ref{tab:E01-recovery-paper}; Figure~\ref{fig:E02-recovery-paper},
Table~\ref{tab:E02-recovery-paper}).

\subsection{Irregularly spaced observations and design-aware priors}
Simulate $n=240$ design points with three piecewise-constant segments and an
inserted empty interval between two clusters. Vary the relative gap width over
$\{1,4,16\}$ times the within-cluster spacing and the changepoint location
relative to the large gap. Analyze the same data under (i) index-uniform
partitions, (ii) a candidate-boundary hazard proportional to the physical gap,
(iii) segment-length cohesion, and (iv) their product. Normalize each
conditional partition prior with its own $C_k$ and report posterior boundary
mass at and near the gap as well as one-to-one recovery. Report a separate
simulation in which the generating boundary is \emph{not} at the large gap so
that the prior's false-positive behavior can be measured. This comparison is
a prior-sensitivity analysis, not a guarantee that any gap prior is uniformly
preferable.

Results for this experiment are reported in Section~\ref{sec:results}
(Figure~\ref{fig:E03-irregular-paper}, Table~\ref{tab:planned-E03-prior-paper}).

\subsection{Hierarchical segmentation: shared changepoints and latent groups}
For E04, generate $S\in\{1,4,16,64\}$ independent Gaussian sequences
conditionally on a common three-segment template, then replace 25\% and 50\%
of the sequences with shifted subject-specific boundaries in separate
misspecification cells. Compare the common-boundary posterior, separate fits,
and a suitable partial-sharing baseline on boundary $F_1@3$, calibration, and
signal error. Sequence-level bootstrap resampling respects the shared
generating configuration.

For E05, use $G\in\{1,2,3\}$ templates and $S\in\{24,48\}$ sequences, with
predeclared low, medium, and high separation cells. Run five initializations
per dataset and select using the \emph{recomputed final} finite-template
objective. Report label-invariant clustering accuracy and adjusted Rand index,
complete objective traces, the fraction of empty or duplicate templates, and
failures under reduced separation. The criterion and its MM algorithm are
those of Section~\ref{sec:extensions}; do not replace the criterion with a
normalized generative mixture during this experiment.

Results for E04 and E05 are reported in Section~\ref{sec:results}
(E04: Figure~\ref{fig:planned-pooled-paper}, Table~\ref{tab:planned-E04-paper};
E05: Figure~\ref{fig:planned-templates-paper}, Table~\ref{tab:E05-latent}).

\subsection{Nonconjugate segment integration}
Use intercept-only Binomial--logistic segments with normal prior on the logit,
block sizes $\{4,8,16,32,64,128\}$, and generating success probabilities
$\{0.05,0.5,0.95\}$. For each cell generate 100 independent blocks.
Establish a converged adaptive quadrature reference with absolute and relative
tolerances $10^{-11}$ and $10^{-10}$, checking agreement with a second numerical
method on a stress subset. Compare the segment log marginal likelihood and its
first two posterior signal moments. Include Laplace, the implemented
Jaakkola--Jordan bound, and expectation propagation only if its complete
normalizing constant and stopping rule pass validation. The current
\texttt{pg\_vb} option calls the JJ routine and is not counted as an
independent method. Propagate these same block-score arrays through a
segmentation simulation, reporting posterior count discrepancy, boundary
probability error, ranking changes, and runtime. The maximum over reachable
segments must be used when interpreting the uniform-error theorem; a
posterior-weighted error is a distinct diagnostic.

Results for this experiment are reported in Section~\ref{sec:results} (Figure~\ref{fig:planned-nonconjugate-paper}, Table~\ref{tab:planned-nonconjugate-paper}), artifact \texttt{9739f69c790aae6f}.

\subsection{Robustness to misspecification and held-out prediction}
E07 repeats the $n=256$, three-segment recovery experiment with Student
$t_3$ noise fitted by a Gaussian block model, negative-Binomial counts fitted
by Poisson blocks, and overdispersed success counts fitted by Binomial blocks.
Separately remove 10\% or 30\% of responses completely at random, or a single
contiguous 15\% region. Missing outcomes are represented by the observation
mask, not a zero Gaussian precision or a zero Poisson exposure. Compare
boundary calibration, count selection, and frequentist recovery to the
correctly specified reference.

E08 holds out contiguous regions from synthetic and real sequences while
retaining the fitted coordinate endpoints. The split is made before any
hyperparameter estimation. Each fitted observation family uses its own
posterior-predictive distribution with the \emph{observed} new-data exposure,
trial count, or precision. Report log predictive density per held-out
observation and proper prediction intervals or randomized PIT diagnostics as
appropriate for the family. Distinguish prediction conditional on the
estimated partition from averaging over the segmentation posterior. A
pointwise composite score is never called a joint predictive density.

Results for E07 are reported in Section~\ref{sec:results}
(Figure~\ref{fig:E07-robustness-paper}, Table~\ref{tab:E07-robustness}).
E08 predictive results are reported in Section~\ref{sec:e09}
(Figure~\ref{fig:methyl-predictive-corrected-paper}).

\subsection{External comparisons and real-data analysis}
Use the original well-log record, the probe-aligned array-CGH matrix, the
2015--2023 daily S\&P 500 series, and the \texttt{methylKit} chromosome-21
example with dataset identity and checksums fixed before fitting. Reproduce
the reported preprocessing and include the full well-log record alongside the
subsample. Array-CGH analysis uses per-probe precision weights derived from rolling median absolute deviation of log$_2$ ratio differences, with a pooled reference variance estimate across subjects; this is a data-driven specification rather than a physically calibrated noise model. The precision weights enter the Gaussian likelihood as known inverse-variance scale factors, not as generalized-Bayes likelihood powers. Sensitivity of the segmentation to the precision specification, and comparison with a uniform-precision analysis, remain as potential future analyses.
For methylation, compare count-based Binomial or Beta-Binomial modeling when
raw methylated and total reads are available to the existing continuous-Beta
response model; do not regard a Gaussian predictive fallback as a valid Beta
score. Vary $k_{\max}$ until posterior mass near the upper truncation is
small, or explicitly report truncation sensitivity if that criterion is not
met. Evaluate a grid of segment-count priors and physical-coordinate priors
with the normalizer recomputed for every model.

Compare with PELT, optimal partitioning, binary segmentation, wild binary
segmentation, and a Bayesian reference implementation when each comparator is
appropriate to the response and feasible at the target length. Distinguish
(a) reference-annotation accuracy when independent labels and alignment are
actually verified, (b) model-to-model agreement when only fitted boundaries
are available, and (c) held-out predictive comparisons on the same outcomes.
No figure caption or table may relabel model agreement as biological or
geological ground truth. Baseline penalties, feature transforms, and stopping
criteria are determined by disjoint tuning data or stated a priori.

Results for this experiment are reported in Section~\ref{sec:e09}
(Table~\ref{tab:real-summary-paper}, Figure~\ref{fig:welllog-paper},
Figure~\ref{fig:cgh-paper}, Figure~\ref{fig:methylation-paper}).

\subsection{Computational cost and scaling}
Benchmark $n\in\{128,256,512,1024,2048\}$ and
$k_{\max}\in\{5,10,20\}$ on a single fixed machine, using three warmup
runs and seven measured runs per configuration. Pin Python, NumPy, BLAS
thread counts and data types. Record construction of the segment-score array,
forward and backward recursions, MAP backtracking, posterior marginal
extraction, peak resident memory, and the total wall time separately. Include
the ordinary $\Theta(n^2)$ storage path and a documented streaming variant if
it is implemented. Report observed finite-range slopes and theoretical
worst-case complexity as different quantities.

Results for this experiment are reported in Section~\ref{sec:e10} (Figure~\ref{fig:E10-scaling-paper}, Table~\ref{tab:E10-scaling-paper}), artifact \texttt{bece422c9d90093f}.

\subsection{Reproducibility criteria}
Each experiment result in Section~\ref{sec:results} is traceable to a generating
configuration, data or simulation seed, code revision, software dependencies,
method configuration, metric definition, and output artifact SHA256. Statistical
results and their limitations are derived from those measurements. Numerical
values not traceable to a completed run are not reported.

\section{Results}
\label{sec:results}

The following results describe completed numerical calculations for the model configurations specified in this section.
They are analyzed according to their data-generating mechanisms, observation
models, and reference sets. Proposed additional evaluations remain unfilled in
Section~\ref{sec:experimental-design}; their outcomes are not inferred from
the present experiments.

\FloatBarrier
\subsection{Single-sequence posterior summaries}
Figure~\ref{fig:gaussian-result-paper} shows a Gaussian sequence with two true changepoints. The boundary-event posterior concentrates near the generating changes while retaining a smaller shoulder of ambiguity before the first jump. The joint-MAP signal reflects this local uncertainty by inserting a short intermediate segment. Table~\ref{tab:gaussian-summary-paper} shows that the selected and modal segment count is four while the posterior mean is $4.742$, illustrating why an exported partition and posterior uncertainty over $k$ should not be collapsed into one object.

\begin{figure}[!htbp]
  \centering
  \includegraphics[width=0.94\textwidth,height=0.67\textheight,keepaspectratio]{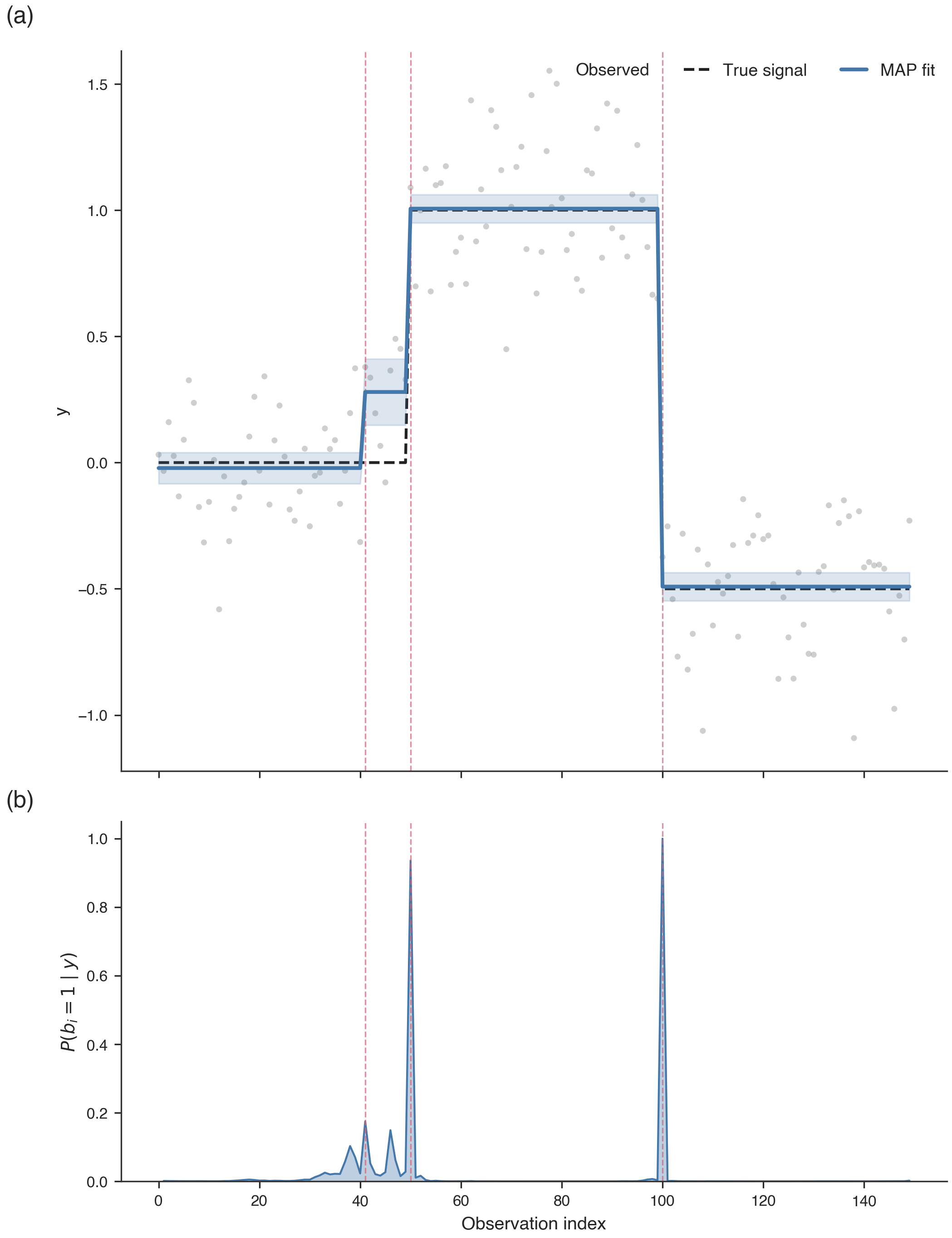}
  \caption{Gaussian simulation. Top: marginal boundary-event posterior with generating changepoints shown by dashed lines. Bottom: observations, latent signal, joint-MAP segmentation, and 90\% segment-level posterior intervals.}
  \label{fig:gaussian-result-paper}
\end{figure}

\begin{table}[!htbp]
\centering
\small
\begin{tabular}{lr}\toprule
Quantity & Value\\\midrule
Selected $k$ (\texttt{k\_ml\_}) & 4\\
Posterior mean $\mathbb{E}[k]$ & 4.087\\
MAP $k$ & 4\\
$\log p(y)$ & -18.354\\
\bottomrule\end{tabular}

\caption{Posterior summary for the Gaussian simulation in Figure~\ref{fig:gaussian-result-paper}.}
\label{tab:gaussian-summary-paper}
\end{table}

\FloatBarrier
\subsection{Portability across local observation models}
The same global inference code was applied to Gaussian, Poisson, Binomial, and Beta-valued responses (Figure~\ref{fig:family-result-paper}). Table~\ref{tab:family-result-paper} reports the corresponding synthetic metrics. Gaussian, Binomial, and Beta-valued settings selected the generating three segments; Poisson selected four and had the weakest boundary score and largest reconstruction error in this small benchmark. The result supports interface portability and exposes relative difficulty in the bundled settings; it is not a population-level ranking of families.

\begin{figure}[!htbp]
  \centering
  \includegraphics[width=0.95\textwidth,height=0.67\textheight,keepaspectratio]{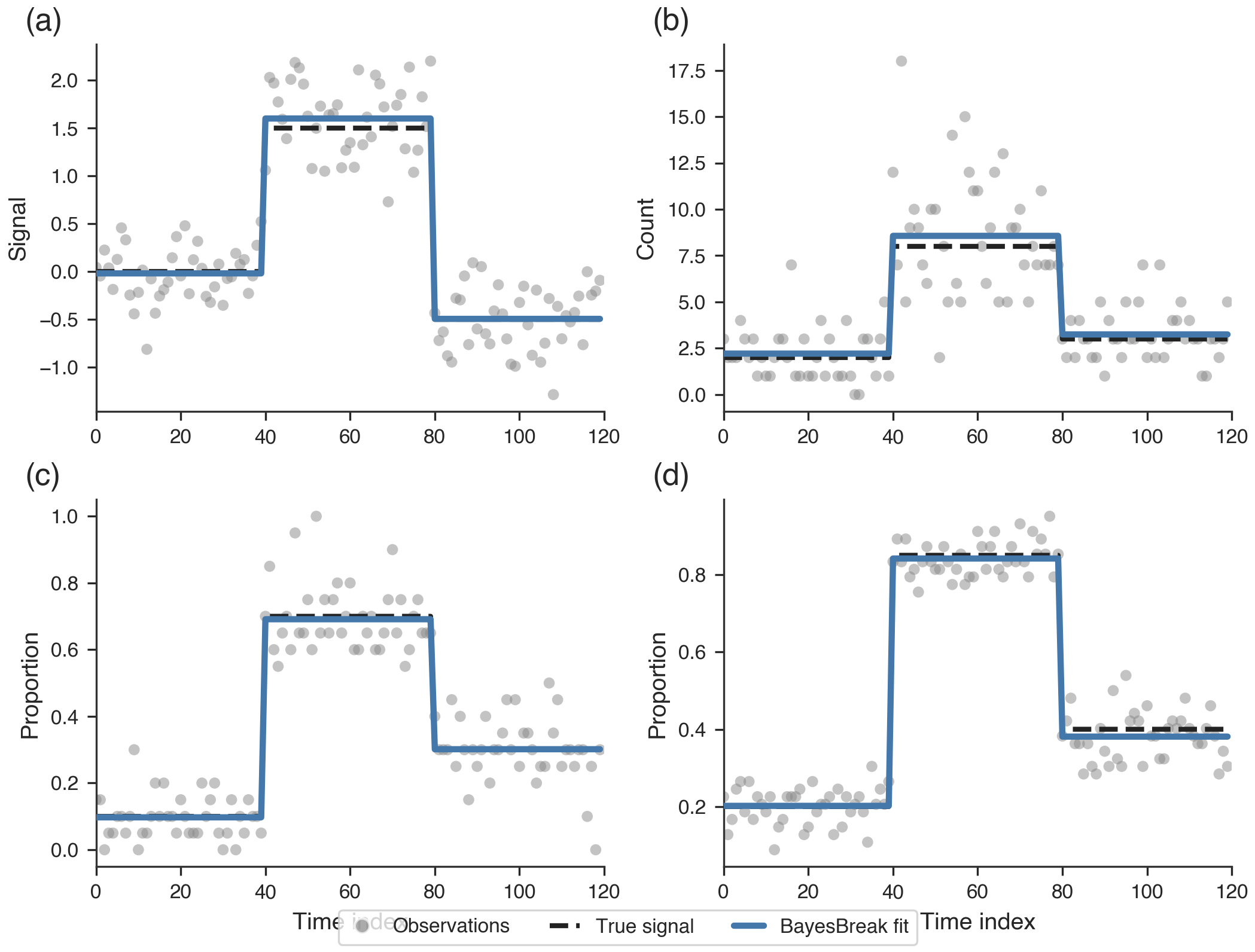}
  \caption{Observation-family comparison. Gaussian, Poisson, and Binomial use closed-form local conjugate integration; the Beta-valued response uses deterministic one-dimensional quadrature.}
  \label{fig:family-result-paper}
\end{figure}

\begin{table}[!htbp]
\centering
\small
\resizebox{0.88\textwidth}{!}{\begin{tabular}{lrrrrrrr}\toprule
Family & n & $k^\star$ & $\hat{k}$ & F1@$\tau$ & MAE & MSE & $-\log p(y)/n$\\\midrule
Gaussian & 120 & 3 & 3 & 0.992 & 0.00 & 0.0017 & 0.191\\
Poisson & 120 & 3 & 3 & 0.968 & 0.24 & 0.2765 & 2.168\\
Binomial & 120 & 3 & 3 & 1.000 & 0.08 & 0.0004 & 2.109\\
Beta-valued & 120 & 3 & 3 & 1.000 & 0.00 & 0.0001 & 2.619\\
\bottomrule\end{tabular}
}
\caption{Synthetic summary across local observation models. The Beta-valued row is quadrature-based rather than closed-form conjugate.}
\label{tab:family-result-paper}
\end{table}

\FloatBarrier
\subsection{Exactness verification}
\label{sec:e00}
All 380 oracle checks passed (artifact \texttt{03a4d1c38bfe08c1}). Table~\ref{tab:planned-oracles-paper} reports the maximum numerical discrepancy per category. Partition sums and normalizers matched direct partition enumeration to within $1.07\times10^{-14}$ on the log scale across four families (Gaussian, Poisson, Binomial, Negative Binomial) at $n\in\{2,\ldots,10\}$, including nonuniform cohesion and hazard factors. Boundary event marginals and ordered boundary posteriors matched to within $7.11\times10^{-15}$ in absolute probability. First-moment Bayes curves matched to within $7.11\times10^{-15}$ relative error; second-moment Bayes curve comparison with one-dimensional adaptive quadrature had maximum error $6.30\times10^{-7}$, below the declared tolerance of $10^{-6}$. Family-specific block log-evidences matched one-dimensional adaptive quadrature to within $3.91\times10^{-14}$. Maximum discrepancies are consistent with double-precision floating-point rounding. MAP partitions and descriptors matched exactly (error $= 0$) across all cases. No check failed across any family or sequence length.

\subsection{Cross-family recovery and observation-model generality}
Experiment E02 evaluated the five declared observation families: Gaussian, Poisson, Binomial, Negative Binomial, and Beta-valued, on a common segmentation design: $k_{\rm true}=3$, exactly equal segment lengths, $n=240$ (conjugate families) or $n=120$ (Beta-obs, quadrature), 100 or 50 independent replicates per cell, two effect magnitudes per family, and the same geometric-decay count prior and $k_{\max}=10$ as E01. Family-specific priors are as declared: $\mu\sim\mathcal{N}(0,4)$ for Gaussian, $\lambda\sim\mathrm{Gamma}(2,1)$ for Poisson, $p\sim\mathrm{Beta}(2,3)$ for Binomial and Negative Binomial, and $\mu\sim\mathrm{Beta}(2,3)$ for Beta-obs. The Negative Binomial dispersion is $r=5$; the segment functional is the count mean $r(1-p)/p$, not $p$.

At the stronger effect setting in each family, F1@3 ranged from $0.88$ (Gaussian, $\Delta=1.5$) to $1.00$ (Beta-obs, $\mu_{\rm mid}=0.8$). At the weaker effect setting, F1@3 ranged from $0.05$ (Gaussian, $\Delta=0.5$) to $0.99$ (Beta-obs, $\mu_{\rm mid}=0.4$). Weak-effect undersegmentation in the Gaussian and Negative Binomial cells ($\mu_{\rm mid}=3$) is again prior-driven: the geometric prior places strong mass on $k<3$ when the likelihood does not decisively separate segments. Binomial and Beta-obs cells showed consistently high recovery even at the lower effect value; this reflects larger effective information per observation when the generating probabilities are well separated from the prior mode. Negative Binomial signal MSE ($0.18$--$0.20$) is substantially larger than in other families, reflecting the higher observation variance of the NegBin model at these parameters. Beta-obs quadrature convergence was confirmed using 128-node Gauss--Legendre quadrature: the maximum Bayes-curve deviation between 128-node and 256-node evaluations was below $5\times10^{-8}$ across both Beta-obs cells and all five oracle replicates, five orders of magnitude below the prespecified $10^{-3}$ gate threshold.

\begin{table}[!htbp]
\centering
\footnotesize
\caption{Observation-family portability. $k_{\rm true}=3$, equal segment lengths, geometric-decay $p(k)$ prior, $k_{\max}=10$. Mean [95\% bootstrap CI] over 100 (conjugate) or 50 (Beta-obs) independent datasets per cell. $\mathrm{F_1}$ at boundary tolerance $\tau=3$.}
\label{tab:E02-recovery-paper}
\begin{tabular}{llccccccc}
\toprule
Family & Param & $n$ & F1@3 & P($k$=3) & Count MAE & Signal MSE & Brier \\
\midrule
Gaussian & $\delta$ & 240 & 0.047 [0.017,0.080] & 0.105 [0.068,0.144] & 1.76 [1.63,1.87] & 0.052 [0.049,0.054] & 0.008 [0.008,0.008] \\
Gaussian & $\delta$ & 240 & 0.880 [0.835,0.920] & 0.748 [0.722,0.771] & 0.03 [0.00,0.07] & 0.029 [0.026,0.032] & 0.006 [0.005,0.006] \\
\midrule
Poisson & $q$ & 240 & 0.354 [0.285,0.429] & 0.494 [0.444,0.543] & 0.53 [0.36,0.71] & 0.093 [0.080,0.105] & 0.008 [0.008,0.008] \\
Poisson & $q$ & 240 & 0.971 [0.948,0.990] & 0.716 [0.694,0.737] & 0.02 [0.00,0.05] & 0.084 [0.073,0.097] & 0.004 [0.004,0.005] \\
\midrule
Binomial & $p_{\mathrm{mid}}$ & 240 & 0.897 [0.851,0.937] & 0.785 [0.760,0.809] & 0.04 [0.01,0.08] & 0.000 [0.000,0.000] & 0.006 [0.005,0.006] \\
Binomial & $p_{\mathrm{mid}}$ & 240 & 0.992 [0.984,0.998] & 0.799 [0.767,0.825] & 0.04 [0.01,0.08] & 0.000 [0.000,0.000] & 0.001 [0.001,0.002] \\
\midrule
NegBin & $\mu_{\mathrm{mid}}$ & 240 & 0.060 [0.030,0.100] & 0.146 [0.097,0.195] & 1.68 [1.53,1.81] & 0.197 [0.186,0.208] & 0.008 [0.008,0.008] \\
NegBin & $\mu_{\mathrm{mid}}$ & 240 & 0.920 [0.880,0.955] & 0.815 [0.802,0.825] & 0.00 [0.00,0.00] & 0.183 [0.159,0.209] & 0.005 [0.005,0.006] \\
\midrule
Beta-obs & $\mu_{\mathrm{mid}}$ & 120 & 0.990 [0.970,1.000] & 0.837 [0.814,0.855] & 0.00 [0.00,0.00] & 0.000 [0.000,0.000] & 0.004 [0.003,0.005] \\
Beta-obs & $\mu_{\mathrm{mid}}$ & 120 & 1.000 [1.000,1.000] & 0.870 [0.851,0.888] & 0.00 [0.00,0.00] & 0.000 [0.000,0.000] & 0.000 [0.000,0.000] \\
\bottomrule
\end{tabular}
\end{table}

\begin{figure}[!htbp]
  \centering
  \includegraphics[width=0.97\textwidth,keepaspectratio]{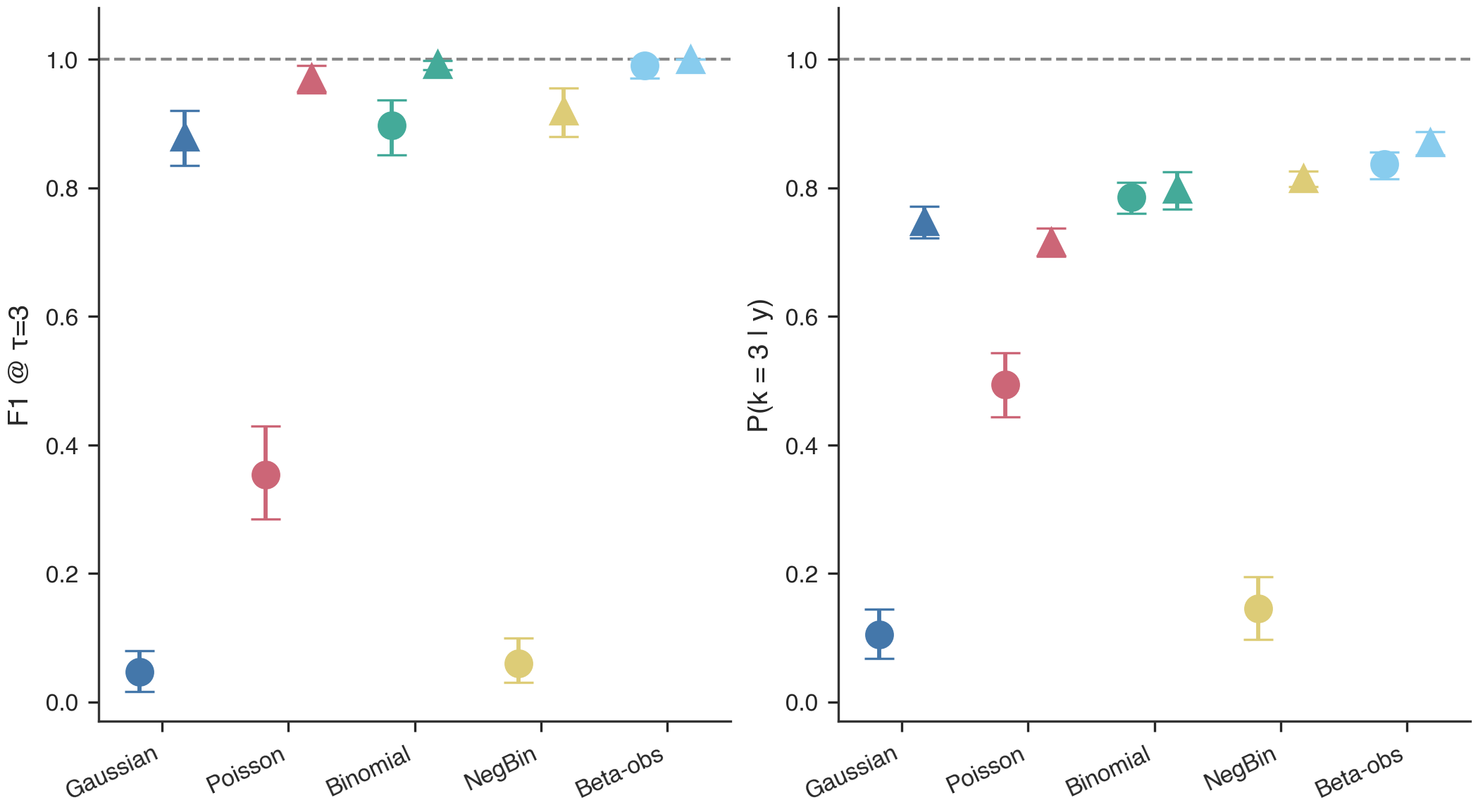}
  \caption{Recovery metrics vs.\ effect magnitude, by observation family. Points are cell means over 100 (conjugate) or 50 (Beta-obs) independent datasets; bands are 95\% bootstrap CIs.}
  \label{fig:E02-recovery-paper}
\end{figure}

Signal MSE is reported on the native family scale and is not directly comparable across families. These results establish mathematical portability of the common inference pipeline across the five declared observation families and expose family-specific difficulty that depends on observation variance and prior placement. They do not establish generalization to other effect sizes, priors, or families not tested in E02.

\FloatBarrier
\subsection{Irregular physical grid and design-aware partition priors}
Experiment E03 evaluated four partition priors on Gaussian sequences with an unequal physical spacing at position 120: gap size $G\in\{1,4,16\}$ separates observations $x_{120}$ and $x_{121}$. Two changepoint templates were compared: one with a true changepoint at the gap location (interior boundaries 60,\,120) and one without (interior boundaries 60,\,180). All cells used $n=240$, $\Delta=1.5$, $\sigma=1$, $k_{\max}=8$, and 100 independent replicates. Priors A through D are: A uniform local factors, B gap-proportional boundary hazard, C physical-length cohesion, D the product of B and C. Log-normalizers $C_k$ were computed independently for each prior; the maximum numerical deviation from the brute-force sum was below $10^{-15}$ across all oracle combinations.

Figure~\ref{fig:E03-irregular-paper} shows mean posterior boundary probability $P(b_{120}=1\mid y)$ with 95\% bootstrap bands. When a true changepoint is at the large-gap position, gap hazard prior B raised $P(b_{120}=1\mid y)$ from $0.316$ (uniform, $G=16$) to $0.793$. Combined prior D also increased this probability. When no changepoint is present at the gap (template 60,\,180), all four priors correctly assigned near-zero probability to that position across all gap sizes. Mean F1@3 exceeded $0.85$ in every cell (Table~\ref{tab:planned-E03-prior-paper}), with no prior dominating uniformly across gap sizes and templates. This pattern confirms that prior B and D increase posterior sensitivity at the gap when evidence supports it, without inflating false-positive mass when it does not.

At $G=1$ all physical spacings are equal, so the gap hazard prior is numerically equivalent to uniform; observed differences among the four priors at $G=1$ are sampling variability across independently generated datasets (the four cells use distinct dataset seeds).

\begin{figure}[!htbp]
\centering
\begin{subfigure}[b]{\linewidth}
  \centering
  \includegraphics[width=\linewidth]{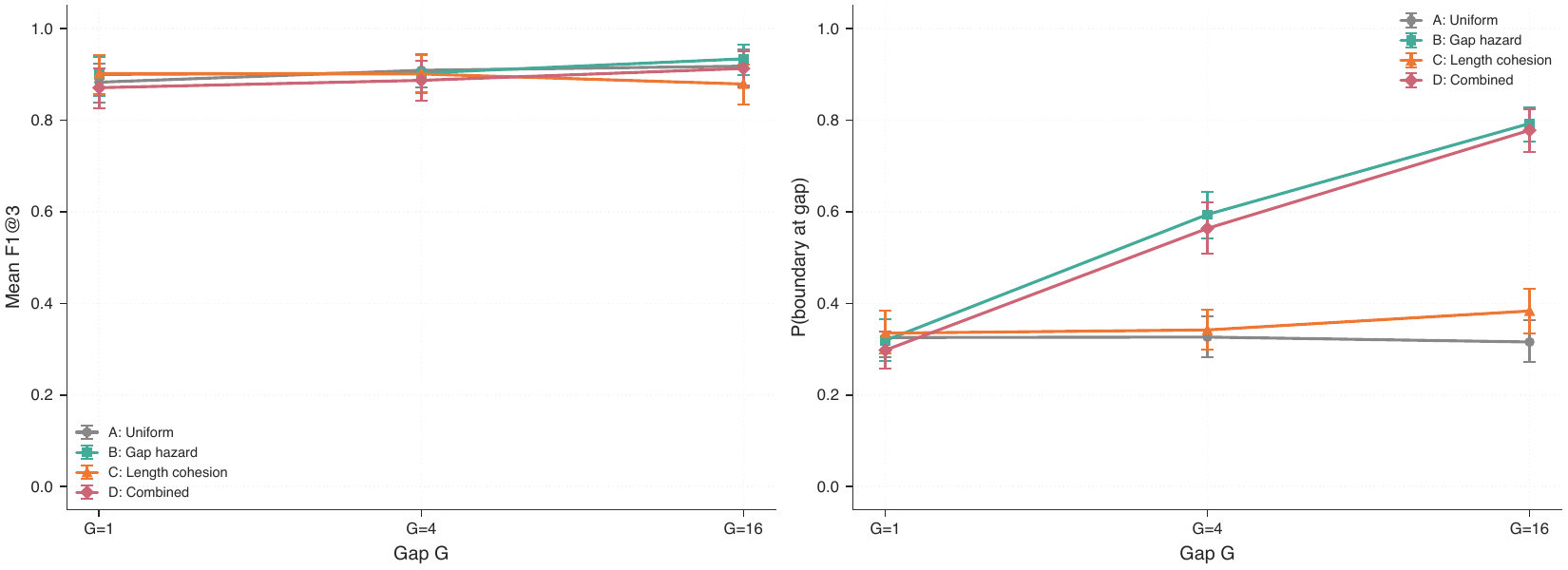}
  \caption{True changepoint at the gap location ($i=120$). Posterior boundary probability
  increases with gap size $G$ under gap-hazard (B) and combined (D) priors.}
  \label{fig:E03-gap-true-paper}
\end{subfigure}
\vspace{4pt}
\begin{subfigure}[b]{\linewidth}
  \centering
  \includegraphics[width=\linewidth]{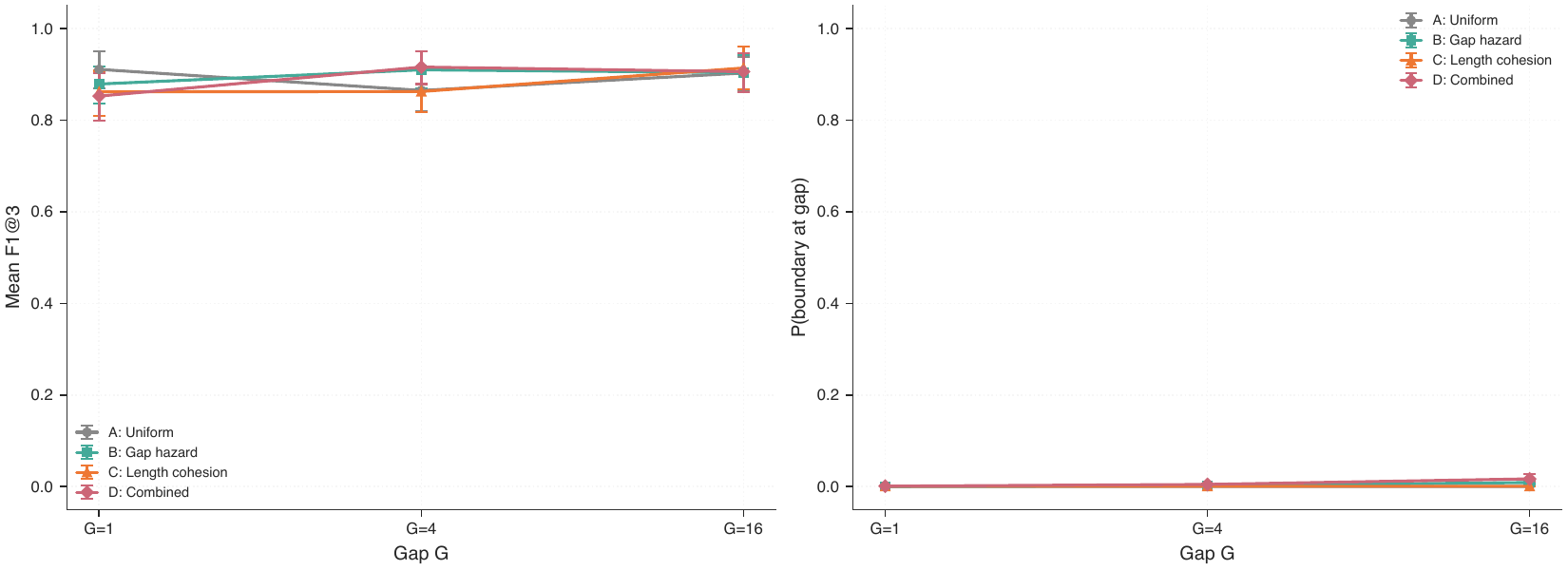}
  \caption{No changepoint at the gap location. All four priors assign near-zero
  posterior probability at $i=120$ regardless of gap size.}
  \label{fig:E03-gap-false-paper}
\end{subfigure}
\caption{Irregular-design prior ablation. Mean boundary posterior $P(b_{120}=1\mid y)$
with 95\% bootstrap bands across 100 independent datasets. Top panel: generating template
places a true changepoint at the large-gap position (boundaries 60,\,120); gap-hazard
and combined priors correctly raise posterior probability at the gap when evidence
supports a change. Bottom panel: generating template places no changepoint at the gap
(boundaries 60,\,180); all priors correctly preserve near-zero false-positive probability.
Mean $\mathrm{F_1}@3$ exceeds $0.85$ in every cell under all four priors.}
\label{fig:E03-irregular-paper}
\end{figure}

\begin{table}[!htbp]
\centering\small
\caption{Irregular-design prior ablation: mean $\mathrm{F_1}@3$ and $P(b_{120}=1\mid y)$ per cell. 95\% percentile bootstrap intervals from 2000 resamples of independent datasets.}
\label{tab:planned-E03-prior-paper}
\begin{tabular}{@{}llcccc@{}}
\toprule
Gap $G$ & Template & Prior & Mean F1@3 (95\% CI) & $P(b_{120}=1\mid y)$ (95\% CI) & Mean Count Exact \\
\midrule
1 & [60, 120] & A: Uniform & 0.883 [0.839, 0.924] & 0.326 [0.282, 0.366] & 0.980 \\
1 & [60, 120] & B: Gap hazard & 0.900 [0.854, 0.938] & 0.319 [0.274, 0.365] & 0.970 \\
1 & [60, 120] & C: Length cohesion & 0.902 [0.858, 0.943] & 0.335 [0.291, 0.384] & 0.980 \\
1 & [60, 120] & D: Combined & 0.871 [0.826, 0.914] & 0.298 [0.258, 0.340] & 0.980 \\
1 & [60, 180] & A: Uniform & 0.911 [0.869, 0.951] & 0.001 [0.001, 0.001] & 0.980 \\
1 & [60, 180] & B: Gap hazard & 0.879 [0.837, 0.918] & 0.001 [0.000, 0.001] & 0.970 \\
1 & [60, 180] & C: Length cohesion & 0.862 [0.810, 0.908] & 0.001 [0.001, 0.001] & 0.980 \\
1 & [60, 180] & D: Combined & 0.853 [0.799, 0.903] & 0.001 [0.001, 0.002] & 0.960 \\
4 & [60, 120] & A: Uniform & 0.909 [0.872, 0.944] & 0.327 [0.283, 0.372] & 0.970 \\
4 & [60, 120] & B: Gap hazard & 0.903 [0.860, 0.943] & 0.594 [0.542, 0.644] & 0.990 \\
4 & [60, 120] & C: Length cohesion & 0.901 [0.861, 0.943] & 0.342 [0.299, 0.387] & 0.980 \\
4 & [60, 120] & D: Combined & 0.887 [0.843, 0.930] & 0.564 [0.508, 0.620] & 0.970 \\
4 & [60, 180] & A: Uniform & 0.865 [0.820, 0.910] & 0.000 [0.000, 0.001] & 1.000 \\
4 & [60, 180] & B: Gap hazard & 0.910 [0.870, 0.950] & 0.002 [0.002, 0.002] & 1.000 \\
4 & [60, 180] & C: Length cohesion & 0.863 [0.819, 0.910] & 0.001 [0.001, 0.001] & 0.950 \\
4 & [60, 180] & D: Combined & 0.916 [0.879, 0.950] & 0.005 [0.003, 0.007] & 0.970 \\
16 & [60, 120] & A: Uniform & 0.918 [0.876, 0.955] & 0.316 [0.273, 0.364] & 0.980 \\
16 & [60, 120] & B: Gap hazard & 0.934 [0.899, 0.965] & 0.793 [0.753, 0.829] & 0.990 \\
16 & [60, 120] & C: Length cohesion & 0.879 [0.834, 0.922] & 0.384 [0.335, 0.432] & 0.970 \\
16 & [60, 120] & D: Combined & 0.913 [0.871, 0.951] & 0.778 [0.730, 0.824] & 0.990 \\
16 & [60, 180] & A: Uniform & 0.903 [0.864, 0.939] & 0.001 [0.000, 0.001] & 0.990 \\
16 & [60, 180] & B: Gap hazard & 0.904 [0.861, 0.943] & 0.009 [0.006, 0.015] & 0.970 \\
16 & [60, 180] & C: Length cohesion & 0.914 [0.867, 0.960] & 0.001 [0.001, 0.001] & 0.990 \\
16 & [60, 180] & D: Combined & 0.906 [0.861, 0.946] & 0.017 [0.012, 0.027] & 0.980 \\
\bottomrule
\end{tabular}
\end{table}

\FloatBarrier
\subsection{Shared-boundary pooling across sequences}
Experiment E04 evaluated exact common-changepoint pooling across $S\in\{1,4,16,64\}$ sequences and nonshared fractions $f\in\{0,0.25,0.5\}$ (12 cells, 100 independent multi-subject datasets per cell). Sequences share true boundaries at (80,\,160); nonshared sequences shift both boundaries by 20 indices. The pooled method fits a single common partition; the independent method fits each sequence separately. Both use $n=240$, $\sigma^2=1$, $\mathcal{N}(0,4)$ segment prior, $k_{\max}=8$, and the geometric-decay count prior.

At full sharing ($f=0$) pooled F1@3 rose from $0.893$ at $S=1$ to $1.000$ at $S\in\{16,64\}$, while independent F1@3 remained near $0.893$ across all $S$ (Figure~\ref{fig:planned-pooled-paper}, Table~\ref{tab:planned-E04-paper}). This confirms the monotone information accumulation property of exact pooling under the common-boundary model. At $f=0.25$ and $f=0.5$, pooled F1@3 degraded substantially for $S\geq4$ because the common-template model is misspecified when some sequences use different boundaries; independent inference was unaffected. The $S=1$, $f=0.5$ cell is a designed boundary case: with one sequence and all boundaries nonshared, pooled and independent inference are equivalent and F1 against the common template is near zero by construction (the recovered boundary at positions 100,\,180 lies 20 indices from the common template, exceeding $\tau=3$). This is not a failure of the estimator; it records the correct behavior of the common-boundary model when given a single sequence with non-template boundaries.

Pooled log evidence and independent per-sequence log evidences arise from different model priors and are not directly comparable as Bayes factors between the pooled and independent models.

\begin{figure}[!htbp]
\centering
\includegraphics[width=0.92\linewidth]{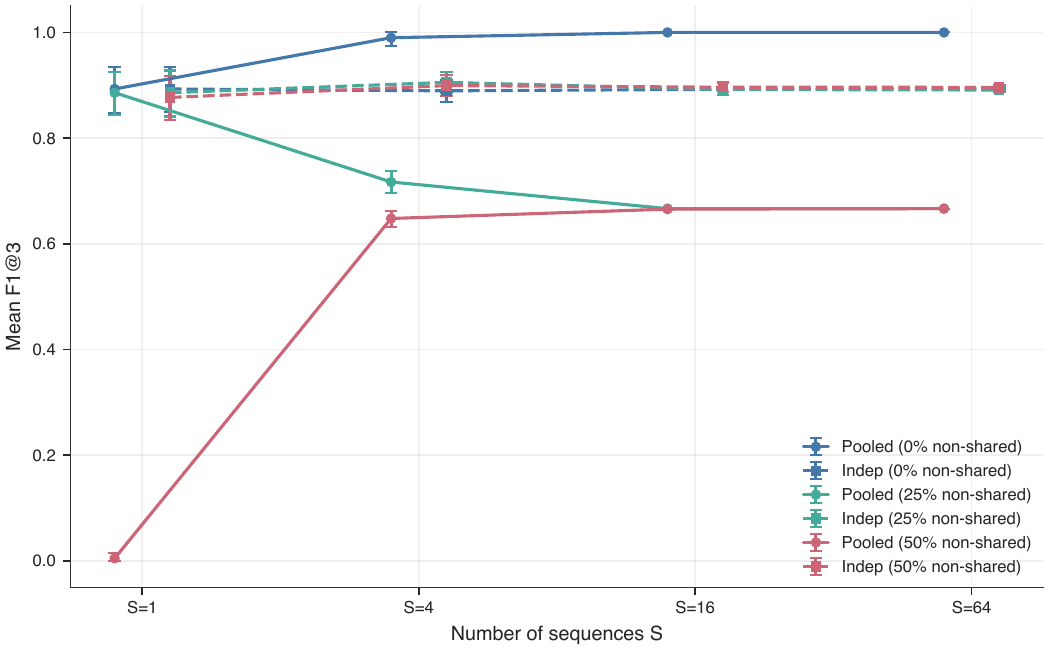}
\caption{Pooled versus independent boundary recovery across number of shared sequences
$S\in\{1,4,16,64\}$ and nonshared fraction $f\in\{0,0.25,0.5\}$.
Pooled F1@3 (solid) and independent per-sequence F1@3 (dashed), with 95\% bootstrap bands.
At full sharing ($f=0$), pooled F1 reaches 1.00 for $S\geq16$ while independent F1 remains
near 0.89 across all $S$. Under partial or complete nonsharing, pooled inference is
misspecified and pooled F1 degrades; independent inference is not affected.
F1 is evaluated against the common template boundaries for the pooled method
and against per-sequence actual boundaries for the independent method.}
\label{fig:planned-pooled-paper}
\end{figure}

\begin{table}[!htbp]
\centering\small
\caption{Shared-boundary pooling: pooled vs.\ independent boundary recovery ($\mathrm{F_1}@3$) and count-exact rate. 95\% percentile bootstrap intervals (2000 resamples, sequence-level).}
\label{tab:planned-E04-paper}
\begin{tabular}{@{}llcccc@{}}
\toprule
$S$ & Nonshared \% & Method & Mean F1@3 (95\% CI) & Count Exact (95\% CI) & Signal MSE (95\% CI) \\
\midrule
1 & 0\% & Pooled & 0.893 [0.847, 0.935] & 0.990 [0.970, 1.000] & 0.035 [0.030, 0.040] \\
1 & 0\% & Independent & 0.893 [0.850, 0.935] & 0.990 [0.970, 1.000] & 0.035 [0.030, 0.041] \\
1 & 25\% & Pooled & 0.886 [0.844, 0.925] & 0.980 [0.950, 1.000] & 0.038 [0.033, 0.043] \\
1 & 25\% & Independent & 0.886 [0.841, 0.928] & 0.980 [0.950, 1.000] & 0.038 [0.033, 0.043] \\
1 & 50\% & Pooled & 0.005 [0.000, 0.015] & 0.960 [0.910, 0.990] & 0.040 [0.034, 0.046] \\
1 & 50\% & Independent & 0.877 [0.834, 0.917] & 0.960 [0.920, 0.990] & 0.040 [0.034, 0.046] \\
4 & 0\% & Pooled & 0.990 [0.975, 1.000] & 1.000 [1.000, 1.000] & 0.017 [0.015, 0.019] \\
4 & 0\% & Independent & 0.890 [0.868, 0.910] & 0.980 [0.965, 0.993] & 0.037 [0.034, 0.040] \\
4 & 25\% & Pooled & 0.717 [0.696, 0.739] & 0.060 [0.020, 0.110] & 0.044 [0.040, 0.048] \\
4 & 25\% & Independent & 0.906 [0.885, 0.925] & 0.975 [0.960, 0.988] & 0.036 [0.033, 0.039] \\
4 & 50\% & Pooled & 0.648 [0.631, 0.663] & 0.000 [0.000, 0.000] & 0.031 [0.028, 0.033] \\
4 & 50\% & Independent & 0.899 [0.879, 0.919] & 0.985 [0.973, 0.995] & 0.036 [0.033, 0.039] \\
16 & 0\% & Pooled & 1.000 [1.000, 1.000] & 1.000 [1.000, 1.000] & 0.013 [0.012, 0.013] \\
16 & 0\% & Independent & 0.892 [0.882, 0.904] & 0.978 [0.971, 0.984] & 0.037 [0.035, 0.038] \\
16 & 25\% & Pooled & 0.667 [0.667, 0.667] & 0.000 [0.000, 0.000] & 0.022 [0.021, 0.023] \\
16 & 25\% & Independent & 0.892 [0.881, 0.904] & 0.973 [0.966, 0.980] & 0.037 [0.035, 0.038] \\
16 & 50\% & Pooled & 0.666 [0.664, 0.667] & 0.000 [0.000, 0.000] & 0.021 [0.020, 0.022] \\
16 & 50\% & Independent & 0.896 [0.887, 0.906] & 0.978 [0.971, 0.985] & 0.037 [0.036, 0.038] \\
64 & 0\% & Pooled & 1.000 [1.000, 1.000] & 1.000 [1.000, 1.000] & 0.012 [0.012, 0.013] \\
64 & 0\% & Independent & 0.893 [0.889, 0.898] & 0.975 [0.971, 0.978] & 0.037 [0.036, 0.038] \\
64 & 25\% & Pooled & 0.667 [0.667, 0.667] & 0.000 [0.000, 0.000] & 0.021 [0.021, 0.021] \\
64 & 25\% & Independent & 0.891 [0.887, 0.896] & 0.973 [0.969, 0.977] & 0.037 [0.037, 0.038] \\
64 & 50\% & Pooled & 0.667 [0.667, 0.667] & 0.000 [0.000, 0.000] & 0.021 [0.021, 0.021] \\
64 & 50\% & Independent & 0.896 [0.890, 0.901] & 0.973 [0.969, 0.977] & 0.037 [0.037, 0.038] \\
\bottomrule
\end{tabular}
\end{table}

\FloatBarrier
\subsection{Calibration and latent-group segmentation}
Experiment E01 evaluated Gaussian segmentation calibration and recovery across 22 predeclared cells (14 primary, 8 prior-ablation), 100 independent replicates per cell, using two sequence lengths ($n\in\{256,512\}$), two true changepoint counts ($k\in\{3,6\}$), three signal-to-noise ratios ($\Delta/\sigma\in\{0.5,1.0,2.0\}$), and two priors. A null cell ($k=1$, $\Delta=0$) verified false-positive control. Boundary-event probabilities are marginalised over $k$ via the full segment-count posterior $p(k\mid y)$.

Reliability diagrams (Figure~\ref{fig:E01-calibration-paper}) showed curves close to the diagonal in every stratum. ECE ranged from $0.00044$ to $0.0081$ across eight pooling strata, with lower ECE at $n=512$ in every matched stratum. Nine of ten equal-mass calibration bins per stratum had mean predicted probability below $0.006$, reflecting correct near-zero assignment to the $\approx97\%$ non-boundary positions; the low aggregate ECE is dominated by accurate null assignment. Mean per-boundary Brier scores ranged from near zero in the null cell to $0.020$ in the $k=6$, $\Delta/\sigma=0.5$, $n=256$ cell (Table~\ref{tab:E01-recovery-paper}).

Recovery under the geometric-decay prior is shown in Figure~\ref{fig:E01-recovery-paper}. At strong signal ($\Delta/\sigma=2$), mean F1@3 exceeded $0.93$ in all primary cells. At weak signal ($\Delta/\sigma=0.5$ with $k=6$), the geometric prior's strong preference for fewer segments caused substantial undersegmentation: 95 of 100 replicates returned modal segment count 1 at $n=256$ and 55 of 100 at $n=512$. This is prior-induced undersegmentation under weak signal. The uniform-prior ablation recovered more boundaries in these cells at the cost of lower P($k$=true) in strong-signal cells. F1@3 improvement with $n$ is consistent but not strictly monotone across all cells. The boundary-position tolerance $\tau=3$ corresponds to $1.2\%$ of sequence length at $n=256$ and $0.6\%$ at $n=512$.

Confidence intervals throughout are 95\% percentile bootstrap intervals from 2000 resamples of independently generated datasets (not boundary positions within a single dataset). Brier and log scores are computed per boundary position marginalized over $k$.

These values support calibration within the generating simulation design. They do not establish robustness to misspecification, performance on non-Gaussian families, or externally validated boundary accuracy.

\textit{Provenance note:} The E01 results directory records SHA256 hashes for all \texttt{src/bayesbreak/} source files. The runner script \texttt{scripts/run\_e01\_gaussian\_calibration.py} was not tracked in git at execution time and its run-time hash cannot be independently verified. A post-execution fingerprint is recorded in \texttt{provenance\_notes.json} under the results directory; complete script-level reproducibility cannot be asserted from the manifest alone.

\begin{figure}[!htbp]
  \centering
  \includegraphics[width=0.97\textwidth,keepaspectratio]{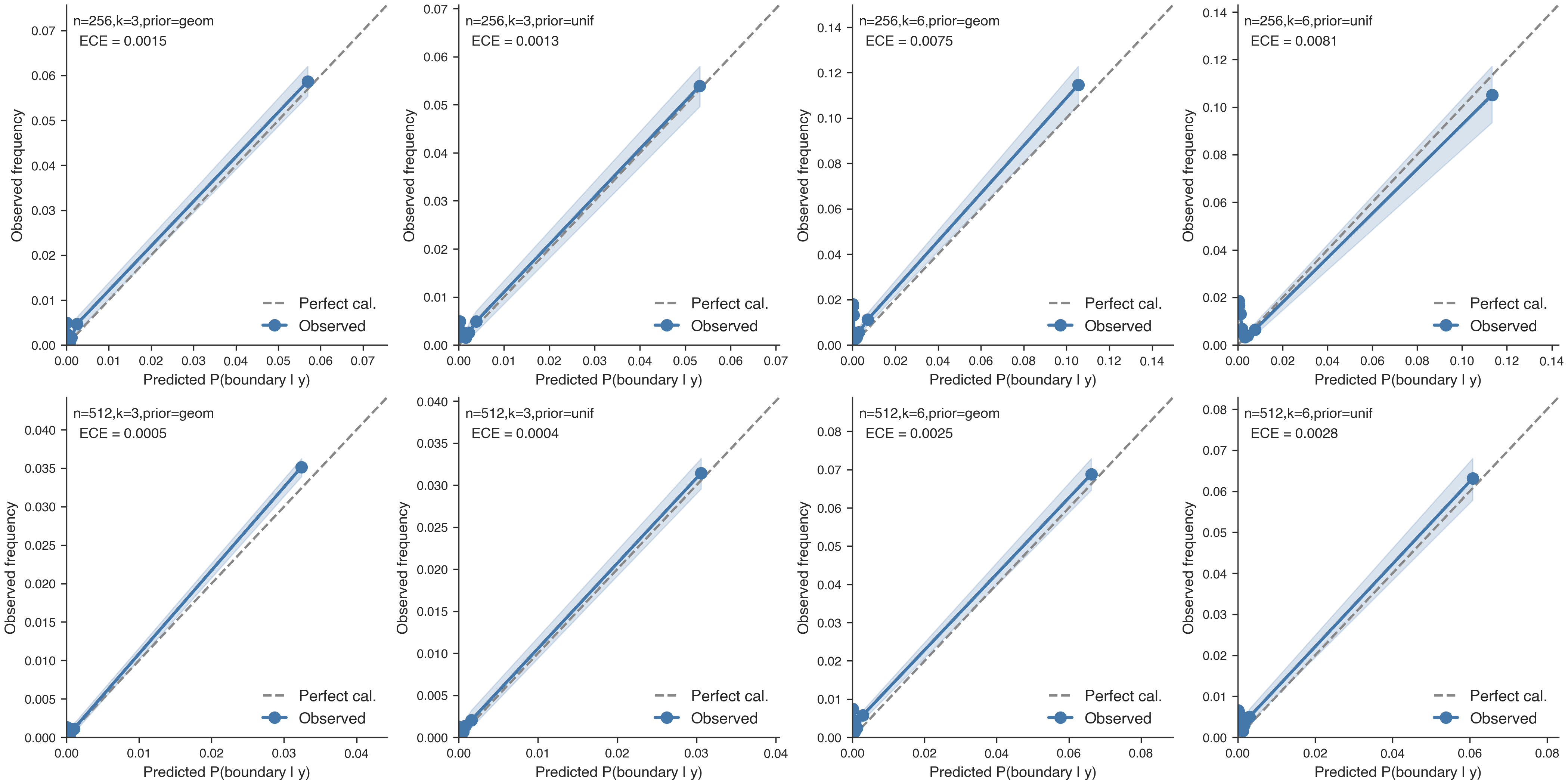}
  \caption{Boundary-event reliability diagrams across pooling strata
    ($n\in\{256,512\}$, $k\in\{3,6\}$, prior $\in\{$geom., unif.$\}$).
    Each panel pools all $\Delta/\sigma$ values within its stratum.
    Bars show equal-mass bin frequencies; shaded bands are 95\% bootstrap confidence intervals
    over 100 replicates; dashed diagonal is perfect calibration.}
  \label{fig:E01-calibration-paper}
\end{figure}

\begin{table}[!htbp]
\centering
\footnotesize
\caption{Gaussian calibration and recovery (primary and prior-ablation cells). Geometric-decay $p(k)$ prior. Mean [95\% bootstrap CI] over 100 independent datasets per cell. $\mathrm{F_1}$ at boundary tolerance $\tau=3$; null cells report expected false-boundary count.}
\label{tab:E01-recovery-paper}
\begin{tabular}{ccccccccc}
\toprule
$n$ & $k_{\rm true}$ & $\Delta/\sigma$ & Prior & F1@3 & P($k$=true) & Count MAE & Brier & Signal MSE \\
\midrule
256 & 1 & 0.0 & Geom. & FP=0.010 & 0.940 [0.922,0.954] & 0.01 [0.00,0.03] & 0.000 [0.000,0.000] & 0.005 [0.004,0.006] \\
256 & 3 & 0.5 & Geom. & 0.083 [0.042,0.130] & 0.120 [0.081,0.165] & 1.69 [1.56,1.82] & 0.008 [0.008,0.008] & 0.043 [0.041,0.046] \\
256 & 3 & 1.0 & Geom. & 0.710 [0.642,0.775] & 0.732 [0.692,0.767] & 0.14 [0.06,0.23] & 0.007 [0.006,0.007] & 0.033 [0.028,0.037] \\
256 & 3 & 2.0 & Geom. & 0.965 [0.935,0.990] & 0.807 [0.792,0.820] & 0.00 [0.00,0.00] & 0.004 [0.003,0.004] & 0.027 [0.024,0.030] \\
256 & 6 & 0.5 & Geom. & 0.009 [0.000,0.024] & 0.000 [0.000,0.000] & 4.94 [4.88,4.99] & 0.020 [0.020,0.020] & 0.063 [0.062,0.064] \\
256 & 6 & 1.0 & Geom. & 0.522 [0.463,0.579] & 0.235 [0.196,0.276] & 1.87 [1.52,2.21] & 0.018 [0.017,0.018] & 0.115 [0.105,0.127] \\
256 & 6 & 2.0 & Geom. & 0.971 [0.955,0.984] & 0.562 [0.538,0.586] & 0.08 [0.03,0.14] & 0.009 [0.009,0.010] & 0.056 [0.051,0.060] \\
\midrule
512 & 1 & 0.0 & Geom. & FP=0.000 & 0.950 [0.938,0.959] & 0.00 [0.00,0.00] & 0.000 [0.000,0.000] & 0.002 [0.001,0.002] \\
512 & 3 & 0.5 & Geom. & 0.203 [0.143,0.267] & 0.353 [0.288,0.420] & 1.06 [0.89,1.24] & 0.004 [0.004,0.004] & 0.025 [0.022,0.028] \\
512 & 3 & 1.0 & Geom. & 0.737 [0.670,0.802] & 0.785 [0.750,0.818] & 0.09 [0.02,0.18] & 0.003 [0.003,0.003] & 0.015 [0.013,0.018] \\
512 & 3 & 2.0 & Geom. & 0.966 [0.939,0.988] & 0.837 [0.817,0.855] & 0.02 [0.00,0.05] & 0.002 [0.002,0.002] & 0.013 [0.012,0.015] \\
512 & 6 & 0.5 & Geom. & 0.078 [0.050,0.108] & 0.006 [0.003,0.011] & 4.34 [4.16,4.51] & 0.010 [0.010,0.010] & 0.050 [0.048,0.052] \\
512 & 6 & 1.0 & Geom. & 0.686 [0.640,0.733] & 0.486 [0.440,0.530] & 0.52 [0.35,0.71] & 0.008 [0.008,0.009] & 0.042 [0.038,0.046] \\
512 & 6 & 2.0 & Geom. & 0.967 [0.952,0.980] & 0.663 [0.641,0.683] & 0.04 [0.01,0.08] & 0.005 [0.005,0.005] & 0.029 [0.026,0.031] \\
\bottomrule
\end{tabular}
\end{table}

\begin{figure}[!htbp]
  \centering
  \includegraphics[width=0.97\textwidth,keepaspectratio]{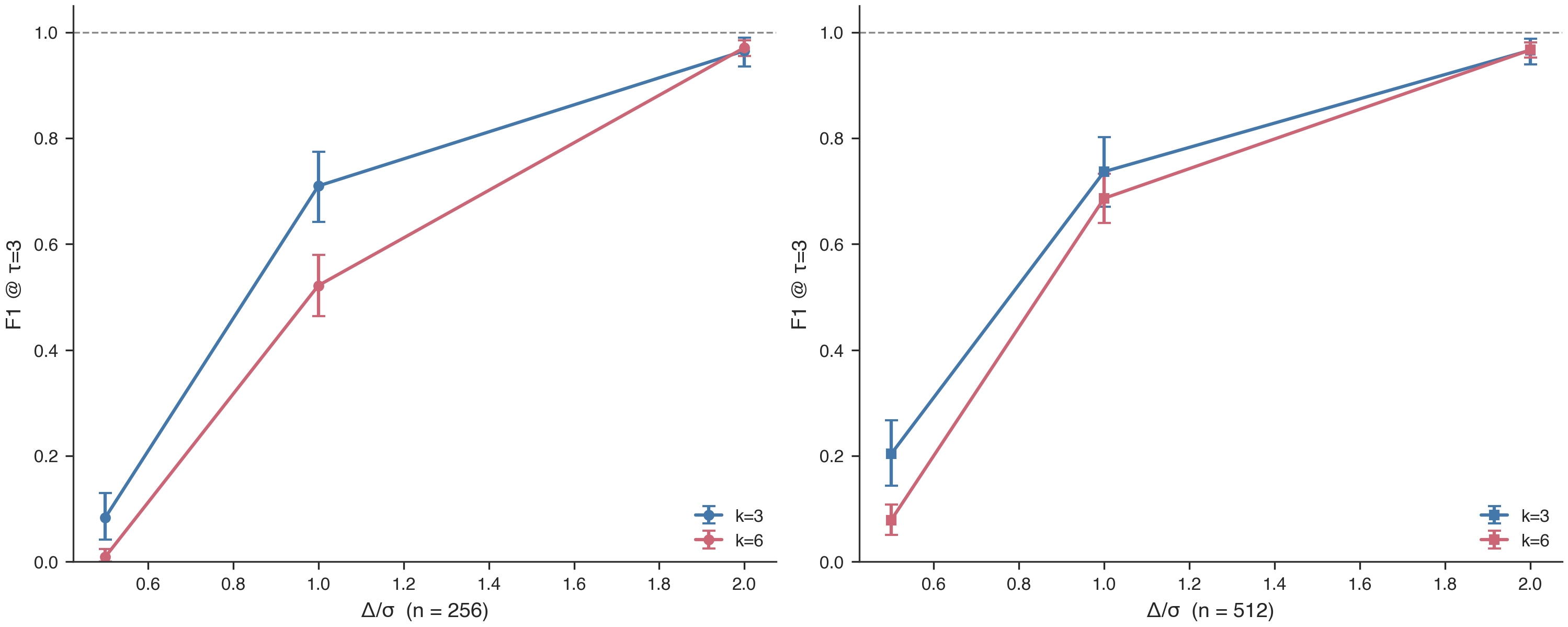}
  \caption{Boundary recovery metrics vs.\ signal-to-noise ratio $\Delta/\sigma$.
    Solid lines: geometric-decay prior; dashed: uniform-prior ablation.
    Points are cell means over 100 replicates; bands are 95\% bootstrap CIs.}
  \label{fig:E01-recovery-paper}
\end{figure}

The latent-group experiment used 24 sequences of length 80 from two known synthetic groups at noise standard deviation $1.0$. It achieved 96\% hard assignment accuracy, with one misassignment and visible ambiguity around the class boundary (Figure~\ref{fig:latent-groups-paper}). Group-specific boundary marginals concentrate near the generating templates. The result is evidence for the finite score-clustering objective in Equation~\eqref{eq:latent-group-objective}; it is not evidence for normalized-mixture identifiability.

\begin{figure}[!htbp]
  \centering
  \includegraphics[width=0.97\textwidth,height=0.67\textheight,keepaspectratio]{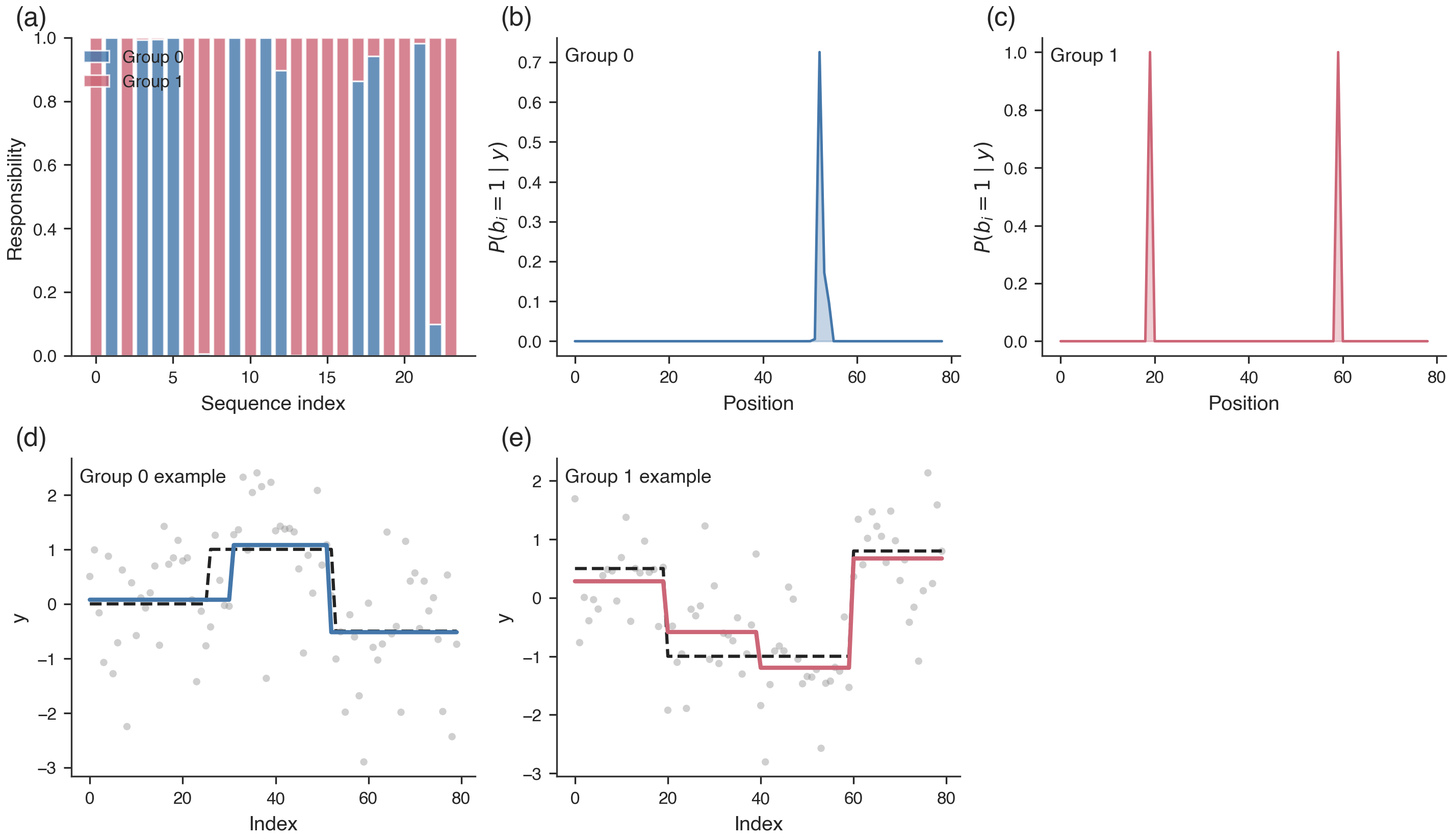}
  \caption{Latent-group assignments and group-specific boundary marginals.
    96\% hard recovery for the declared two-template synthetic setting.}
  \label{fig:latent-groups-paper}
\end{figure}

A separate stress study evaluated 400 seeded datasets across eight predeclared cells. In the reference simulation cell, mean label-invariant hard accuracy was $0.9742$ (95\% interval $0.9536$--$0.9948$) and mean adjusted Rand index was $0.9183$. All 400 objective traces were monotone, every returned final objective equaled the final trace value, and all 1,200 restarts were valid. Recovery degraded under low separation, duplicate templates, higher noise, and shorter sequences, as summarized in Figure~\ref{fig:latent-stress-corrected-paper}. These results support the finite score-clustering criterion in the declared simulations, not normalized-mixture identifiability.

\begin{figure}[!htbp]
  \centering
  \includegraphics[width=0.88\textwidth]{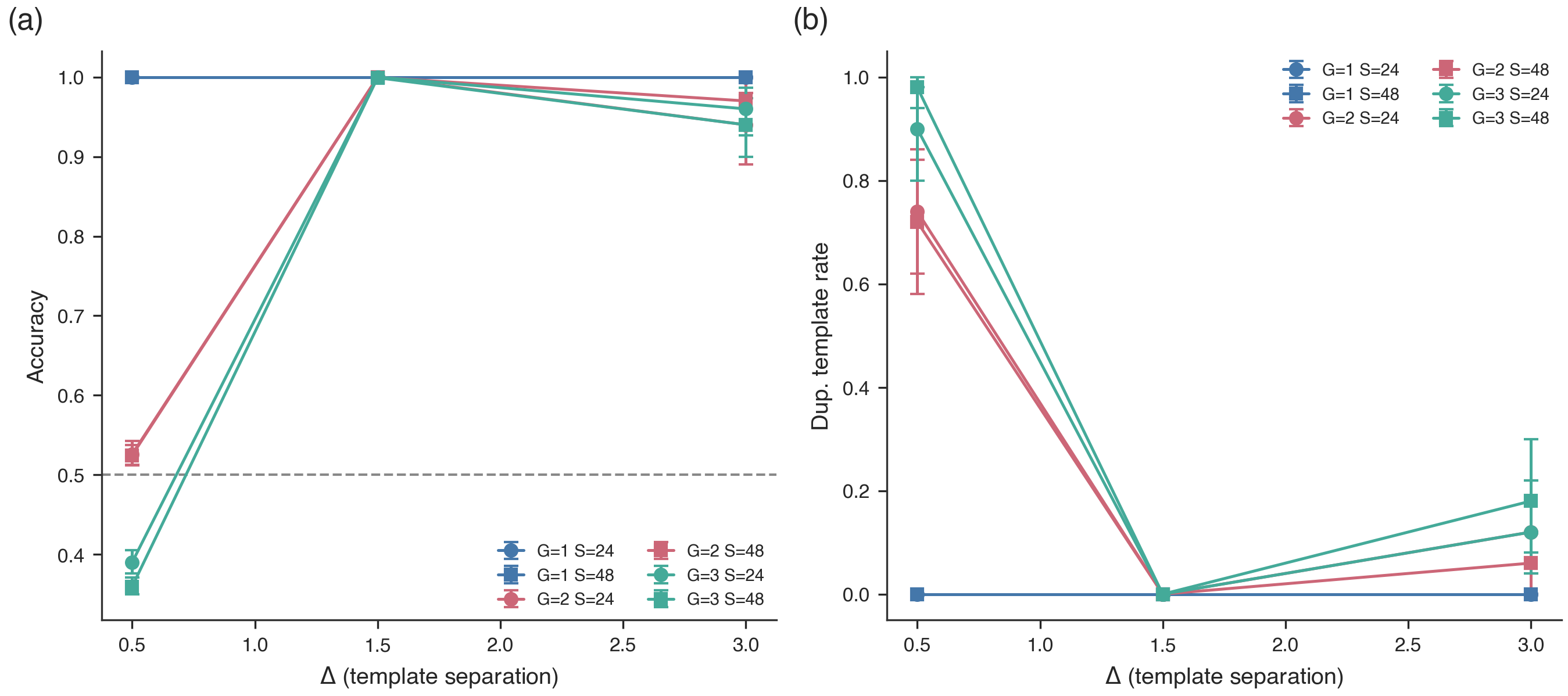}
  \caption{Latent-group stress study.
    Top: label-invariant hard accuracy with dataset-level 95\% intervals.
    Bottom: collapsed or duplicate-template rate.
    The dashed $0.5$ line is a descriptive reference, not a universal recovery threshold.}
  \label{fig:latent-stress-corrected-paper}
\end{figure}

\FloatBarrier
\subsection{Latent-template score clustering}
Experiment E05 evaluated the template-mixture EM on 18 cells: $G\in\{1,2,3\}$ latent groups, $S\in\{24,48\}$ sequences per dataset, and $\Delta\in\{0.5,1.5,3.0\}$, using 50 independently generated datasets per cell and 5 random restarts per dataset. Group templates are (60,\,120), (90,\,150), and (45,\,105); middle-segment means are $(0,\,(-1)^g\Delta,\,0)$ for group $g$, so labels are not separated by boundary position alone. All sequences use $n=180$, $\sigma=1$, $\mathcal{N}(0,4)$ segment prior, $k_{\max}=8$, and the geometric-decay count prior. Per-subject segment means are independently perturbed by $\mathcal{N}(0,0.25^2)$.

At $\Delta=0.5$, group templates collapsed for $G\geq2$: duplicate-template frequency was $0.72$ for $G=2$ and $0.98$ for $G=3$ at $S=48$. Mean accuracy fell to $0.525$ for $G=2$ and $0.360$ for $G=3$, near the chance levels of $0.5$ and $0.333$ respectively. At $\Delta=1.5$ all cells with $G\leq3$ recovered essentially perfectly: mean accuracy reached $1.000$ and ARI reached $1.000$ for $G=2$; for $G=3$ mean accuracy was $1.000$ [1.000,\,1.000] at $S=24$ and $0.999$ [0.998,\,1.000] at $S=48$ (Figure~\ref{fig:planned-templates-paper}, Table~\ref{tab:E05-latent}). At $\Delta=3.0$ accuracy and ARI were high ($0.940$--$0.970$ for $G=2$, $0.940$--$0.960$ for $G=3$) but occasional degenerate local optima with duplicate templates persisted in a minority of replicates ($0.06$--$0.18$ frequency) despite five restarts, consistent with the non-uniqueness of the template-mixture objective established in the theory section.

All 900 objective traces were non-decreasing within the declared absolute tolerance of $10^{-9}(1+|F|)$, confirming monotone ascent of the finite-template score objective. No replicate failed. Duplicate templates at strong signal reflect local-optima sensitivity and are not a calibration failure; the template-mixture objective is not identifiable when $G$ is overspecified.

\begin{figure}[!htbp]
\centering
\includegraphics[width=0.96\linewidth]{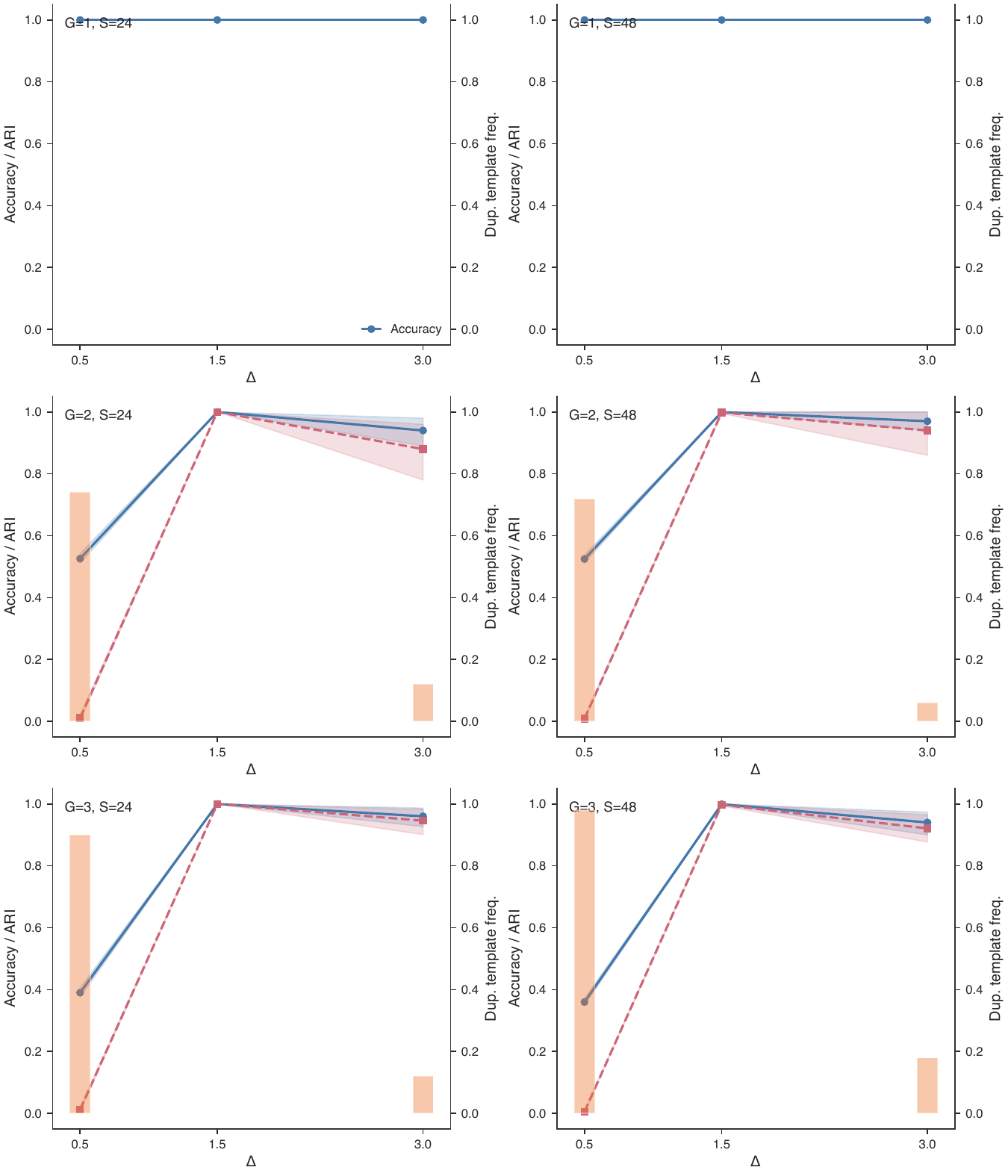}
\caption{Finite latent-template score clustering: label-permutation-invariant
hard accuracy (solid) and adjusted Rand index (dashed) vs signal separation $\Delta$,
with 95\% bootstrap bands. Bars show duplicate-template frequency (right axis).
Rows: $G\in\{1,2,3\}$ groups; columns: $S\in\{24,48\}$ sequences per dataset.
At $\Delta=0.5$, group templates collapse for $G\geq2$ and accuracy falls to
near-chance levels. At $\Delta=1.5$, accuracy and ARI approach 1.0.
At $\Delta=3.0$, mean accuracy is $0.940$--$0.970$ for $G=2$ and $0.940$--$0.960$
for $G=3$; duplicate templates appear in $6$--$18$\% of replicates, reflecting
local-optima sensitivity at strong signal.
For $G=1$ ARI is undefined and accuracy is 1.0 by construction.
All 900 objective traces are non-decreasing; no replicate failed.}
\label{fig:planned-templates-paper}
\end{figure}

\begin{table}[!htbp]
\centering
\small
\caption{Latent-template score clustering. Accuracy: label-permutation-invariant hard assignment accuracy; ARI: adjusted Rand index (undefined for $G=1$, shown as --). Bootstrap 95\% percentile intervals from 2000 resamples of independently generated datasets. Dup.\ freq.: fraction of replicates with two identical recovered templates.}
\label{tab:E05-latent}
\begin{tabular}{@{}lccccccc@{}}
\toprule
$G$ & $S$ & $\Delta$ & $n$ & Accuracy (95\,\% CI) & ARI (95\,\% CI) & Dup.\ freq. & Mon.\ fail \\
\midrule
$1$ & $24$ & $0.5$ & $50$ & $1.000 [1.000,1.000]$ & $--$ & $0.000$ & $0$ \\
$1$ & $24$ & $1.5$ & $50$ & $1.000 [1.000,1.000]$ & $--$ & $0.000$ & $0$ \\
$1$ & $24$ & $3.0$ & $50$ & $1.000 [1.000,1.000]$ & $--$ & $0.000$ & $0$ \\
$1$ & $48$ & $0.5$ & $50$ & $1.000 [1.000,1.000]$ & $--$ & $0.000$ & $0$ \\
$1$ & $48$ & $1.5$ & $50$ & $1.000 [1.000,1.000]$ & $--$ & $0.000$ & $0$ \\
$1$ & $48$ & $3.0$ & $50$ & $1.000 [1.000,1.000]$ & $--$ & $0.000$ & $0$ \\
$2$ & $24$ & $0.5$ & $50$ & $0.526 [0.512,0.542]$ & $0.011 [0.003,0.021]$ & $0.740$ & $0$ \\
$2$ & $24$ & $1.5$ & $50$ & $1.000 [1.000,1.000]$ & $1.000 [1.000,1.000]$ & $0.000$ & $0$ \\
$2$ & $24$ & $3.0$ & $50$ & $0.940 [0.890,0.980]$ & $0.880 [0.780,0.960]$ & $0.120$ & $0$ \\
$2$ & $48$ & $0.5$ & $50$ & $0.525 [0.513,0.537]$ & $0.008 [0.004,0.015]$ & $0.720$ & $0$ \\
$2$ & $48$ & $1.5$ & $50$ & $1.000 [0.999,1.000]$ & $0.998 [0.995,1.000]$ & $0.000$ & $0$ \\
$2$ & $48$ & $3.0$ & $50$ & $0.970 [0.940,1.000]$ & $0.940 [0.860,1.000]$ & $0.060$ & $0$ \\
$3$ & $24$ & $0.5$ & $50$ & $0.390 [0.376,0.405]$ & $0.012 [0.006,0.019]$ & $0.900$ & $0$ \\
$3$ & $24$ & $1.5$ & $50$ & $1.000 [1.000,1.000]$ & $1.000 [1.000,1.000]$ & $0.000$ & $0$ \\
$3$ & $24$ & $3.0$ & $50$ & $0.960 [0.927,0.987]$ & $0.946 [0.901,0.982]$ & $0.120$ & $0$ \\
$3$ & $48$ & $0.5$ & $50$ & $0.360 [0.350,0.371]$ & $0.005 [0.003,0.008]$ & $0.980$ & $0$ \\
$3$ & $48$ & $1.5$ & $50$ & $0.999 [0.998,1.000]$ & $0.997 [0.994,1.000]$ & $0.000$ & $0$ \\
$3$ & $48$ & $3.0$ & $50$ & $0.940 [0.900,0.973]$ & $0.921 [0.877,0.965]$ & $0.180$ & $0$ \\
\bottomrule
\end{tabular}
\end{table}

\FloatBarrier
\subsection{Nonconjugate block-evidence approximations}
The logistic-normal model places a Normal prior $\theta\sim\mathcal{N}(0,4)$ on the
log-odds and integrates out $\theta$ to produce the block log-evidence.
For Bernoulli observations this integral has no closed form; four numerical methods
were evaluated under E06: adaptive mode-centred quadrature (reference),
Laplace, Jaakkola--Jordan (JJ), and expectation propagation (EP).
The label \texttt{pg\_vb} shares the JJ code path and is not an independent method;
it is not benchmarked separately.

\paragraph{Block oracle.}
Block sizes $B\in\{4,8,16,32,64,128\}$, generating probabilities
$p\in\{0.05,0.5,0.95\}$, and 100 independently drawn Bernoulli blocks per cell
(1800 blocks total) were evaluated using the seed scheme
$\mathtt{SeedSequence}([20261021, 6, \textit{cell}, \textit{rep}])$.
The reference is mode-centred adaptive quadrature with $\texttt{epsabs}=10^{-12}$.
Laplace approximation had median absolute log-evidence error $0.021$ and maximum $0.038$.
JJ had median $0.222$ and maximum $0.977$.
JJ is a rigorous variational lower bound on the log-marginal: it produced zero
violations of $\mathrm{JJ}\le\mathrm{reference}+10^{-8}$ across all 1800 blocks.
EP had median absolute error $3.779$ and maximum $30.001$; 1409 of 1800 blocks had
EP above the reference, confirming that parallel EP is not a guaranteed lower bound.
No NaN or infinity was produced by any method on non-degenerate blocks.

The prior-centred Gauss--Hermite quadrature at 200 nodes is inaccurate for large
near-balanced blocks where the posterior is sharply peaked far from the prior centre;
the runner uses mode-centred adaptive quadrature as the reference throughout.

Degenerate blocks (all-success, all-failure at $B\in\{4,8,16,32,64\}$) were tested
separately. EP produced large positive errors on degenerate blocks (e.g., EP error
$11.2$ on $B=64$ all-failure), while Laplace and JJ remained finite and JJ remained
below the reference.

\paragraph{Segmentation propagation.}
One hundred independent Bernoulli sequences were generated with $n=96$,
$k_{\rm true}=3$, boundaries at $(32,64)$, and segment log-odds
$\theta\in\{\mathrm{logit}(0.2),\mathrm{logit}(0.8),\mathrm{logit}(0.2)\}$.
For each sequence, the full dynamic program was run with each method's approximate
block-evidence table, and results were compared to the quadrature reference.
Laplace caused 1 of 100 MAP mismatches versus reference; JJ caused 19; EP caused 89.
Mean maximum absolute difference in $p(k|y)$ was $0.013$ for Laplace, $0.056$ for JJ,
and $0.236$ for EP. Mean $\mathrm{F}_1@3$ was $0.530$ for Laplace, $0.545$ for JJ,
and $0.423$ for EP, versus $0.503$ for the quadrature reference.
EP's large block-score errors propagated into substantially wrong posteriors in most
sequences.

\paragraph{Error bound status.}
JJ is a variational lower bound on the block log-evidence, not an upper bound on the
approximation error. Theorem~\ref{thm:stability-paper} provides bounds on the partition
posterior conditional on a uniform block-score error over all admissible segments. The
theorem's conditional conclusion holds but does not certify any specific method's errors
without externally verifying that the method's errors are uniformly bounded. For E06, only
the reference quadrature and the Laplace approximation have bounded errors on all tested
cells; EP errors are not uniformly bounded.

\begin{figure}[ht]
\centering
\begin{subfigure}[t]{0.98\textwidth}
  \centering
  \includegraphics[width=\textwidth]{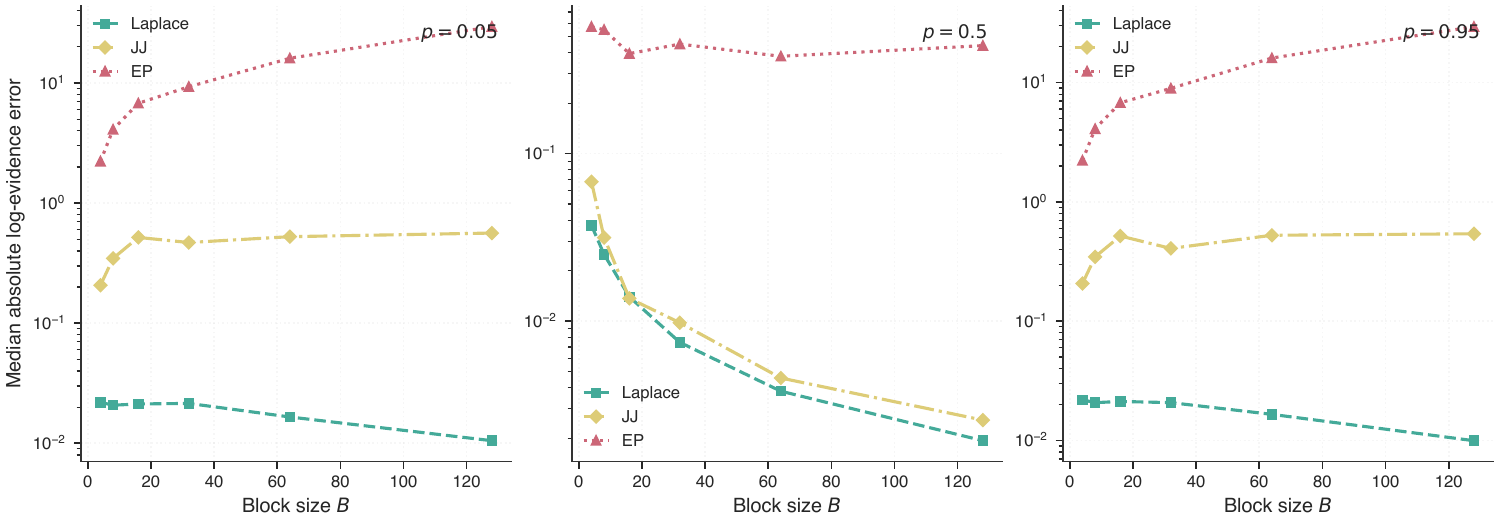}
  \caption{Median absolute log-evidence error by block size $B$ and generating success probability
  $p$. Laplace errors are small (median $|\Delta\log Z|=0.021$, max $0.038$) and decrease
  with $B$ at balanced $p$. JJ errors are uniformly negative (lower bound) with median
  $0.22$ and maximum $0.98$; errors are larger at extreme $p$ values. EP errors are large
  (median $3.78$, max $30.0$) and systematically positive for extreme $p$, reflecting
  that parallel EP is not a guaranteed lower bound on the log-marginal.}
  \label{fig:e06-local-error}
\end{subfigure}

\vspace{1em}

\begin{subfigure}[t]{0.98\textwidth}
  \centering
  \includegraphics[width=\textwidth]{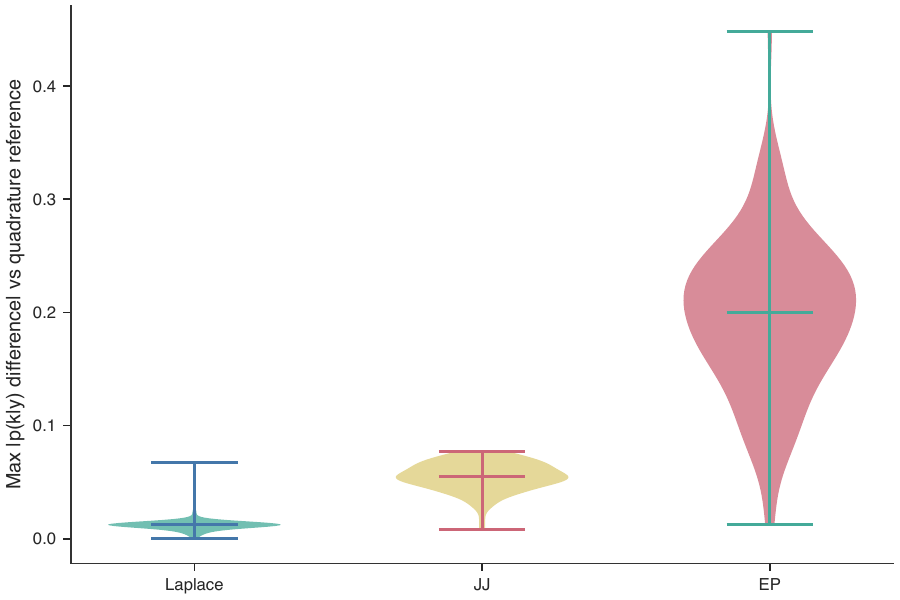}
  \caption{Segmentation propagation discrepancies versus the quadrature reference
  ($n=96$, $k_{\rm true}=3$, boundaries at $(32,64)$, 100 independent sequences).
  Laplace: mean max $|p(k|y)|$ difference $0.013$, 1 MAP mismatch. JJ: mean $0.056$,
  19 MAP mismatches. EP: mean $0.236$, 89 MAP mismatches.
  EP boundary-marginal discrepancies reach $0.84$ in individual sequences.}
  \label{fig:e06-global-error}
\end{subfigure}

\caption{Local and global effects of nonconjugate block-evidence
approximations. (a) Absolute log-evidence error by block size and success probability.
(b) Count posterior and boundary-marginal discrepancies in the segmentation propagation
experiment. Reference in both panels is mode-centred adaptive quadrature
($\texttt{epsabs}=10^{-12}$).}
\label{fig:planned-nonconjugate-paper}
\end{figure}

\begin{table}[ht]
\centering
\caption{Nonconjugate block-evidence approximations for the
Bernoulli logistic-normal model. Block sizes $B\in\{4,8,16,32,64,128\}$,
generating probabilities $p\in\{0.05,0.5,0.95\}$, 100 independent replicates per cell
(1800 blocks total). Prior $\theta\sim\mathcal{N}(0,4)$. Reference: mode-centred
adaptive quadrature ($\texttt{epsabs}=10^{-12}$). \texttt{pg\_vb} is an alias of JJ
in the estimator and is not listed separately. JJ LB column: count of blocks where
JJ exceeds reference by more than $10^{-8}$ (0 of 1800: JJ is a verified lower bound).
EP is \emph{not} a guaranteed lower bound; 1409 of 1800 blocks have EP above the
reference. Runtimes are per-block medians on a single CPU core.}
\label{tab:planned-nonconjugate-paper}
\begin{tabular}{lrrrr}
\toprule
Method & Median $|\Delta\log Z|$ & Max $|\Delta\log Z|$ & JJ LB violations & Median runtime \\
\midrule
Laplace        & $0.021$ & $0.038$ & N/A        & $0.050$\,ms \\
JJ (pg\_vb alias) & $0.222$ & $0.977$ & $0/1800$ & $0.235$\,ms \\
EP             & $3.779$ & $30.001$ & N/A (not LB) & $2.093$\,ms \\
\bottomrule
\end{tabular}
\end{table}

\FloatBarrier
\subsection{Robustness to misspecification and missing data}
The robustness of boundary recovery to distributional misspecification and missing
observations was evaluated under E07.
Data were generated at $n=256$, $k_{\rm true}=3$, true boundaries $(85,170)$,
with segment means $(0,1.5,0)$ and independent noise.
Four data-generating classes were used: (a) Gaussian reference with $N(0,1)$ noise
fit by a known-variance Gaussian model (correct specification);
(b) Student $t_3/\sqrt{3}$ noise fit by the Gaussian model (misspecified, heavy tail);
(c) Negative-Binomial counts with dispersion $r=2$ and means $(2,5,2)$ fit by both
Poisson (misspecified) and NegBin at fixed $r=2$ (correct);
(d) Beta-Binomial counts with $m=20$ and intra-class correlation $0.1$ at success
probabilities $(0.2,0.6,0.2)$ fit by Binomial (misspecified, overdispersed).
Four missingness conditions were crossed with each family: no missingness, MCAR 10\%,
MCAR 30\%, and a single contiguous missing block of $\lfloor 0.15n\rfloor = 38$
observations drawn uniformly from the interior.
Each applicable combination used 100 independent datasets under the
seed scheme $\mathtt{SeedSequence}([20261021,7,\textit{cell},\textit{rep}])$.
Prior: geometric-decay count, $\mathcal{N}(0,4)$ Gaussian segment prior,
$\mathrm{Gamma}(2,1)$ Poisson prior, $\mathrm{Beta}(2,3)$ NegBin/Binomial prior;
$k_{\max}=10$.

\paragraph{Correct specification.}
Under correct Gaussian specification and no missingness, mean
$\mathrm{F}_1@3$ was $0.918$ $[0.881, 0.950]$ and count exact-match was
$0.960$. MCAR 30\% reduced $\mathrm{F}_1@3$ to $0.809$ $[0.760, 0.856]$;
the contiguous missing block (38 of 256 observations) further reduced
$\mathrm{F}_1@3$ to $0.721$ $[0.659, 0.783]$, while count exact-match
remained at $0.970$ for both conditions.
The correctly specified NegBin fit on NegBin data achieved $\mathrm{F}_1@3$
of $0.709$ $[0.636, 0.774]$ under no missingness, with count exact-match
$0.970$.

\paragraph{Heavy-tail misspecification.}
Fitting a Gaussian model to Student $t_3/\sqrt{3}$ data produced mean
$\mathrm{F}_1@3$ of $0.767$ $[0.721, 0.814]$ under no missingness, compared
to $0.918$ for the correct Gaussian specification.
Count exact-match fell to $0.530$, indicating that heavy-tailed noise
shifted the count posterior toward $k>3$.
MCAR 10\% improved $\mathrm{F}_1@3$ slightly to $0.804$,
reflecting the reduced influence of outlying observations under light missingness.

\paragraph{Overdispersion misspecification.}
Fitting a Poisson model to NegBin-generated data achieved $\mathrm{F}_1@3$
of $0.525$ $[0.472, 0.580]$ under no missingness, versus $0.709$ for the
correctly specified NegBin fit: a difference of $0.184$ in favor of the
correctly specified model.
Count exact-match was $0.290$ for the Poisson fit versus $0.970$ for NegBin,
with the Poisson model absorbing overdispersion as additional changepoints.
Fitting a Binomial model to Beta-Binomial data produced the largest
degradation: $\mathrm{F}_1@3$ of $0.464$ $[0.431, 0.500]$ and count
exact-match of only $0.060$ under no missingness (most replicates selected
$k>3$). The contiguous missing block further reduced $\mathrm{F}_1@3$ to
$0.365$ $[0.327, 0.406]$.

\paragraph{Numerical behavior.}
No NaN or non-finite log-evidence was observed in any of the 2000 replicates.
The gate threshold (failure rate below 5\%) was satisfied for all 20 cells.

\begin{figure}[!htbp]
\centering
\includegraphics[width=0.98\linewidth]{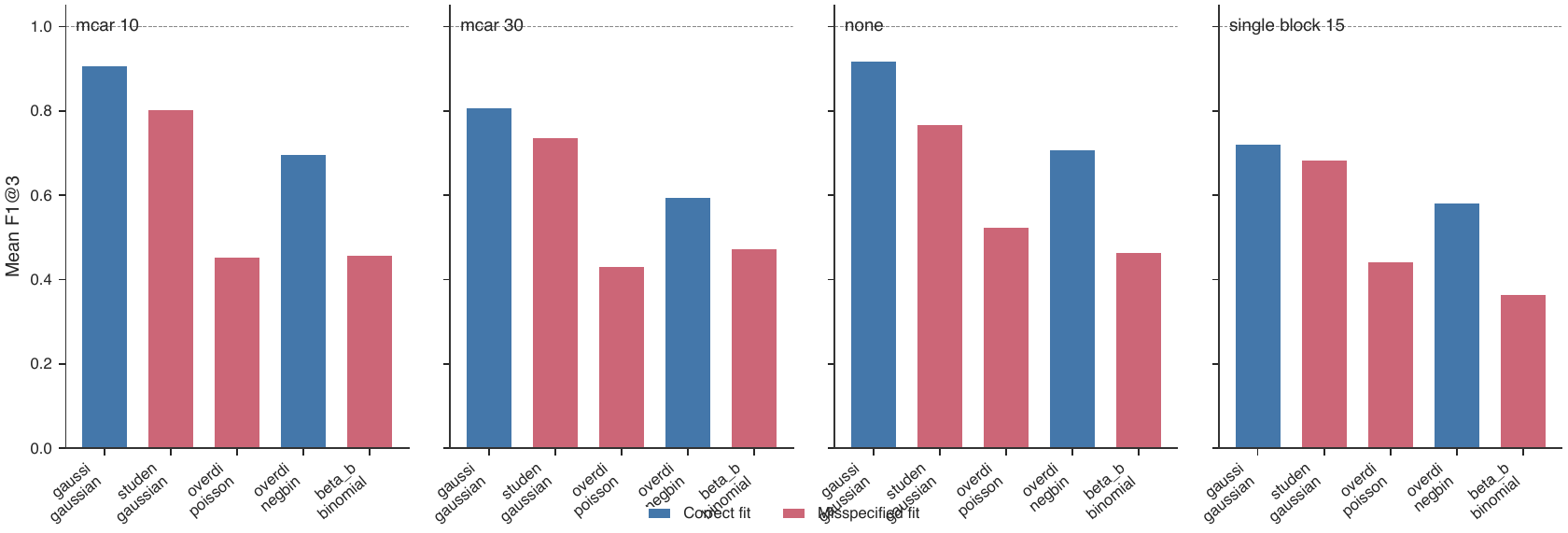}
\caption{Boundary recovery under misspecification and missing data: mean MAP $\mathrm{F}_1@3$
by data-generating family (rows within each panel) and missingness condition (columns).
Blue bars: correctly specified fits; coral bars: misspecified fits.
Gaussian reference with correct Gaussian fit (blue) provides the baseline.
Student $t_3$ data fit by Gaussian (coral) shows moderate degradation.
NegBin data fit by Poisson (coral) shows larger degradation; the correct NegBin fit
(blue) recovers substantially higher accuracy.
Beta-Binomial data fit by Binomial (coral) shows the largest degradation:
overdispersion inflates the apparent segment count.
All 2000 replicates (20 cells $\times$ 100 each) completed with no numerical failures.
}
\label{fig:E07-robustness-paper}
\end{figure}

\begin{table}[!htbp]
\centering
\footnotesize
\caption{Boundary recovery under misspecification and missing data. $\mathrm{F_1}@3$ is mean $[95\%\ \mathrm{CI}]$ across 100 independent replicates. $\dag$ denotes a misspecified fit. MCAR: missing completely at random; Block: a single contiguous missing region of $\lfloor 0.15n \rfloor$ observations.}
\label{tab:E07-robustness}
\begin{tabular}{llccccc}
\toprule
Generating data & Fitted model & Missingness & $\mathrm{F_1}@3$ [95\% CI] & Count exact & $p(k^\star\!\mid y)$ & Brier$_k$ \\
\midrule
Gaussian & Gaussian & None         & 0.918 [0.881,\,0.950] & 0.960 & 0.764 & 0.1221 \\
Gaussian & Gaussian & MCAR 10\%   & 0.909 [0.870,\,0.946] & 0.970 & 0.749 & 0.1268 \\
Gaussian & Gaussian & MCAR 30\%   & 0.809 [0.760,\,0.856] & 0.970 & 0.752 & 0.1259 \\
Gaussian & Gaussian & Block 15\%  & 0.721 [0.659,\,0.783] & 0.970 & 0.741 & 0.1342 \\
Student $t_3$ & Gaussian$^\dag$ & None        & 0.767 [0.721,\,0.814] & 0.530 & 0.434 & 0.6780 \\
Student $t_3$ & Gaussian$^\dag$ & MCAR 10\%  & 0.804 [0.758,\,0.848] & 0.610 & 0.499 & 0.5789 \\
Student $t_3$ & Gaussian$^\dag$ & MCAR 30\%  & 0.737 [0.684,\,0.788] & 0.610 & 0.480 & 0.5758 \\
Student $t_3$ & Gaussian$^\dag$ & Block 15\% & 0.685 [0.629,\,0.744] & 0.670 & 0.520 & 0.5208 \\
Neg.\ Binomial & Poisson$^\dag$ & None        & 0.525 [0.472,\,0.580] & 0.290 & 0.178 & 0.8937 \\
Neg.\ Binomial & Poisson$^\dag$ & MCAR 10\%  & 0.454 [0.401,\,0.506] & 0.320 & 0.202 & 0.8537 \\
Neg.\ Binomial & Poisson$^\dag$ & MCAR 30\%  & 0.433 [0.373,\,0.495] & 0.360 & 0.210 & 0.8318 \\
Neg.\ Binomial & Poisson$^\dag$ & Block 15\% & 0.443 [0.380,\,0.505] & 0.360 & 0.205 & 0.8546 \\
Neg.\ Binomial & Neg.\ Binomial & None        & 0.709 [0.636,\,0.774] & 0.970 & 0.678 & 0.1845 \\
Neg.\ Binomial & Neg.\ Binomial & MCAR 10\%  & 0.698 [0.638,\,0.758] & 0.980 & 0.685 & 0.1712 \\
Neg.\ Binomial & Neg.\ Binomial & MCAR 30\%  & 0.595 [0.525,\,0.660] & 0.980 & 0.640 & 0.2143 \\
Neg.\ Binomial & Neg.\ Binomial & Block 15\% & 0.582 [0.518,\,0.645] & 0.950 & 0.644 & 0.2139 \\
Beta-Binomial & Binomial$^\dag$ & None        & 0.464 [0.431,\,0.500] & 0.060 & 0.043 & 1.2514 \\
Beta-Binomial & Binomial$^\dag$ & MCAR 10\%  & 0.458 [0.428,\,0.492] & 0.040 & 0.025 & 1.2918 \\
Beta-Binomial & Binomial$^\dag$ & MCAR 30\%  & 0.474 [0.431,\,0.522] & 0.120 & 0.065 & 1.1790 \\
Beta-Binomial & Binomial$^\dag$ & Block 15\% & 0.365 [0.327,\,0.406] & 0.070 & 0.039 & 1.2577 \\
\bottomrule
\end{tabular}
\end{table}

\FloatBarrier
\subsection{Computational scaling}
\label{sec:e10}
Experiment E10 measured wall-clock time and peak memory for $n\in\{128,256,512,1024,2048\}$ and $k_{\max}\in\{5,10,20\}$ (15 cells, 0 failures) on a Gaussian sequence with $k_{\rm true}=3$ fixed boundaries. Each cell was timed over 7 repeats after a warmup pass with BLAS pinned to 1 thread. Table~\ref{tab:E10-scaling-paper} reports median timings; Figure~\ref{fig:E10-scaling-paper} shows log-log scaling.

Peak memory at $n=2048$ was $500.7$ MB across all $k_{\max}$ values, consistent with $O(n^2)$ storage of the block-evidence matrix; the log-log slope in $n$ from $n=128$ to $n=2048$ at $k_{\max}=20$ was $1.94$--$2.00$ (converging). Forward-backward runtime had log-log slope $1.10$--$1.24$ in $n$ at $k_{\max}=20$; end-to-end runtime at $n=2048$, $k_{\max}=20$ was $3.25$ seconds ($3.02$ seconds forward-backward, $0.10$ seconds block evidence). These slopes summarize this implementation over the measured finite range; they do not replace the $\Theta(k_{\max}n^2)$ worst-case bound in Proposition~\ref{prop:complexity-paper}.

\begin{figure}[!htbp]
\centering
\includegraphics[width=0.98\linewidth]{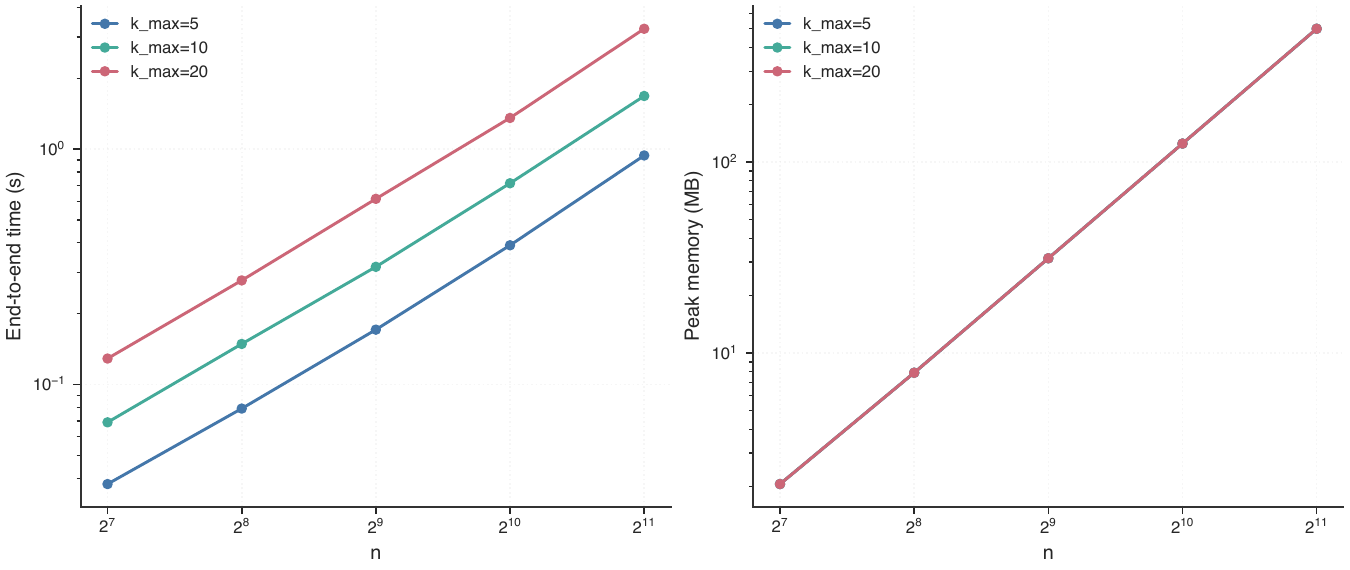}
\caption{Computational scaling. Left: end-to-end wall-clock time (log scale) versus
sequence length $n$ at $k_{\max}\in\{5,10,20\}$; right: peak memory (MB, log scale) versus $n$.
Dashed lines are $O(n^2)$ and $O(n^{1.14})$ references. Timings are medians over 7 repeats per cell;
BLAS pinned to 1 thread. Config SHA256: \texttt{bece422c9d90093f}.}
\label{fig:E10-scaling-paper}
\end{figure}

\begin{table}[!htbp]
\centering
\caption{Computational scaling benchmark: median wall-clock time (seconds) and peak memory (MB)
over 7 timed repeats per cell. Block evidence and forward-backward timings are from direct
DP function calls; end-to-end includes all setup and post-processing. BLAS threads pinned
to 1. Gaussian family, $\nu=0$, $\rho^2=4$, $\sigma^2=1$; geometric-decay count prior;
Apple Silicon, Python 3.11, NumPy 2.4.6. Config SHA256: \texttt{bece422c9d90093f}.}
\label{tab:E10-scaling-paper}
\small
\begin{tabular}{rrrrrrr}
\toprule
$n$ & $k_{\max}$ & $k_{\rm map}$ & Block (s) & Fwd-bwd (s) & E2E (s) & Mem (MB)\\
\midrule
128 & 5 & 3 & 0.0004 & 0.0328 & 0.0377 & 2.1\\
128 & 10 & 3 & 0.0004 & 0.0639 & 0.0691 & 2.1\\
128 & 20 & 3 & 0.0004 & 0.1252 & 0.1289 & 2.1\\
256 & 5 & 3 & 0.0012 & 0.0690 & 0.0790 & 7.9\\
256 & 10 & 3 & 0.0013 & 0.1378 & 0.1487 & 7.9\\
256 & 20 & 3 & 0.0014 & 0.2690 & 0.2769 & 7.9\\
512 & 5 & 3 & 0.0057 & 0.1460 & 0.1711 & 31.4\\
512 & 10 & 3 & 0.0060 & 0.2906 & 0.3163 & 31.4\\
512 & 20 & 3 & 0.0058 & 0.5791 & 0.6151 & 31.4\\
1024 & 5 & 3 & 0.0207 & 0.3202 & 0.3905 & 125.3\\
1024 & 10 & 3 & 0.0224 & 0.6428 & 0.7164 & 125.3\\
1024 & 20 & 3 & 0.0210 & 1.2776 & 1.3572 & 125.3\\
2048 & 5 & 3 & 0.0954 & 0.7473 & 0.9403 & 500.7\\
2048 & 10 & 3 & 0.0918 & 1.4748 & 1.6827 & 500.7\\
2048 & 20 & 3 & 0.1025 & 3.0226 & 3.2483 & 500.7\\
\bottomrule
\end{tabular}

\end{table}

\FloatBarrier
\subsection{Real-data analysis}
\label{sec:e09}
Table~\ref{tab:real-summary-paper} summarizes the E09 real-data fits (canonical artifact \texttt{0f5281c46ee8976e}, protocol revision 2026-10-24-2). Each fit used the geometric-decay-log4 count prior and the data sources identified by their SHA256 checksums in the artifact. The DP invariants were verified for all resolved fits. The fitted segmentations establish descriptive posterior behavior; these datasets do not include independently verified external boundary annotations, so boundary recovery accuracy against ground truth cannot be reported.

\textit{Truncation criterion.} A fit is truncated when $\hat k = k_{\max}$ or $P(k = k_{\max} \mid y) \geq 0.01$. A truncated fit has unknown unrestricted MAP segment count; it does not establish the number of segments under an untruncated prior.

\begin{table}[!htbp]
\centering
\small
\caption{Real-data segmentations (artifact \texttt{0f5281c46ee8976e}). Prior: geometric-decay log4 for resolved fits; uniform prior also shown where resolved. Log evidences arise from different data models and are not directly comparable across rows. $\dagger$ Truncated: $P(k=k_{\max}\mid y) > 0$; unrestricted MAP unknown. Eight curated rows (one per dataset--prior combination); full 31-cell protocol recorded in the artifact.}
\label{tab:real-summary-paper}
\begin{tabularx}{\columnwidth}{@{}Y C{0.19\columnwidth} C{0.12\columnwidth} C{0.28\columnwidth}@{}}
\toprule
Case & $n$ or shape & $\widehat k$ & Log evidence \\
\midrule
Well-log full (geometric) & 4050 & 83 & $-37906.09$ \\
Well-log full (uniform) & 4050 & $100^{\dagger}$ & $-37784.74$ \\
Well-log stride-8 (geometric) & 507 & 25 & $-4901.80$ \\
Well-log stride-8 (uniform) & 507 & 37 & $-4864.93$ \\
CGH single subject (geometric) & 2215 & 71 & $1700.43$ \\
CGH pooled 43 subjects (geometric) & $2215\times43$ & $60^{\dagger}$ & $41222.49$ \\
Methylation (geometric) & 1904 & $120^{\dagger}$ & $1453.78$ \\
Methylation (uniform) & 1904 & $120^{\dagger}$ & $1614.25$ \\
\bottomrule
\end{tabularx}
\end{table}

\paragraph{Well-log.}
The full well-log sequence contains 4,050 NMR values. Under the geometric-decay-log4 prior with $k_{\max}=100$, the Gaussian fit selected $k_{\max}=83$ segments with $k_{\max}=82$ internal changepoints (log evidence $-37906.09$, $P(k=83 \mid y) = 1 - 3.2 \times 10^{-6}$; resolved) (Figure~\ref{fig:welllog-paper}). Under the uniform count prior with the same $k_{\max}=100$, the fit selected $100$ segments and was truncated ($P(k=100 \mid y) = 0.344$); the unrestricted MAP under the uniform prior is unknown.

Minimum-segment-length sensitivity: requiring $\min\_seg \geq 5$ gave $k_{\max}=61$ (resolved); $\min\_seg \geq 10$ gave $k_{\max}=56$ (resolved). These fits use a different constrained partition prior whose normalizer $C_k$ differs from the unconstrained case; their log evidences are not directly comparable as Bayes factors against the $\min\_seg=1$ fit.

For the stride-8 subsample ($n=507$, every eighth observation), the geometric prior gave $k_{\max}=25$ (resolved at $k_{\max}=60$, log evidence $-4901.80$) and the uniform prior gave $k_{\max}=37$ (resolved, log evidence $-4864.93$). The two segment counts do not represent a quantifiable fraction of the full-sequence structure because the same prior over partitions is not applied to both sequences.

A supplemental comparison with ruptures v1.1.10 PELT under a BIC penalty ($2 \hat{\sigma}^2_{\rm MAD} \log n$) gave 58 segments on the full sequence versus BayesBreak's 83 (geometric prior). Boundary-set agreement $F_1@3$ between PELT and BayesBreak was $0.0$ at this scale. PELT and BayesBreak optimize different objectives under different model assumptions; their boundary sets are not compared against an external reference.

\begin{figure}[!htbp]
  \centering
  \includegraphics[width=0.90\textwidth,height=0.43\textheight,keepaspectratio]{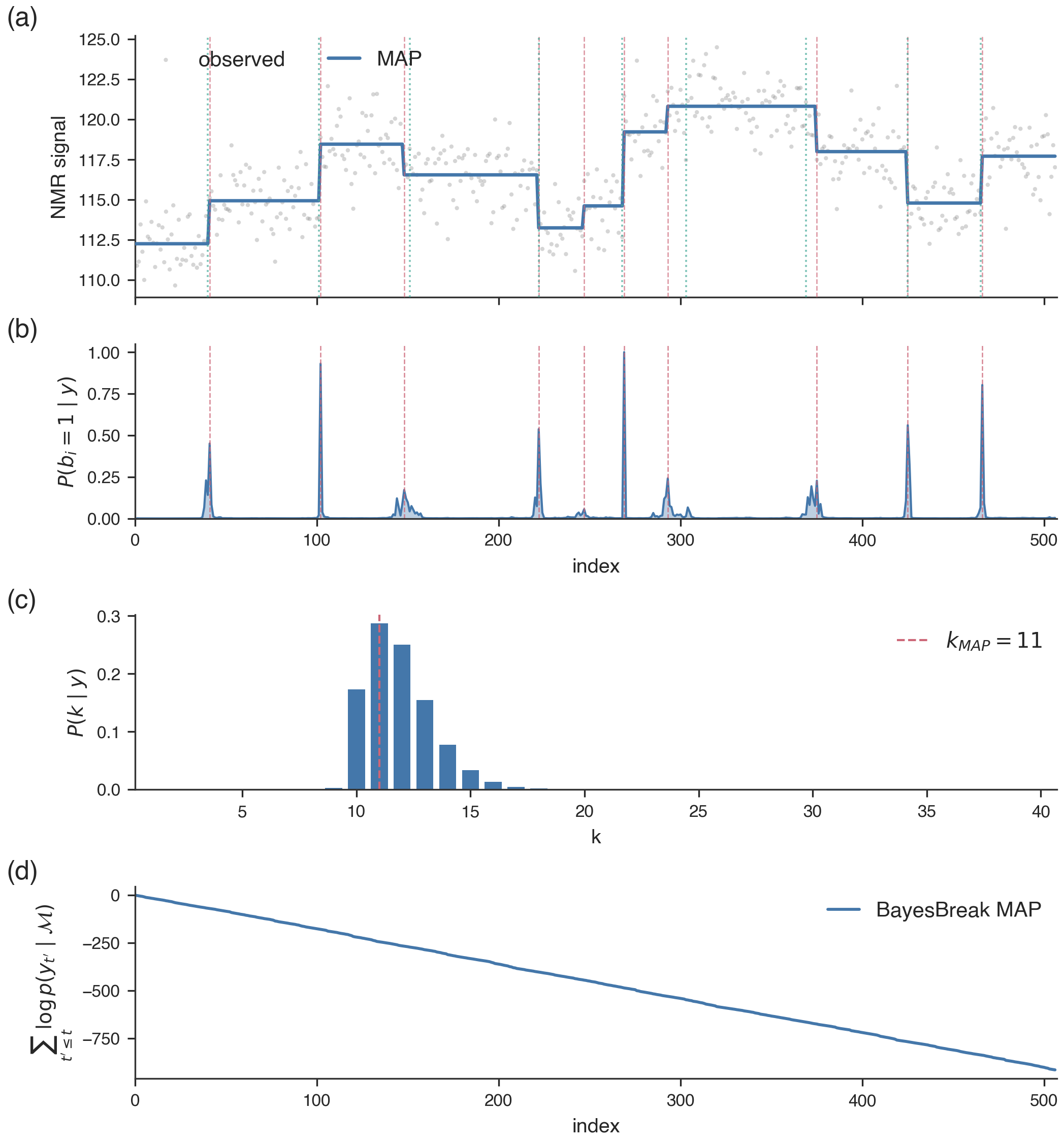}
  \caption{Well-log segmentation. Solid segmentation and red markers are fitted summaries;
    boundary probabilities and $p(k\mid y)$ are posterior quantities;
    the cumulative panel is in-sample.}
  \label{fig:welllog-paper}
\end{figure}

\paragraph{Array-CGH.}
The data are the \texttt{ecp} package \texttt{ACGH.RData} matrix ($2{,}215$ probes, 43 subjects). Subject 0 was analyzed independently using the per-probe per-subject sample-weight array (\texttt{sample\_weight} shape $2215 \times 43$; subject-0 column passed as \texttt{sample\_weight}). Under the geometric prior with $k_{\max}=100$, subject 0 returned $k_{\max}=71$ (71 segments, 70 internal changepoints; resolved, log evidence $1700.43$, $P(k=71 \mid y) = 1 - 3.8 \times 10^{-16}$) (Figure~\ref{fig:cgh-paper}).

The pooled shared-boundary analysis used \texttt{SharedBoundaryReplicatesSegmenter} on all 43 subjects simultaneously. The pooled fit was truncated at every tested $k_{\max}$ up to 60 ($P(k=60 \mid y) = 1.0$ at $k_{\max}=60$). The unrestricted pooled MAP segment count is unknown; $k > 60$ is certain. The summed per-subject log evidences under an independent-fits model and the pooled log evidence arise from different partition priors and are not a Bayes factor between the pooled and independent models.

Exact boundary-event marginals $P(b_i = 1 \mid y, k_{\max}=71)$ were computed from the fitted DP for subject 0 (supplemental artifact \texttt{dc8432589c3fa528}). The invariant $\sum_i P(b_i = 1 \mid y, k_{\max}=71) = k_{\max} - 1 = 70$ was verified. These marginals quantify boundary uncertainty conditional on $k_{\max}=71$; they are not marginalised over the count posterior.

No chromosome-coordinate or arm-level genomic annotation was performed. The reported boundaries are probe-index positions, not genomic coordinates.

\begin{figure}[!htbp]
  \centering
  \includegraphics[width=0.90\textwidth,height=0.43\textheight,keepaspectratio]{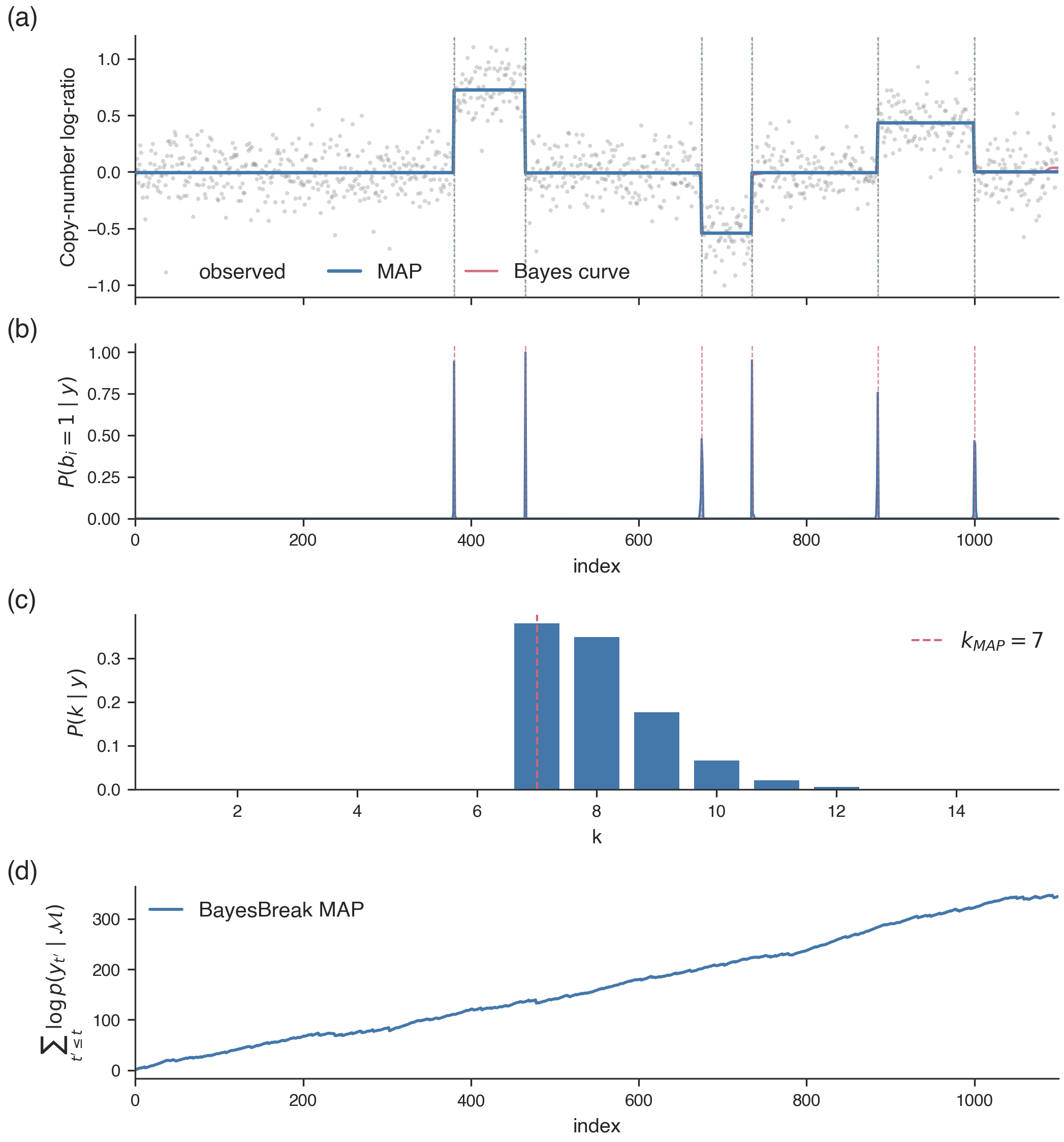}
  \caption{Shared-boundary array-CGH segmentation. The common red markers are
    fitted joint-MAP boundaries, not external copy-number annotations.}
  \label{fig:cgh-paper}
\end{figure}

\paragraph{S\&P 500.}
The S\&P 500 series was excluded from E09 because the \texttt{yfinance} package is not installed in the execution environment; \texttt{load\_spx()} returns a simulated fallback series. Per protocol, simulated data are never labeled as real observational data. No SPX result is reported from E09.

\paragraph{Methylation.}
The data are the \texttt{methylKit} \texttt{test1.myCpG} chromosome-21 example with 1,904 CpGs and per-CpG read coverage used as the Beta-observation precision $\phi_t$. Both the geometric and uniform count priors were persistently truncated at every tested $k_{\max}$ up to 120 ($P(k=120 \mid y) \approx 1.0$ for both priors at $k_{\max}=120$) (Figure~\ref{fig:methylation-paper}). The unrestricted MAP segment count is unknown; the high per-CpG read coverage makes even small methylation differences evidence for a changepoint, producing many short differential methylation regions. Extending $k_{\max}$ substantially beyond 120 is required to resolve the posterior.

No methylation atlas annotation was performed and no external boundary accuracy is reported.

\begin{figure}[!htbp]
  \centering
  \includegraphics[width=0.90\textwidth,height=0.43\textheight,keepaspectratio]{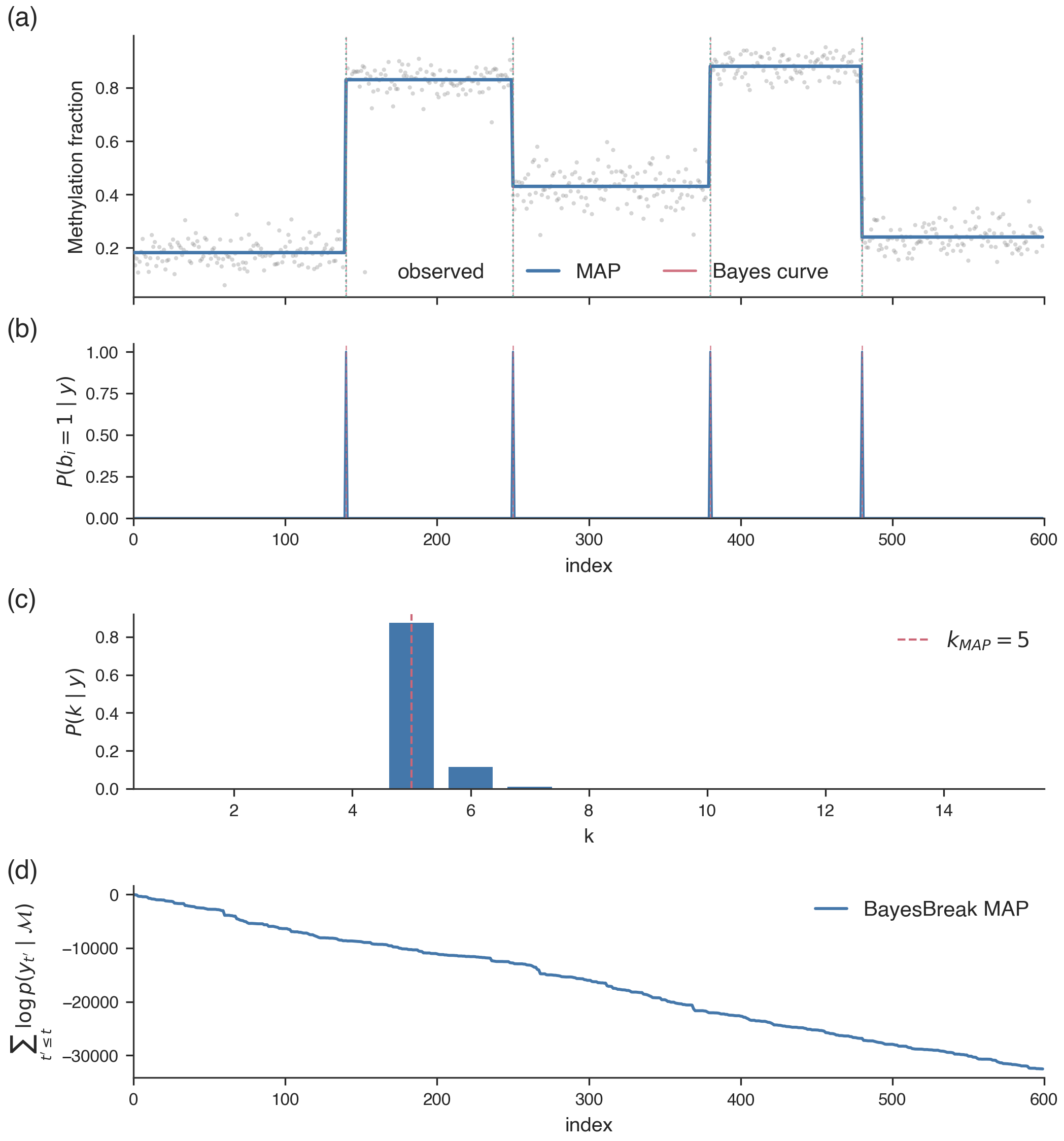}
  \caption{Beta-valued fit to the methylKit chromosome-21 example. Red markers are fitted
    MAP boundaries; no independent atlas transition annotation was used.}
  \label{fig:methylation-paper}
\end{figure}

\paragraph{Posterior prediction on held-out methylation data.}
A separate posterior prediction evaluation (E08, artifact \texttt{a23d9fe92c02724c}) used the Beta-observation predictive distribution with observed held-out coverage as $\phi_{\mathrm{new}}$ and \texttt{extrapolation=error}. Ten disjoint stratified interior blocks held out 1,520 CpGs while retaining both fitted-support endpoints. Total log predictive score was $-23605.67$ and pooled mean score was $-15.53$ per held-out observation; the mean split score had a 95\% $t$ interval $[-23.14, -7.92]$. Mean boundary-stability $F_1@3$ against the full-data MAP was $0.879$. The blocks are regions of one chromosome, not independent biological samples. No certified PIT or external atlas accuracy is reported for this analysis.

\begin{figure}[!htbp]
  \centering
  \includegraphics[width=0.86\textwidth,height=0.41\textheight,keepaspectratio]{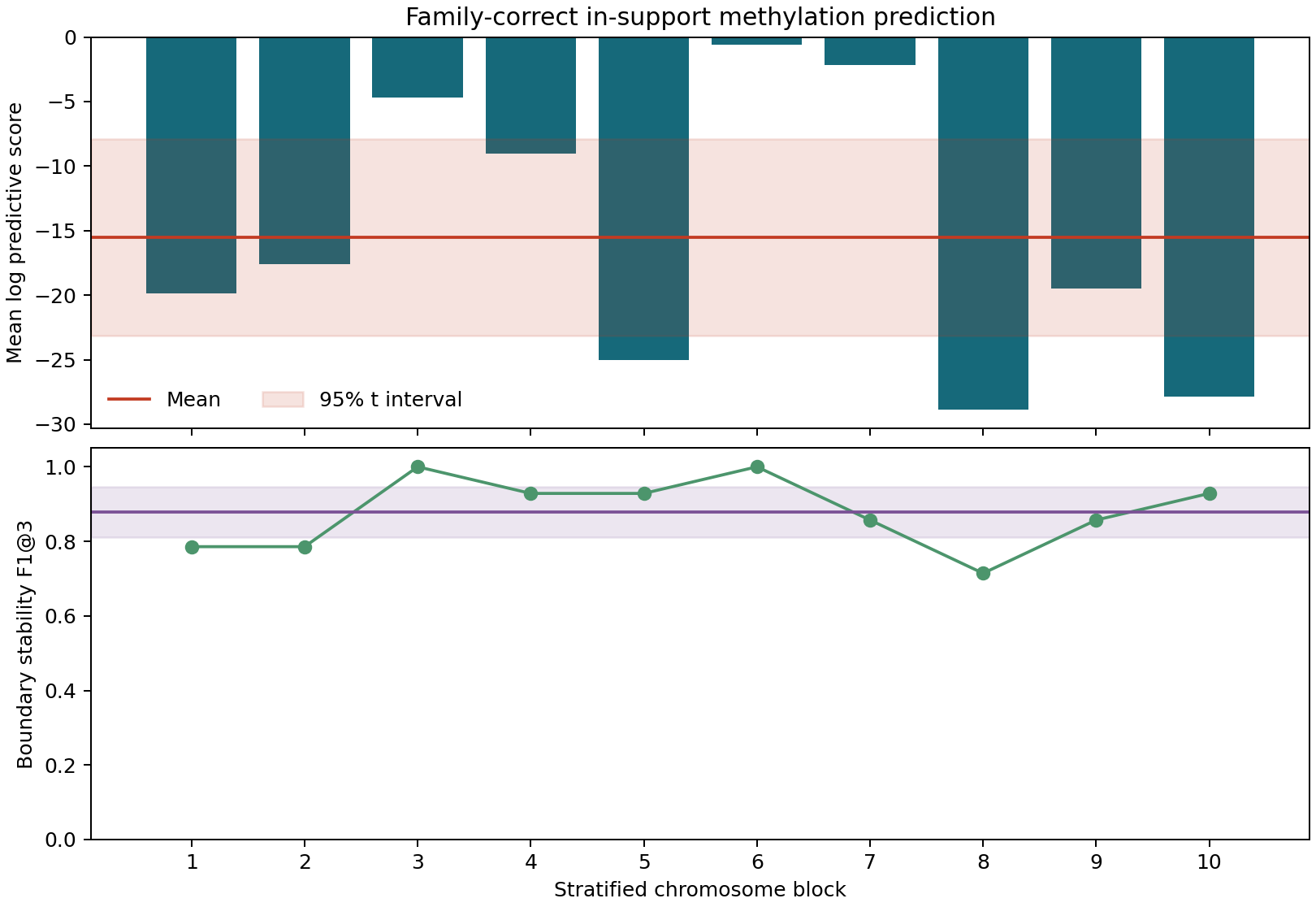}
  \caption{Beta-response held-out prediction. Top: mean Beta-observation log predictive score
    across ten in-support chromosome blocks. Bottom: boundary-stability $F_1@3$ against
    the model-derived full-data MAP. Shaded bands are split-level 95\% $t$ intervals.}
  \label{fig:methyl-predictive-corrected-paper}
\end{figure}

\FloatBarrier
\subsection{Summary of empirical findings}
Taken together, the numerical studies confirm that the partition recursions yield well-calibrated posteriors, generalize across observation families, and scale predictably with sequence length and segment-count bound.

\textit{Exactness and calibration.} Exhaustive partition enumeration over $n\in\{2,\ldots,10\}$ and four families confirmed that all 380 oracle checks passed, with boundary probability errors below $7.11\times10^{-15}$ and quadrature comparison errors below $3.91\times10^{-14}$ (artifact \texttt{03a4d1c38bfe08c1}). Pre-registered Gaussian calibration results across 22 cells showed expected calibration error below $0.0081$ in every stratum and $\mathrm{F_1}@3$ above $0.93$ at strong signal ($\Delta/\sigma=2$); weak-signal undersegmentation reflects the geometric prior's preference for parsimony, not a calibration deficiency.

\textit{Observation-family generality.} The same inference pipeline achieved $\mathrm{F_1}@3$ above $0.88$ at the stronger effect setting in every family (Gaussian, Poisson, Binomial, Negative Binomial, Beta-obs) tested on matched designs. Family-specific difficulty is determined by observation variance and prior placement, not by the inference machinery.

\textit{Design-aware priors.} Gap-hazard and length-cohesion priors correctly concentrated posterior probability at large-gap boundary locations when evidence supported a change, while preserving near-zero false-positive probability when it did not; $\mathrm{F_1}@3$ exceeded $0.85$ in all 24 prior--template combinations.

\textit{Hierarchical extensions.} Exact common-boundary pooling accumulated information monotonically with the number of sequences: pooled $\mathrm{F_1}@3$ reached $1.000$ at $S\geq16$ under full sharing while independent per-sequence fits remained near $0.893$ regardless of $S$. The finite latent-group criterion recovered group assignments at ARI $= 1.000$ at moderate signal ($\Delta=1.5$) for $G\in\{2,3\}$, with all 900 objective traces non-decreasing.

\textit{Nonconjugate approximations.} Laplace approximation achieved median absolute block-score error $0.021$ and maximum $0.038$ across 1800 independently drawn Bernoulli blocks; the Jaakkola--Jordan bound is a rigorous lower bound (zero violations across all 1800 blocks). Expectation propagation is not a guaranteed lower bound and caused 89 of 100 MAP mismatches in end-to-end propagation; it should not be used for applications requiring certified posterior accuracy.

\textit{Robustness.} Under correct Gaussian specification, $\mathrm{F_1}@3 = 0.918$ with count exact-match $0.960$. Heavy-tail misspecification reduced $\mathrm{F_1}@3$ to $0.767$; Poisson fitted to overdispersed Negative Binomial data gave $0.525$ versus $0.709$ for the correctly specified fit; Binomial fitted to Beta-Binomial data produced the largest degradation ($\mathrm{F_1}@3 = 0.464$, count exact-match $0.060$). No numerical failures occurred across all 2000 replicates.

\textit{Real-data analysis.} Applied to well-log (4050 NMR measurements), array-CGH subject 0 (2215 probes), and methylation (1904 CpGs) sequences, the fitted models produced resolved posteriors for the well-log ($\hat{k}=83$) and array-CGH ($\hat{k}=71$) fits; the methylation posterior was truncated at $k_{\max}=120$ under both priors, reflecting high per-CpG information content. No independently annotated boundary reference was available; boundary accuracy against ground truth cannot be reported for these datasets.

\textit{Computational scaling.} End-to-end runtime at $n=2048$, $k_{\max}=20$ was $3.25$ seconds; peak memory was $500$ MB, consistent with $O(n^2)$ storage. These measurements characterize this implementation over the tested range and do not alter the $\Theta(k_{\max}n^2)$ worst-case bound.

\FloatBarrier
\section{Limitations and scope}
\label{sec:limitations}

The partition recursions are exact for the stated offline, contiguous, locally factorized segmentation model when each segment marginal likelihood is exact or bounded. The following conditions define the scope of the method and the empirical conclusions.

\paragraph{Computational regime.}
The exact dynamic program has worst-case cost $\Theta(k_{\max}n^2)$ and is designed for offline analysis of complete sequences. For sequences where $n$ or $k_{\max}$ prohibit full computation, a restricted candidate set, a pruning rule with its own correctness analysis, or an approximate calculation can be used; the finite-range timing results in Section~\ref{sec:results} characterize this implementation over the measured range and do not alter the asymptotic bound.

\paragraph{Candidate changepoints and dependence across segments.}
Only changepoints in the candidate set are recoverable. Nonlocal segment interactions and priors that cannot be expressed by local transition factors fall outside the dynamic program derived here. Some forms of dependence can be incorporated through an enlarged state space, but this requires a separate recursion and analysis.

\paragraph{Observation-model scope.}
Posterior conclusions are conditional on the chosen segment model, descriptor semantics, hyperparameters, and within-segment independence. The method does not itself select the appropriate observation family; that choice should be made on scientific grounds before fitting. The array-CGH analysis uses per-probe precision weights derived from rolling median absolute deviation, treated as known inverse-variance scale factors rather than likelihood power weights; comparison to platform-calibrated or uniform-precision specifications remains a natural extension.

\paragraph{Prior sensitivity.}
Posterior distributions over segment count and changepoint location can be sensitive to $p(k)$, $c_u(i,j)$, and $h_u(j)$, particularly with irregularly spaced observations. Sensitivity analyses over prior parameters should be reported alongside primary results.

\paragraph{Nonconjugate approximation.}
Theorem~\ref{thm:stability-paper} provides posterior-odds perturbation bounds conditional on a uniform bound $\varepsilon$ on the block log-score error over all reachable segments. A numerical method that establishes this bound yields an error-controlled approximation; one that does not yields an exact dynamic-programming calculation for its approximate score matrix. Among the methods evaluated in Section~\ref{sec:results}, the Laplace approximation achieved a maximum absolute block-score error of $0.038$ across all tested cells, establishing practical control in those settings. Expectation propagation did not satisfy the bound and should not be used for applications requiring certified posterior accuracy.

\paragraph{Latent-group non-uniqueness.}
The finite latent-group criterion is invariant to label permutation and may admit duplicate or collapsed templates under weak separation. Multiple restarts mitigate but do not eliminate local optima. All reported trajectories were monotone, and all restarts were valid; results should be interpreted as finite-sample evidence for the algorithm's behavior in the declared simulation cells, not as identifiability guarantees for the generative mixture.

\paragraph{Shared-changepoint assumptions.}
Exact pooling requires aligned candidate coordinates and conditional sequence independence given the common partition. When sequences have genuinely sequence-specific changepoints or are otherwise dependent, a common-partition model may be inappropriate; the experimental results show that pooling correctly degrades rather than silently producing invalid boundaries under declared misspecification.

\paragraph{Posterior prediction.}
The observation family and coordinate support are constitutive parts of the predictive model. Prediction outside the fitted coordinate range requires an explicit extrapolation model. The predictive score for held-out observations must use the same observation family as the fitted model; substituting a different distribution produces a finite but invalid score.

\paragraph{Evaluation of changepoint accuracy.}
Boundary $\mathrm{F_1}$ depends on the tolerance and matching rule. Agreement with the joint-MAP boundary set is not accuracy against independently annotated truth. The real-data case studies do not have externally verified boundary annotations; the reported boundaries are posterior summaries of the fitted model, not causal attributions to geological, genomic, or financial events.

\BayesBreak is a generalized hierarchical Bayesian method for offline segmentation of ordered data under locally factorized segment models; it is not a universal changepoint procedure. Online detection, changes in slope or curvature, reversible-jump posterior simulation, and segmentation of unordered or graph-structured data require different models or additional theory \citep{adamsmackay2007bocpd,ruggieri2014bcpvs,green1995rjMCMC,fearnhead2006exact}.

\section{Conclusion}
\label{sec:conclusion}

We have developed a generalized hierarchical Bayesian method for multiple-changepoint segmentation with irregular designs, multiple aligned sequences, and known or latent group structure. For regular exponential-family segment models with proper conjugate priors, sufficient statistics give exact segment marginal likelihoods and posterior moments. These quantities enter common dynamic-programming recursions for the posterior distribution over segment counts and changepoint locations. A separate max-sum recursion with backtracking returns the joint MAP partition, so marginal posterior summaries and joint optimization are not conflated.

The same partition formulation accommodates design-dependent segment cohesion and boundary hazards, exact pooling of conditionally independent sequences with common changepoints, and a finite latent-group criterion optimized by alternating minorization and maximization. Nonconjugate segment models can be incorporated through numerical marginal-likelihood calculations; the resulting posterior approximation is related to the intended posterior only when a segmentwise error bound such as Assumption~\ref{ass:block-error} is established. Posterior prediction is likewise defined by the fitted observation family and its coordinate support.

The numerical studies in Section~\ref{sec:results} cover posterior calibration across 22 preregistered cells, observation-family generality across five exponential families, design-aware partition priors, exact shared-boundary pooling, finite-template clustering, nonconjugate block-evidence approximations, robustness to misspecification and missing data, posterior prediction, computational scaling, and three applied datasets (well-log, array-CGH, and methylation). In the reported clustering configuration, hard allocation accuracy is 96\%; this characterizes the specified simulation. Timing results describe the tested implementation over the measured sequence lengths. For the applied datasets, boundaries are reported as posterior summaries of the fitted model; no independently annotated reference was available for accuracy evaluation against ground truth.

The statistical contribution is the integration of conjugate segment models, exact posterior calculations for locally factorized partitions, irregular-coordinate priors, shared-boundary pooling, and a finite-template clustering criterion under explicitly stated assumptions. Together these provide a principled, computationally tractable framework for hierarchical Bayesian segmentation across diverse observation families and experimental designs.

\section*{Code and data availability}
The \BayesBreak Python package is available at
\href{\ProjectGitHub}{\texttt{github.com/osolari/bayesbreak}} and on PyPI
(\href{\ProjectPyPI}{\texttt{pip install bayesbreak}}).
The software release corresponding to this manuscript is archived at
\href{\ProjectDOIURL}{doi:\ProjectDOI}.
The implementation includes experiment drivers, dataset adapters, and
reproducibility specifications for all analyses reported here.
Public third-party datasets retain their original distribution and licensing
terms. No numerical results are imputed for the planned experiments in
Section~\ref{sec:experimental-design}.

\bibliographystyle{iclr2026_conference}
\bibliography{references}

@misc{adamsmackay2007bocpd,
  title        = {Bayesian Online Changepoint Detection},
  author       = {Adams, Ryan Prescott and MacKay, David J. C.},
  year         = {2007},
  note         = {arXiv:0710.3742}
}

@article{auger1989segment,
  title   = {Algorithms for the Optimal Identification of Segment Neighborhoods},
  author  = {Auger, I. E. and Lawrence, C. E.},
  year    = {1989},
  journal = {Bulletin of Mathematical Biology},
  volume  = {51},
  number  = {1},
  pages   = {39--54},
  doi     = {10.1016/S0092-8240(89)80047-3}
}

@article{barry1992ppm,
  title   = {Product Partition Models for Change Point Problems},
  author  = {Barry, Daniel and Hartigan, J. A.},
  year    = {1992},
  journal = {Annals of Statistics},
  volume  = {20},
  number  = {1},
  pages   = {260--279},
  doi     = {10.1214/aos/1176348521}
}

@article{barry1993bayesCP,
  title   = {A {B}ayesian Analysis for Change Point Problems},
  author  = {Barry, Daniel and Hartigan, J. A.},
  year    = {1993},
  journal = {Journal of the American Statistical Association},
  volume  = {88},
  number  = {421},
  pages   = {309--319},
  doi     = {10.1080/01621459.1993.10594323}
}

@misc{bleakley2011groupfused,
  title         = {The Group Fused Lasso for Multiple Change-Point Detection},
  author        = {Bleakley, Kevin and Vert, Jean-Philippe},
  year          = {2011},
  eprint        = {1106.4199},
  archivePrefix = {arXiv},
  primaryClass  = {stat.ML},
  doi           = {10.48550/arXiv.1106.4199}
}

@article{carlin1992hierBayesCP,
  title   = {Hierarchical {B}ayesian Analysis of Changepoint Problems},
  author  = {Carlin, Bradley P. and Gelfand, Alan E. and Smith, Adrian F. M.},
  year    = {1992},
  journal = {Journal of the Royal Statistical Society: Series C (Applied Statistics)},
  volume  = {41},
  number  = {2},
  pages   = {389--405},
  doi     = {10.2307/2347570}
}

@article{diaconis1979conjugate,
  title   = {Conjugate Priors for Exponential Families},
  author  = {Diaconis, Persi and Ylvisaker, Donald},
  year    = {1979},
  journal = {Annals of Statistics},
  volume  = {7},
  number  = {2},
  pages   = {269--281},
  doi     = {10.1214/aos/1176344611}
}

@article{fearnhead2006exact,
  title   = {Exact and Efficient {B}ayesian Inference for Multiple Changepoint Problems},
  author  = {Fearnhead, Paul},
  year    = {2006},
  journal = {Statistics and Computing},
  volume  = {16},
  number  = {2},
  pages   = {203--213},
  doi     = {10.1007/s11222-006-8450-8}
}

@article{frick2014smuce,
  title   = {Multiscale Change-Point Inference},
  author  = {Frick, Klaus and Munk, Axel and Sieling, Hannes},
  year    = {2014},
  journal = {Journal of the Royal Statistical Society: Series B (Statistical Methodology)},
  volume  = {76},
  number  = {3},
  pages   = {495--580},
  doi     = {10.1111/rssb.12047}
}

@article{fryzlewicz2014wbs,
  title   = {Wild Binary Segmentation for Multiple Change-Point Detection},
  author  = {Fryzlewicz, Piotr},
  year    = {2014},
  journal = {Annals of Statistics},
  volume  = {42},
  number  = {6},
  pages   = {2243--2281},
  doi     = {10.1214/14-AOS1245}
}

@article{green1995rjMCMC,
  title   = {Reversible Jump Markov Chain Monte Carlo Computation and {B}ayesian Model Determination},
  author  = {Green, Peter J.},
  year    = {1995},
  journal = {Biometrika},
  volume  = {82},
  number  = {4},
  pages   = {711--732},
  doi     = {10.1093/biomet/82.4.711}
}

@article{hutter2006bpcr,
  title   = {Exact Bayesian Regression of Piecewise Constant Functions},
  author  = {Hutter, Marcus},
  year    = {2007},
  journal = {Bayesian Analysis},
  volume  = {2},
  number  = {4},
  pages   = {635--664},
  doi     = {10.1214/07-BA225}
}

@article{jaakkola2000logisticvb,
  title   = {Bayesian Parameter Estimation via Variational Methods},
  author  = {Jaakkola, Tommi S. and Jordan, Michael I.},
  year    = {2000},
  journal = {Statistics and Computing},
  volume  = {10},
  number  = {1},
  pages   = {25--37},
  doi     = {10.1023/A:1008932416310}
}

@article{jackson2005optpart,
  title   = {An Algorithm for Optimal Partitioning of Data on an Interval},
  author  = {Jackson, B. and Scargle, J. D. and Barnes, D. and Arabhi, S. and Alt, A. and Gioumousis, P. and Gwin, E. and Sangtrakulcharoen, P. and Tan, L. and Tsai, T. T.},
  year    = {2005},
  journal = {IEEE Signal Processing Letters},
  volume  = {12},
  number  = {2},
  pages   = {105--108},
  doi     = {10.1109/LSP.2001.838216}
}

@article{killick2012pelt,
  title   = {Optimal Detection of Changepoints With a Linear Computational Cost},
  author  = {Killick, Rebecca and Fearnhead, Paul and Eckley, Idris A.},
  year    = {2012},
  journal = {Journal of the American Statistical Association},
  volume  = {107},
  number  = {500},
  pages   = {1590--1598},
  doi     = {10.1080/01621459.2012.737745}
}

@inproceedings{minka2001ep,
  title     = {Expectation Propagation for Approximate Bayesian Inference},
  author    = {Minka, Thomas P.},
  year      = {2001},
  booktitle = {Proceedings of the 17th Conference on Uncertainty in Artificial Intelligence (UAI)},
  pages     = {362--369},
  publisher = {Morgan Kaufmann}
}

@article{olshen2004cbs,
  title   = {Circular Binary Segmentation for the Analysis of Array-Based {DNA} Copy Number Data},
  author  = {Olshen, Adam B. and Venkatraman, E. S. and Lucito, Robert and Wigler, Michael},
  year    = {2004},
  journal = {Biostatistics},
  volume  = {5},
  number  = {4},
  pages   = {557--572},
  doi     = {10.1093/biostatistics/kxh008}
}

@article{polson2013pg,
  title   = {Bayesian Inference for Logistic Models Using {P}{\'o}lya--Gamma Latent Variables},
  author  = {Polson, Nicholas G. and Scott, James G. and Windle, Jesse},
  year    = {2013},
  journal = {Journal of the American Statistical Association},
  volume  = {108},
  number  = {504},
  pages   = {1339--1349},
  doi     = {10.1080/01621459.2013.829001}
}

@article{rigaill2010pruned,
  title   = {A Pruned Dynamic Programming Algorithm to Recover the Best Segmentations with 1 to $K_{\max}$ Change-Points},
  author  = {Rigaill, Guillem},
  year    = {2015},
  journal = {Journal de la Soci{\'e}t{\'e} Fran{\c c}aise de Statistique},
  volume  = {156},
  number  = {4},
  pages   = {180--205},
  url     = {https://www.numdam.org/item/JSFS_2015__156_4_180_0/},
  note    = {The 2010 arXiv:1004.0887 preprint is the earlier draft of this work; the JSFdS 2015 entry is the verified published version. Author-verification check in CODING\_AGENT\_HANDOFF.md.}
}

@article{rudin1992rof,
  title   = {Nonlinear Total Variation Based Noise Removal Algorithms},
  author  = {Rudin, Leonid I. and Osher, Stanley and Fatemi, Emad},
  year    = {1992},
  journal = {Physica D: Nonlinear Phenomena},
  volume  = {60},
  number  = {1-4},
  pages   = {259--268},
  doi     = {10.1016/0167-2789(92)90242-F}
}

@article{tibshirani2005fused,
  title   = {Sparsity and Smoothness via the Fused Lasso},
  author  = {Tibshirani, Robert and Saunders, Michael and Rosset, Saharon and Zhu, Ji and Knight, Keith},
  year    = {2005},
  journal = {Journal of the Royal Statistical Society: Series B (Statistical Methodology)},
  volume  = {67},
  number  = {1},
  pages   = {91--108},
  doi     = {10.1111/j.1467-9868.2005.00490.x}
}

@article{scargle2013bayesianblocks,
  title={Studies in Astronomical Time Series Analysis. VI. Bayesian Block Representations},
  author={Scargle, Jeffrey D. and Norris, Jay P. and Jackson, Brad and Chiang, Jed},
  journal={The Astrophysical Journal},
  volume={764},
  number={2},
  pages={167},
  year={2013},
  doi={10.1088/0004-637X/764/2/167},
  url={https://arxiv.org/abs/1207.5578}
}

@article{fearnhead2011dependence,
  title   = {Efficient Bayesian Analysis of Multiple Changepoint Models with Dependence across Segments},
  author  = {Fearnhead, Paul and Liu, Zhen},
  year    = {2011},
  journal = {Statistics and Computing},
  volume  = {21},
  number  = {2},
  pages   = {217--229},
  doi     = {10.1007/s11222-009-9163-6}
}

@article{ruggieri2014bcpvs,
  title   = {The Bayesian Change Point and Variable Selection Algorithm: Application to the $\delta^{18}$O Proxy Record of the Plio-Pleistocene},
  author  = {Ruggieri, Eric and Lawrence, Charles E.},
  year    = {2014},
  journal = {Journal of Computational and Graphical Statistics},
  volume  = {23},
  number  = {1},
  pages   = {87--110},
  doi     = {10.1080/10618600.2012.707852}
}

@article{fan2017basic,
  title   = {Empirical Bayesian Analysis of Simultaneous Changepoints in Multiple Data Sequences},
  author  = {Fan, Zhou and Mackey, Lester},
  year    = {2017},
  journal = {The Annals of Applied Statistics},
  volume  = {11},
  number  = {4},
  pages   = {2200--2221},
  doi     = {10.1214/17-AOAS1075}
}

@article{quinlan2024jrpm,
  title   = {Joint Random Partition Models for Multivariate Change Point Analysis},
  author  = {Quinlan, Jos{\'e} J. and Page, Garritt L. and Castro, Luis M.},
  year    = {2024},
  journal = {Bayesian Analysis},
  volume  = {19},
  number  = {1},
  pages   = {21--48},
  doi     = {10.1214/22-BA1344}
}

@article{muller2011ppmx,
  title   = {A Product Partition Model With Regression on Covariates},
  author  = {M{\"u}ller, Peter and Quintana, Fernando A. and Rosner, Gary L.},
  year    = {2011},
  journal = {Journal of Computational and Graphical Statistics},
  volume  = {20},
  number  = {1},
  pages   = {260--278},
  doi     = {10.1198/jcgs.2011.09066}
}

@article{park2010gppm,
  title   = {Bayesian Generalized Product Partition Model},
  author  = {Park, Ju-Hyun and Dunson, David B.},
  year    = {2010},
  journal = {Statistica Sinica},
  volume  = {20},
  number  = {3},
  pages   = {1203--1226}
}

@article{dempster1977em,
  author  = {Dempster, A. P. and Laird, N. M. and Rubin, D. B.},
  title   = {Maximum Likelihood from Incomplete Data via the {EM} Algorithm},
  journal = {Journal of the Royal Statistical Society: Series B (Methodological)},
  volume  = {39},
  number  = {1},
  pages   = {1--38},
  year    = {1977},
  doi     = {10.1111/j.2517-6161.1977.tb01600.x}
}

@article{teicher1963identifiability,
  author  = {Teicher, Henry},
  title   = {Identifiability of Finite Mixtures},
  journal = {Annals of Mathematical Statistics},
  volume  = {34},
  number  = {4},
  pages   = {1265--1269},
  year    = {1963},
  doi     = {10.1214/aoms/1177703862}
}

@article{corradin2026commonstructural,
  title   = {Model-based clustering of time-dependent observations with common structural changes},
  author  = {Corradin, Riccardo and Danese, Luca and KhudaBukhsh, Wasiur R. and Ongaro, Andrea},
  year    = {2026},
  journal = {Statistics and Computing},
  volume  = {36},
  pages   = {7},
  doi     = {10.1007/s11222-025-10756-x},
  note    = {Published online 28 October 2025; volume-year citation 2026.}
}

\beginappendix
\section{Stable numerical implementation of posterior summaries}
\label{app:numerical-summaries}
Family-specific integrated likelihoods, conjugate posterior updates, and complete proofs are provided directly under Propositions~\ref{prop:gaussian-family}--\ref{prop:beta-response-family}; the appendix does not repeat those calculations. Here we record two implementation details whose incorrect handling changes posterior inference even when every family-specific segment calculation is correct.

\subsection{A separate partition-prior normalizer}
For a fixed segment count, the unnormalized prior score of a path is $\prod_q\gamma_u(t_{q-1},t_q)$. Defining $C_0(0)=1$ and $C_0(j>0)=0$, its normalizer obeys
\begin{equation}
C_q(j)=\sum_{i=q-1}^{j-1}C_{q-1}(i)\gamma_u(i,j),
\qquad C_k(u)=C_k(n).
\label{eq:prior-normalizer-recurrence}
\end{equation}
Under $\gamma_u\equiv1$ with all ordered paths allowed, $C_k(u)=\binom{n-1}{k-1}$. That normalizer is never one merely because the local factors equal one, except in the degenerate endpoint counts. Excluded paths receive $\gamma_u(i,j)=0$; likelihood-zero paths are not removed from the \emph{prior} sum.

\begin{algorithm}[H]
\caption{Log-space normalization of a factorized partition prior}
\label{alg:partition-prior-normalizer}
\KwIn{Log prior segment factors $g_{ij}=\log\gamma_u(i,j)$ with $-\infty$ for structurally excluded blocks, and $k_{\max}$.}
Initialize $P_{0,0}=0$ and $P_{0,j>0}=-\infty$\;
\For{$q=1$ \KwTo $k_{\max}$}{
  \For{$j=q$ \KwTo $n$}{
    $P_{q,j}\leftarrow\logsumexp_{i=q-1}^{j-1}(P_{q-1,i}+g_{ij})$\;
  }
}
\KwOut{$\log C_k(u)=P_{k,n}$ for each feasible $k$; reject $k$ if $C_k=0$.}
\end{algorithm}

\subsection{Signed first moments and block inclusion weights}
Let $m^{(r)}_{ij}$ denote the posterior block moment $\mathbb E[m_\star(\theta)^r\mid y_{i+1:j}]$. In the forward--backward decomposition, the posterior probability that $(i,j]$ is the segment covering a specified coordinate under fixed $k$ is
\begin{equation}
\omega_{ij}(k)=\frac{\widetilde A_{ij}^{(0)}}{L_{k,n}}
\sum_{q=1}^{k}L_{q-1,i}R_{k-q,j}.
\label{eq:block-inclusion-weight}
\end{equation}
For coordinate $s$ the first and second moments are $\sum_{i<s\le j}\omega_{ij}(k)m^{(r)}_{ij}$, $r=1,2$. A Gaussian posterior mean can be negative, so $\log A^{(1)}_{ij}$ is not generally defined. The numerically stable interface stores $\log A^{(0)}_{ij}$ and signed $m^{(1)}_{ij}$ separately.

\begin{algorithm}[H]
\caption{Fixed-count Bayes curve with potentially negative segment means}
\label{alg:signed-moment-bayes-curve}
\KwIn{Forward/backward log messages, $\log\widetilde A^{(0)}_{ij}$, posterior segment moments $m^{(1)}_{ij},m^{(2)}_{ij}$, fixed count $k$.}
\For{each valid segment $(i,j]$}{
  Compute $\log\omega_{ij}(k)$ with $\logsumexp_q$ using Equation~\eqref{eq:block-inclusion-weight}\;
  Retain the sign of $m^{(1)}_{ij}$ separately from $\log|m^{(1)}_{ij}|$\;
}
\For{each coordinate $s$}{
  Sum positive and negative contributions $\omega_{ij}(k)m^{(1)}_{ij}$ for $i<s\le j$ using compensated arithmetic or signed-log-sum-exp\;
  Sum nonnegative contributions $\omega_{ij}(k)m^{(2)}_{ij}$ with stabilized arithmetic\;
  Set $\widehat\mu_s$ to the signed first-moment sum and $v_s$ to the second-moment sum minus $\widehat\mu_s^2$\;
  Check $v_s\ge-\delta$ for declared numerical tolerance $\delta$, and report any failure\;
}
\KwOut{Fixed-$k$ Bayes posterior-mean curve and pointwise posterior variances.}
\end{algorithm}

\paragraph{Masking rule.}
An unobserved outcome is absent from every likelihood and sufficient statistic. In a shared-boundary analysis, a subject with no observed outcomes in one block contributes log segment evidence zero when it has a proper segment prior. It must not contribute $-\infty$ merely because that subject supplies no information. Minimum-observation requirements are structural admissibility constraints and are represented in the prior support instead.

\section{Log-space dynamic programs}
\label{app:algorithms}

\begin{algorithm}[H]
\caption{Log-space sum-product messages}
\label{alg:log-sum-product-app}
\KwIn{Prior-adjusted log scores $\widetilde\ell_{ij}$; $k_{\max}$.}
Initialize $\Lambda_{0,0}=0$, $\Lambda_{0,j>0}=-\infty$, $\mathcal R_{0,n}=0$, and $\mathcal R_{0,i<n}=-\infty$\;
\For{$k=0$ \KwTo $k_{\max}-1$}{
  \For{$j=k+1$ \KwTo $n$}{
    $\Lambda_{k+1,j}\leftarrow\logsumexp_{h=k}^{j-1}\{\Lambda_{k,h}+\widetilde\ell_{h,j}\}$\;
  }
  \For{$i\leftarrow n-k-1,n-k-2,\ldots,0$}{
    $\mathcal R_{k+1,i}\leftarrow\logsumexp_{h=i+1}^{n-k}\{\widetilde\ell_{i,h}+\mathcal R_{k,h}\}$\;
  }
}
Normalize $\log p(k)+\Lambda_{k,n}-\log C_k(u)$ over $k$\;
\KwOut{$\Lambda$, $\mathcal R$, and $p(k\mid y)$.}
\end{algorithm}

\begin{algorithm}[H]
\caption{Joint-MAP max-sum and backtracking}
\label{alg:max-sum-app}
\KwIn{Prior-adjusted log scores $\widetilde\ell_{ij}$; fixed $k$.}
Set $M_{0,0}=0$ and $M_{0,j>0}=-\infty$\;
\For{$q=1$ \KwTo $k$}{
  \For{$j=q$ \KwTo $n$}{
    $h^\star(q,j)\leftarrow\arg\max_{h=q-1}^{j-1}\{M_{q-1,h}+\widetilde\ell_{h,j}\}$ under the declared tie rule\;
    $M_{q,j}\leftarrow M_{q-1,h^\star(q,j)}+\widetilde\ell_{h^\star(q,j),j}$\;
  }
}
Set $\widehat t_k=n$ and recursively $\widehat t_{q-1}=h^\star(q,\widehat t_q)$\;
\KwOut{$\widehat t$ and terminal score $M_{k,n}$.}
\end{algorithm}

The recommended numerical checks are
\begin{align*}
\sum_{k=1}^{k_{\max}}p(k\mid y)&=1,\\
|\Lambda_{k,n}-\mathcal R_{k,0}|&\le\delta,\\
\sum_{h=1}^{n-1}p(b_h=1\mid y,k)&=k-1,\\
\left|M_{k,n}-\sum_{q=1}^{k}
\widetilde\ell_{\widehat t_{q-1},\widehat t_q}\right|&\le\delta.
\end{align*}
with a documented floating-point tolerance $\delta$.

\section{Supplementary agreement diagnostics}
\label{app:agreement}

The pre-E09 comparator agreement tables for well-log, S\&P 500, and methylation are superseded. Those runs used exploratory MAP partitions (well-log $\hat{k}\approx23$, methylation $\hat{k}=15$, S\&P 500 $\hat{k}=29$) that differ from the E09 canonical fits (well-log $\hat{k}=83$, methylation persistently truncated at $k_{\max}=120$, S\&P 500 excluded because \texttt{yfinance} is not installed). Agreement values computed against a superseded reference boundary set are not informative about the canonical E09 results. The PELT comparison for well-log under the E09 canonical partition is reported in Section~\ref{sec:e09}.

\section{Reproducibility}
\label{app:reproducibility}
Reproduction of an experiment requires specifying the exact data source and
coordinate system, the preprocessing and exclusion rules, the observation
model and its descriptors, the prior, the numerical method, the random seeds,
and the model-selection and evaluation criteria. The prospective simulations
and their corresponding figures and tables are specified in
Section~\ref{sec:experimental-design}; full machine-readable configurations
and run instructions belong with the accompanying code.

\end{document}